\documentclass{article}

\usepackage[utf8]{inputenc} 
\usepackage[T1]{fontenc}    
\usepackage{hyperref}       
\usepackage{url}            
\usepackage{booktabs}       
\usepackage{amsfonts}       
\usepackage{nicefrac}       
\usepackage{microtype}      

\usepackage{amsmath,amsfonts,bm}

\def\Figref#1{Figure~\ref{#1}}

\def\1{\bm{1}}

\def\eps{{\epsilon}}

\def\vzero{{\bm{0}}}
\def\vone{{\bm{1}}}
\def\vmu{{\bm{\mu}}}
\def\vtheta{{\bm{\theta}}}
\def\valpha{{\bm{\alpha}}}
\def\vbeta{{\bm{\beta}}}
\def\vgamma{{\bm{\gamma}}}
\def\vomega{{\bm{\omega}}}
\def\vrho{{\bm{\rho}}}
\def\vell{{\bm{\ell}}}
\def\vpsi{{\bm{\psi}}}
\def\veps{{\bm{\eps}}}

\def\vdelta{{\bm{\delta}}}

\def\va{{\bm{a}}}
\def\vb{{\bm{b}}}

\def\vf{{\bm{f}}}

\def\vq{{\bm{q}}}

\def\vu{{\bm{u}}}
\def\vv{{\bm{v}}}
\def\vw{{\bm{w}}}
\def\vx{{\bm{x}}}
\def\vy{{\bm{y}}}

\def\tvx{{\widetilde{{\bm{x}}}}}

\def\tvv{{\widetilde{{\bm{v}}}}}

\def\mA{{\bm{A}}}
\def\mB{{\bm{B}}}

\def\mE{{\bm{E}}}
\def\mF{{\bm{F}}}

\def\mI{{\bm{I}}}
\def\mJ{{\bm{J}}}
\def\mK{{\bm{K}}}

\def\mV{{\bm{V}}}
\def\mW{{\bm{W}}}
\def\mX{{\bm{X}}}

\def\mPhi{{\bm{\Phi}}}
\def\mPsi{{\bm{\Psi}}}
\def\mLambda{{\bm{\Lambda}}}
\def\mSigma{{\bm{\Sigma}}}
\def\mTheta{{\bm{\Theta}}}

\def\tmW{{\widetilde{\bm{W}}}}
\def\tmX{{\widetilde{\bm{X}}}}

\DeclareMathAlphabet{\mathsfit}{\encodingdefault}{\sfdefault}{m}{sl}
\SetMathAlphabet{\mathsfit}{bold}{\encodingdefault}{\sfdefault}{bx}{n}

\def\gF{{\mathcal{F}}}

\def\gR{{\mathcal{R}}}

\newcommand{\E}{\mathbb{E}}

\newcommand{\R}{\mathbb{R}}

\DeclareMathOperator{\sign}{sign}
\DeclareMathOperator{\Tr}{Tr}

\usepackage{subcaption}
\usepackage{bbm}
\usepackage{bm}
\usepackage{mathtools}
\usepackage{amsthm, amssymb}
\newtheorem{thm}{Theorem}[section]

\newtheorem{lem}[thm]{Lemma}
\newtheorem{cor}[thm]{Corollary}
\newtheorem{prop}[thm]{Proposition}

\newtheorem{claim}[thm]{Claim}
\newtheorem{asmp}{Assumption}[section]

\def\1{\bm{1}}
\newcommand{\D}{\mathcal{D}}

\newcommand{\Y}{\mathcal{Y}}

\newcommand{\flin}{f^{\mathrm{lin}}}
\newcommand{\flinone}{f^{\mathrm{lin}1}}
\newcommand{\flintwo}{f^{\mathrm{lin}2}}
\newcommand{\faux}{f^{\mathrm{aux}}}
\newcommand{\fauxone}{f^{\mathrm{aux}1}}

\newcommand{\mThetalin}{\mTheta^{\mathrm{lin}}}
\newcommand{\mThetalinone}{\mTheta^{\mathrm{lin}1}}
\newcommand{\mThetalintwo}{\mTheta^{\mathrm{lin}2}}

\newcommand{\vulin}{\vu^{\mathrm{lin}}}

\newcommand{\N}{\mathcal N}

\newcommand{\abs}[1]{\left| #1 \right|}
\newcommand{\norm}[1]{\left\| #1 \right\|}

\newcommand{\relu}{{\mathrm{ReLU}}}
\newcommand{\erf}{{\mathrm{erf}}}

\newcommand{\unif}{\mathsf{Unif}}
\newcommand{\diag}{\mathrm{diag}}
\newcommand{\vect}[1]{\mathrm{vec}\left(#1\right)}
\newcommand{\offdiag}{\mathrm{off}}

\newcommand{\simiid}{\stackrel{\text{i.i.d.}}{\sim}}

\def\cnn{\mathsf{CNN}}

\newcommand{\error}{\tilde{O}\left(\frac{1}{\sqrt d}\right)}

\newcommand{\ind}[1]{\mathbbm{1}_{\left\{#1\right\}}}

\renewcommand{\index}[2]{\left[#1\right]_{#2}}

\allowdisplaybreaks

\usepackage{xcolor}

\usepackage{enumerate}
\usepackage{comment}
\usepackage{wrapfig}
\usepackage{floatrow}
\usepackage{cleveref}

\hypersetup{colorlinks,
	linkcolor=blue,
	citecolor=blue,
	urlcolor=magenta,
	linktocpage,
	plainpages=false}

\newcommand{\wei}[1]{}
\newcommand{\jp}[1]{}
\newcommand{\ba}[1]{}
\newcommand{\xl}[1]{}

\usepackage{natbib}
\usepackage[left=1in,right=1in,top=1in,bottom=1in]{geometry}
\date{}

\title{The Surprising Simplicity of the Early-Time Learning Dynamics of Neural Networks}
\author{%
  Wei Hu\thanks{Princeton University. Work partly performed at Google. Email: \texttt{huwei@cs.princeton.edu}}
  \and Lechao Xiao\thanks{Google Research, Brain Team. Email: \texttt{xlc@google.com}}
  \and Ben Adlam\thanks{Google Research, Brain Team. Work done as a member of the Google AI Residency program (\url{http://g.co/brainresidency}). Email: \texttt{adlam@google.com}}
  \and Jeffrey Pennington\thanks{Google Research, Brain Team. Email: \texttt{jpennin@google.com}}
}

\begin{document}

\maketitle

\begin{abstract}
Modern neural networks are often regarded as complex black-box functions whose behavior is difficult to understand owing to their nonlinear dependence on the data and the nonconvexity in their loss landscapes. In this work, we show that these common perceptions can be completely false in the early phase of learning. In particular, we formally prove that, for a class of well-behaved input distributions, the early-time learning dynamics of a two-layer fully-connected neural network can be mimicked by training a simple linear model on the inputs. We additionally argue that this surprising simplicity can persist in networks with more layers and with convolutional architecture, which we verify empirically. Key to our analysis is to bound the spectral norm of the difference between the Neural Tangent Kernel (NTK) at initialization and an affine transform of the data kernel; however, unlike many previous results utilizing the NTK, we do not require the network to have disproportionately large width, and the network is allowed to escape the kernel regime later in training.
\end{abstract}

\section{Introduction} \label{sec:intro}
Modern deep learning models are enormously complex function approximators, with many state-of-the-art architectures employing millions or even billions of trainable parameters \citep{radford2019language, adiwardana2020towards}. While the raw parameter count provides only a crude approximation of a model's capacity, more sophisticated metrics such as those based on PAC-Bayes~\citep{mcallester1999pac,dziugaite2017computing,neyshabur2017exploring}, VC dimension~\citep{vapnik1971uniform}, and parameter norms~\citep{bartlett2017spectrally,neyshabur2017pac} also suggest that modern architectures have very large capacity. Moreover, from the empirical perspective, practical models are flexible enough to perfectly fit the training data, even if the labels are pure noise~\citep{zhang2017understanding}. Surprisingly, these same high-capacity models generalize well when trained on real data, even without any explicit control of capacity.

These observations are in conflict with classical generalization theory, which contends that models of intermediate complexity should generalize best, striking a balance between the bias and the variance of their predictive functions. To reconcile theory with observation, it has been suggested that deep neural networks may enjoy some form of implicit regularization induced by gradient-based training algorithms that biases the trained models towards simpler functions. However, the exact notion of simplicity and the mechanism by which it might be achieved remain poorly understood except in certain simplistic settings.

One concrete mechanism by which such induced simplicity can emerge is the hypothesis that neural networks learn simple functions early in training, and increasingly build up their complexity in later time. In particular, recent empirical work~\cite{nakkiran2019sgd} found that, intriguingly, in some natural settings the simple function being learned in the early phase may just be a linear function of the data.

In this work, we provide a novel theoretical result to support this hypothesis. Specifically, we formally prove that, for a class of well-behaved input distributions, the early-time learning dynamics of gradient descent on a two-layer fully-connected neural network with any common activation can be mimicked by training a simple model of the inputs. 
When training the first layer only, this simple model is a linear function of the input features; when training the second layer or both layers, it is a linear function of the features and their $\ell_2$ norm. This result implies that neural networks do not fully exercise their nonlinear capacity until late in training.

Key to our technical analysis is a bound on the spectral norm of the difference between the Neural Tangent Kernel (NTK)~\citep{jacot2018neural} of the neural network at initialization and that of the linear model; indeed, a weaker result, like a bound on the Frobenius norm, would be insufficient to establish our result.
Although the NTK is usually associated with the study of ultra-wide networks, our result only has a mild requirement on the width and allows the network to leave the kernel regime later in training.
While our formal result focuses on two-layer fully-connected networks and data with benign concentration properties (specified in Assumption~\ref{asmp:input-distr}), we argue with theory and provide empirical evidence that the same linear learning phenomenon persists for more complex architectures and real-world datasets.




\paragraph{Related work.} 
The early phase of neural network training has been the focus of considerable recent research.  \citet{frankle2019lottery} found that sparse, trainable subnetworks --  ``lottery tickets" -- emerge early in training. \citet{achille2017critical} showed the importance of early learning from the perspective of creating strong connections that are robust to corruption. \citet{gur2018gradient} observed that after a short period of training, subsequent gradient updates span a low-dimensional subspace.
\citet{li2019towards, lewkowycz2020large} showed that an initial large learning rate can benefit late-time generalization performance.

Implicit regularization of (stochastic) gradient descent has also been studied in various settings, suggesting a bias towards large-margin, low-norm, or low-rank solutions~\citep{gunasekar2017implicit,gunasekar2018implicit,soudry2018implicit,li2018algorithmic,ji2019implicit,ji2019gradient,arora2019implicit,lyu2019gradient,chizat2020implicit,razin2020implicit}. These results mostly aim to characterize the final solutions at convergence, while our focus is on the early-time learning dynamics.
Another line of work has identified that deep linear networks gradually increase the rank during training~\citep{arora2019implicit,saxe2014exact,lampinen2018analytic,gidel2019implicit}.

A line of work adopted the Fourier perspective and demonstrated that low-frequency functions are often learned first~\citep{rahaman2018spectral, xu2018understanding,xu2019frequency,xu2019training}.
Based on the NTK theory, \citet{arora2019fine} showed that for very wide networks, components lying in the top eigenspace of the NTK are learned faster than others. Using this principle, \citet{su2019learning, cao2019towards} analyzed the spectrum of the infinite-width NTK. However, in order to obtain precise characterization of the spectrum these papers require special data distributions such as uniform distribution on the sphere.



Most relevant to our work is the finding of~\citet{nakkiran2019sgd} that a neural network learned in the early phase of training can be almost fully explained by a linear function of the data. They supported this claim empirically by examining an information theoretic measure between the predictions of the neural network and the linear model. Our result formally proves that neural network and a corresponding linear model make similar predictions in early time, thus providing a theoretical explanation of their empirical finding.


\paragraph{Paper organization.}
In \Cref{sec:prelim}, we introduce notation and briefly recap the Neural Tangent Kernel.
In \Cref{sec:two-layer}, we present our main theoretical results on two-layer neural networks as well as empirical verification.
In \Cref{sec:extension}, we discuss extensions to more complicated architecture from both theoretical and empirical aspects.
We conclude in \Cref{sec:conclu}, and defer additional experimental results and all the proofs to the appendices.

\section{Preliminaries} \label{sec:prelim}

\paragraph{Notation.}
We use bold lowercases $\va,\vb,\valpha,\vbeta, \ldots$ to represent vectors, bold uppercases $\mA,\mB,\ldots$ to represent matrices, and unbold letters $a, b, \alpha, \beta,\ldots$ to represent scalars.
We use $\index{\mA}{i,j}$ or $\index{\va}{i}$ to index the entries in matrices or vectors.
We denote by $\norm{\cdot}$ the spectral norm (largest singular value) of a matrix or the $\ell_2$ norm of a vector, and denote by $\norm{\cdot}_F$ the Frobenius norm of a matrix.
We use $\langle\cdot,\cdot\rangle$ to represent the standard Euclidean inner product between vectors or matrices, and use $\odot$ to denote the Hadamard (entry-wise) product between matrices.
For a positive semidefinite (psd) matrix $\mA$, let $\mA^{1/2}$ be the psd matrix such that  $(\mA^{1/2})^2 = \mA$; let $\lambda_{\max}(\mA)$ and $\lambda_{\min}(\mA)$ be the maximum and minimum eigenvalues of $\mA$.

Let $[n]:=\{1,2,\ldots,n\}$.
For $a, b\in\R$ ($b>0$), we use $a\pm b$ to represent any number in the interval $[a-b,a+b]$.
Let $\mI_d$ be the $d\times d$ identity matrix,
$\vzero_d$ be the all-zero vector in $\R^d$,
and $\vone_d$ be the all-one vector in $\R^d$; we write $\mI, \vzero, \vone$ when their dimensions are clear from context.
We denote by $\unif(A)$ the uniform distribution over a set $A$, and by $\N(\mu, \sigma^2)$ or $\N(\vmu, \mSigma)$ the univariate/multivariate Gaussian distribution.
Throughout the paper we let $g$ be a random variable with the standard normal distribution $\N(0,1)$.

We use the standard $O(\cdot)$, $\Omega(\cdot)$ and $\Theta(\cdot)$ notation to only hide universal constant factors.
For $a, b\ge0$,
we also use $a\lesssim b$ or $b\gtrsim a$ to mean $a=O(b)$, and use $a \ll b$ or $b \gg a$ to mean $b \ge C  a$ for a sufficiently large
universal constant $C > 0$.
Throughout the paper, ``high probability'' means a large constant probability arbitrarily close to $1$ (such as $0.99$).

\paragraph{Recap of Neural Tangent Kernel (NTK)~\citep{jacot2018neural}.}


Consider a single-output neural network $f(\vx;\vtheta)$ where $\vx$ is the input and $\vtheta$ is the collection of parameters in the network.
Around a reference network with parameters $\bar{\vtheta}$, we can do a local first-order approximation:
$$
    f(\vx;\vtheta) \approx f(\vx;\bar{\vtheta}) + \langle \nabla_\vtheta f(\vx;\bar{\vtheta}), \vtheta - \bar{\vtheta} \rangle .
$$
Thus when $\vtheta$ is close to $\bar{\vtheta}$, for a given input $\vx$ the network can be viewed as linear in $\nabla_\vtheta f(\vx;\bar{\vtheta})$. This gradient feature map $\vx\mapsto\nabla_\vtheta f(\vx;\bar{\vtheta})$ induces a kernel $K_{\bar{\vtheta}}(\vx,\vx') := \langle \nabla_\vtheta f(\vx;\bar{\vtheta}), \nabla_\vtheta f(\vx';\bar{\vtheta}) \rangle$ which is called the NTK at $\bar{\vtheta}$.
Gradient descent training of the neural network can be viewed as kernel gradient descent on the function space with respect to the NTK.
We use \emph{NTK matrix} to refer to an $n\times n$ matrix that is the NTK evaluated on $n$ datapoints.

While in general the NTK is random at initialization and can vary significantly during training, it was shown that, for a suitable network parameterization (known as the ``NTK parameterization''), when the width goes to infinity or is sufficiently large, the NTK converges to a deterministic limit at initialization and barely changes during training~\citep{jacot2018neural, lee2019wide, arora2019exact, yang2019scaling}, so that the neural network trained by gradient descent is equivalent to a kernel method with respect to a fixed kernel.
{However, for networks with practical widths, the NTK does usually stray far from its initialization.}

\section{Two-Layer Neural Networks} \label{sec:two-layer}
We consider a two-layer fully-connected neural network with $m$ hidden neurons defined as:
\begin{equation} \label{eqn:two-layer-nn}
    f(\vx; \mW, \vv) := \frac{1}{\sqrt m} \sum_{r=1}^m v_r \phi\left(\vw_r^\top \vx / \sqrt{d}\right) = \frac{1}{\sqrt{m}}\vv^\top \phi\left(\mW \vx / \sqrt{d}\right),
\end{equation}
where $\vx \in \R^d$ is the input, $\mW = [\vw_1, \ldots, \vw_m]^\top \in \R^{m\times d}$ is the weight matrix in the first layer, and $\vv = [v_1, \ldots, v_m]^\top \in \R^m$ is the weight vector in the second layer.\footnote{The scaling factors $\frac{1}{\sqrt{d}}$ and $\frac{1}{\sqrt{m}}$ are due to the NTK parameterization such that the weights can be initialized from $\N(0,1)$. The standard parameterization can also be equivalently realized with the NTK parameterization by properly setting different learning rates in different layers~\citep{lee2019wide}, which we do allow here.\label{footnote:NTK-parameterization}} Here $\phi:\R\to\R$ is an activation function that acts entry-wise on vectors or matrices.

Let $\{(\vx_i, y_i)\}_{i=1}^n \subset \R^d\times\R$ be $n$ training samples where $\vx_i$'s are the inputs and $y_i$'s are their associated labels.
Denote by $\mX = [\vx_1, \ldots, \vx_n]^\top \in \R^{n\times d}$ the data matrix and by $\vy=[y_1,\ldots,y_n]^\top\in\R^n$ the label vector.
We assume $|y_i| \le 1$ for all $i\in[n]$.

We consider the following $\ell_2$ training loss:
\begin{equation} \label{eqn:two-layer-training-loss}
    L(\mW, \vv) := \frac{1}{2n} \sum_{i=1}^n \left( f(\vx_i; \mW, \vv) - y_i \right)^2 ,
\end{equation}
and run vanilla gradient descent (GD) on the objective~\eqref{eqn:two-layer-training-loss} starting from random initialization. Specifically, we use the following \emph{symmetric initialization} for the weights $(\mW, \vv)$:
\begin{equation} \label{eqn:two-layer-symm-init}
\begin{aligned}
    &\vw_1, \ldots, \vw_{m/2}  \simiid \N(\vzero_d, \mI_d), \quad \vw_{i+m/2} = \vw_{i}\, (\forall i\in[m/2]),\\
    &v_1 , \ldots,  v_{m/2} \simiid \unif(\{1,-1\}),\footnotemark \quad v_{i+m/2}= -v_i \, (\forall i\in[m/2]) .
\end{aligned}
\end{equation}\footnotetext{Our results also hold for $\N(0,1)$ initialization in the second layer. Here we use $\unif(\{\pm1\})$ for simplicity.}
The above initialization scheme was used by \citet{chizat2019lazy, zhang2019type, hu2020Simple, bai2020Beyond}, etc. It initializes the network to be the difference between two identical (random) networks, which has the benefit of ensuring zero output: $f(\vx; \mW, \vv)=0$ ($\forall \vx\in\R^d$), without altering the NTK at initialization. An alternative way to achieve the same effect is to subtract the function output at initialization~\citep{chizat2019lazy}.

Let $(\mW(0),\vv(0))$ be a set of initial weights drawn from the symmetric initialization~\eqref{eqn:two-layer-symm-init}. Then the weights are updated according to GD:
\begin{equation} \label{eqn:two-layer-gd}
\begin{aligned}
    \mW(t+1) = \mW(t) - \eta_1 \nabla_\mW L\left(\mW(t), \vv(t) \right),\quad
    \vv(t+1) = \vv(t) - \eta_2 \nabla_\vv L\left(\mW(t), \vv(t) \right),
\end{aligned}
\end{equation}
where $\eta_1$ and $\eta_2$ are the learning rates. Here we allow potentially different learning rates for flexibility.

Now we state the assumption on the input distribution used in our theoretical results.
\begin{asmp}[input distribution] \label{asmp:input-distr}
The datapoints $\vx_1, \ldots, \vx_n$ are i.i.d. samples from a distribution $\D$ over $\R^d$ with mean $\vzero$ and covariance $\mSigma$ such that $\Tr[\mSigma] = d$ and $\norm{\mSigma} = O(1)$.
Moreover, $\vx\sim\D$ can be written as $\vx = \mSigma^{1/2}\bar{\vx}$ where $\bar{\vx}\in\R^d$ satisfies $\E[\bar{\vx}]=\vzero_d$, $\E[\bar{\vx}\bar{\vx}^\top] = \mI_d$, and $\bar\vx$'s entries are independent and are all $O(1)$-subgaussian.\footnote{Recall that a zero-mean random variable $X$ is $\sigma^2$-subgaussian if $\E[\exp(sX)] \le \exp(\sigma^2s^2/2) $ ($\forall s\in\R$).}
\end{asmp}

Note that a special case that satisfies Assumption~\ref{asmp:input-distr} is the Gaussian distribution $\N(\vzero, \mSigma)$, but we allow a much larger class of distributions here. The subgaussian assumption is made due to the probabilistic tail bounds used in the analysis, and it can be replaced with a weaker bounded moment condition.
The independence between $\bar{\vx}$'s entries may also be dropped if its density is strongly log-concave.
We choose to use Assumption~\ref{asmp:input-distr} as the most convenient way to present our results.

We allow $\phi$ to be any of the commonly used activation functions, including ReLU, Leaky ReLU, Erf, Tanh, Sigmoid, Softplus, etc. Formally, our requirement on $\phi$ is the following:
\begin{asmp}[activation function] \label{asmp:activation}
    The activation function $\phi(\cdot)$ satisfies either of the followings:
    \begin{enumerate}[(i)]
        \item smooth activation: $\phi$ has bounded first and second derivatives: $|\phi'(z)| = O(1)$ and $|\phi''(z)| = O(1)$ ($\forall z \in\R$), or
        \item piece-wise linear activation: $\phi(z) = \begin{cases}z\,\,\,\, \,(z\ge0)\\az\,\, (z<0) \end{cases}$ for some $a\in\R, |a|=O(1)$.\footnote{We define $\phi'(0)=1$ in this case.}
    \end{enumerate}
\end{asmp}

We will consider the regime where the data dimension $d$ is sufficiently large (i.e., larger than any constant) and the number of datapoints $n$ is at most some polynomial in $d$ (i.e., $n\le d^{O(1)}$). These imply $\log n = O(\log d) < d^{c}$ for any constant $c>0$.

Under Assumption~\ref{asmp:input-distr}, the datapoints satisfy the following concentration properties:
\begin{claim} \label{claim:data-concentration}
    Suppose $n\gg d$. Then under Assumption~\ref{asmp:input-distr}, with high probability we have $\frac{\norm{\vx_i}^2}{d}=1\pm O\Big( \sqrt{\tfrac{\log n}{d}} \Big)$ ($\forall i\in[n]$),  $\frac{|\langle \vx_i, \vx_j \rangle|}{d} = O\Big( \sqrt{\tfrac{\log n}{d}} \Big)$ ($\forall i , j\in[n], i\not=j$),
    and $\norm{\mX\mX^\top} = \Theta(n)$. 
\end{claim}

The main result in this section is to formally prove that the neural network trained by GD is approximately a linear function in the early phase of training.
As we will see, there are distinct contributions coming from the two layers. Therefore, it is helpful to divide the discussion into the cases of training the first layer only, the second layer only, and both layers together.
All the omitted proofs in this section are given in \Cref{app:two-layer}.

\subsection{Training the First Layer}
Now we consider only training the first layer weights $\mW$, which corresponds to setting $\eta_2=0$ in~\eqref{eqn:two-layer-gd}.
Denote by $f_t^{1}:\R^d\to\R$ the network at iteration $t$ in this case, namely $f_t^{1}(\vx) := f(\vx; \mW(t), \vv(t)) = f(\vx; \mW(t), \vv(0))$ (note that $\vv(t)=\vv(0)$).


The linear model which will be proved to approximate the neural network $f_t^{1}$ in the early phase of training is $\flinone(\vx; \vbeta) := \vbeta^\top \vpsi_1(\vx)$, where
\begin{equation} \label{eqn:first-layer-linear-model}
    \vpsi_1(\vx):=\frac{1}{\sqrt{d}}  \begin{bmatrix} \zeta \vx \\ \nu \end{bmatrix}, \qquad \text{with } \zeta = \E[\phi'(g)] \text{ and } \nu =  \E[g \phi'(g)] \cdot \sqrt{{\Tr[\mSigma^2]}/{d}}.
\end{equation}
Here recall that $g\sim\mathcal{N}(0,1)$. We also consider training this linear model via GD on the $\ell_2$ loss, this time starting from zero:
\begin{equation} \label{eqn:first-layer-linear-gd}
    \vbeta(0) = \vzero_{d+1}, \quad \vbeta(t+1) = \vbeta(t) - \eta_1 \nabla_\vbeta \frac{1}{2n}\sum_{i=1}^n\left(\flinone(\vx_i;\vbeta(t))-y_i\right)^2. 
\end{equation}
We let $\flinone_t$ be the model learned at iteration $t$, i.e., $\flinone_t(\vx) := \flinone(\vx; \vbeta(t))$.

We emphasize that~\eqref{eqn:two-layer-gd} and~\eqref{eqn:first-layer-linear-gd} have the same learning rate $\eta_1$.
Our theorem below shows that $f_t^1$ and $\flinone_t$ are close to each other in the early phase of training:
\begin{thm}[main theorem for training the first layer] \label{thm:first-layer-main}
    Let $\alpha \in (0, \frac14)$ be a fixed constant.
    Suppose the number of training samples $n$ and the network width $m$ satisfy $ n\gtrsim d^{1+\alpha} $ 
    and $m \gtrsim d^{1+\alpha}$.
    Suppose $\eta_1 \ll d $ and $\eta_2=0$.
    Then there exists a universal constant $c>0$ such that with high probability, for all $0\le t \le T= c\cdot \frac{ d\log d}{\eta_1}$ simultaneously, the learned neural network $f^1_t$ and the linear model $\flinone_t$ at iteration $t$ are close on average on the training data:
    \begin{equation} \label{eqn:first-layer-training-guarantee}
        \frac1n \sum_{i=1}^n \left( f^1_t(\vx_i) - \flinone_t(\vx_i) \right)^2 \lesssim d^{-\Omega(\alpha)}.
    \end{equation}
    Moreover, $f_t^1$ and $\flinone_t$ are also close on the underlying data distribution $\D$. Namely, with high probability, for all $0\le t \le T$ simultaneously, we have
    \begin{equation} \label{eqn:first-layer-test-guarantee}
        \E_{\vx\sim\D}\left[ \min\{(f^1_t(\vx) - \flinone_t(\vx))^2,1\} \right] \lesssim d^{-\Omega(\alpha)} + \sqrt{\tfrac{\log T}{n}}.
    \end{equation}
\end{thm}

Theorem~\ref{thm:first-layer-main} ensures that the neural network $f^1_t$ and the linear model $\flinone_t$ make almost the same predictions in the early time of training. This agreement is not only on the training data, but also over the underlying input distribution $\D$.
Note that this does not mean that $f^1_t$ and $\flinone_t$ are the same on the entire space $\R^d$ -- they might still differ significantly at low-density regions of $\D$.
We also remark that our result has no assumption on the labels $\{y_i\}$ except they are bounded.

The width requirement in Theorem~\ref{thm:first-layer-main} is very mild as it only requires the width $m$ to be larger than $d^{1+\alpha}$ for some small constant $\alpha$.
Note that the width is allowed to be much smaller than the number of samples $n$, which is usually the case in practice.

The agreement guaranteed in Theorem~\ref{thm:first-layer-main} is up to iteration $ T = c\cdot \frac{ d\log d}{\eta_1}$ (for some constant $c$).
It turns out that for well-conditioned data, after $T$ iterations, a near optimal linear model will have been reached. This means that \emph{the neural network in the early phase approximates a linear model all the way until the linear model converges to the optimum}.
See Corollary~\ref{cor:first-layer-well-condioned} below.

\begin{cor}[well-conditioned data] \label{cor:first-layer-well-condioned}
    Under the same setting as Theorem~\ref{thm:first-layer-main}, and additionally assume that the data distribution $\D$'s covariance $\mSigma$ satisfies $\lambda_{\min}(\mSigma) = \Omega(1)$.
    Let $\vbeta_*\in\R^{d+1}$ be the optimal parameter for the linear model that GD~\eqref{eqn:first-layer-linear-gd} converges to, and denote $\flinone_*(\vx):=\flinone(\vx;\vbeta_*)$.
    Then with high probability, after $T = c\cdot \frac{ d\log d}{\eta_1}$ iterations (for some universal constant $c$), we have
    \begin{align*}
        \frac1n \sum_{i=1}^n \left( f^1_T(\vx_i) - \flinone_*(\vx_i) \right)^2 \lesssim d^{-\Omega(\alpha)}, \,\,
        \E_{\vx\sim\D}\left[ \min\{(f^1_T(\vx) - \flinone_*(\vx))^2, 1\} \right] \lesssim d^{-\Omega(\alpha)} + \sqrt{\tfrac{\log T}{n}}.
    \end{align*}
\end{cor}

\subsubsection{Proof Sketch of Theorem~\ref{thm:first-layer-main}}

The proof of Theorem~\ref{thm:first-layer-main} consists of showing that the NTK matrix for the first layer at random initialization evaluated on the training data is close to the kernel matrix corresponding to the linear model~\eqref{eqn:first-layer-linear-model}, and that furthermore this agreement persists in the early phase of training up to iteration $T$.
Specifically, the NTK matrix $\mTheta_1(\mW)\in\R^{n\times n}$ at a given first-layer weight matrix $\mW$, and the kernel matrix $\mThetalinone\in\R^{n\times n}$ for the linear model~\eqref{eqn:first-layer-linear-model} can be computed as:
\begin{align*}
    \mTheta_1(\mW):=  \big(\phi'(\mX\mW^\top/\sqrt{d}) \phi'(\mX\mW^\top/\sqrt{d})^\top / m\big) \odot ({\mX\mX^\top}/{d}), \,\,
    \mThetalinone:= (\zeta^2\mX\mX^\top + \nu^2\vone\vone^\top)/d.
\end{align*}
We have the following result that bounds the difference between $\mTheta_1(\mW(0))$ and $\mThetalinone$ in spectral norm:
\begin{prop} \label{prop:first-layer-ntk-init-approx}
    With high probability over the random initialization $\mW(0)$ and the training data $\mX$, we have
    $
    \norm{\mTheta_1(\mW(0)) - \mThetalinone} \lesssim \frac{n}{d^{1+\alpha}} 
    $.
\end{prop}

Notice that $\norm{\mThetalinone} = \Theta(\frac{n}{d})$ according to Claim~\ref{claim:data-concentration}. 
Thus the bound $\frac{n}{d^{1+\alpha}}$ in Proposition~\ref{prop:first-layer-ntk-init-approx} is of smaller order.
We emphasize that it is important to bound the spectral norm rather than the more naive Frobenius norm, since the latter would give $\norm{\mTheta_1(\mW(0)) - \mThetalinone}_F \gtrsim \frac{n}{d}$, which is not useful. (See \Cref{fig:spec_norm_decay} for a numerical verification.)

To prove Proposition~\ref{prop:first-layer-ntk-init-approx}, we first use the matrix Bernstein inequality to bound the perturbation of $\mTheta_1(\mW(0))$ around its expectation with respect to $\mW(0)$: $\norm{\mTheta_1(\mW(0)) - \E_{\mW(0)}[\mTheta_1(\mW(0))]} \lesssim \frac{n}{d^{1+\alpha}}$. Then we perform an entry-wise Taylor expansion of $\E_{\mW(0)}[\mTheta_1(\mW(0))]$, and it turns out that the top-order terms exactly constitute $\mThetalinone$, and the rest can be bounded in spectral norm by $\frac{n}{d^{1+\alpha}}$.

After proving Proposition~\ref{prop:first-layer-ntk-init-approx}, in order to prove Theorem~\ref{thm:first-layer-main}, we carefully track (i) the prediction difference between $f^1_t$ and $\flinone_t$, (ii) how much the weight matrix $\mW$ move away from initialization, as well as (iii) how much the NTK changes.
To prove the guarantee on the entire data distribution we further need to utilize tools from generalization theory.
The full proof is given in \Cref{app:two-layer}.

\subsection{Training the Second Layer}

Next we consider training the second layer weights $\vv$, which corresponds to $\eta_1=0$ in~\eqref{eqn:two-layer-gd}.
Denote by $f^2_t:\R^d\to\R$ the network at iteration $t$ in this case. 
We will show that training the second layer is also close to training a simple linear model $\flintwo(\vx; \vgamma) :=  \vgamma^\top \vpsi_2(\vx)$ in the early phase, where:
\begin{equation} \label{eqn:second-layer-linear-model}
     \vpsi_2(\vx) :=
     \begin{bmatrix}
     \frac{1}{\sqrt{d}} \zeta \vx \\
     \frac{1}{\sqrt{2d}}\nu \\
     \vartheta_0 + \vartheta_1(\frac{\norm{\vx}}{\sqrt d}-1) +  \vartheta_2(\frac{\norm{\vx}}{\sqrt d}-1)^2
     \end{bmatrix}, \qquad
     \begin{cases}
     \zeta \text{ and } \nu \text{ are defined in~\eqref{eqn:first-layer-linear-model}},\\
     \vartheta_0 = \E[\phi(g)],\\
     \vartheta_1 = \E[g\phi'(g)],\\
     \vartheta_2 = \E[(\frac12g^3-g)\phi'(g)] .
     \end{cases}
\end{equation}
 As usual, this linear model is trained with GD starting from zero:
\begin{equation} \label{eqn:second-layer-linear-gd}
    \vgamma(0) = \vzero_{d+2}, \quad \vgamma(t+1) = \vgamma(t) - \eta_2 \nabla_\vgamma \frac{1}{2n}\sum_{i=1}^n (\flintwo(\vx_i;\vgamma(t)) - y_i)^2.
\end{equation}
We denote by $\flintwo_t$ the resulting model at iteration $t$. 

Note that strictly speaking $\flintwo(\vx;\vgamma)$ is not a linear model in $\vx$ because the feature map $\vpsi_2(\vx)$ contains a nonlinear feature depending on $\norm{\vx}$ in its last coordinate.
Because $\frac{\norm{\vx}}{\sqrt d}\approx 1$ under our data assumption according to Claim~\ref{claim:data-concentration}, its effect might often be invisible. However, we emphasize that in general the inclusion of this norm-dependent feature is necessary, for example when the target function explicitly depends on the norm of the input.
We illustrate this in Section~\ref{subsec:two-layer-exp}.

Similar to Theorem~\ref{thm:first-layer-main}, our main theorem for training the second layer is the following:
\begin{thm}[main theorem for training the second layer] \label{thm:second-layer-main}
    Let $\alpha \in (0, \frac14)$ be a fixed constant.
    Suppose 
    $ n\gtrsim d^{1+\alpha} $ 
    and $\begin{cases}
    m \gtrsim d^{1+\alpha}, \text{ if }\,\E[\phi(g)] = 0\\
    m \gtrsim d^{2+\alpha}, \text{ otherwise}
    \end{cases}
    $. 
    Suppose $\begin{cases}
    \eta_2 \ll d/\log n, \text{ if }\,\E[\phi(g)] = 0\\
    \eta_2 \ll 1, \ \ \ \ \ \ \ \ \ \ \ \text{ otherwise}
    \end{cases}$ and $\eta_1=0$.
    Then there exists a universal constant $c>0$ such that with high probability, for all $0\le t \le T=c\cdot \frac{ d\log d}{\eta_2}$ simultaneously, we have 
    \begin{align*} 
        \frac1n \sum_{i=1}^n \left( f^2_t(\vx_i) - \flintwo_t(\vx_i) \right)^2 \lesssim d^{-\Omega(\alpha)}, \quad \E_{\vx\sim\D}\left[ \min\{ (f^2_t(\vx) - \flintwo_t(\vx))^2, 1\} \right] \lesssim d^{-\Omega(\alpha)} .
    \end{align*}
\end{thm}

Similar to Theorem~\ref{thm:first-layer-main}, an important step in proving Theorem~\ref{thm:second-layer-main} is to prove that the NTK matrix for the second layer is close to the kernel for the linear model~\eqref{eqn:second-layer-linear-model}. Note that the theorem treats the case $\vartheta_0 = \E[\phi(g)]=0$ differently. This is because when $\vartheta_0\not=0$, the second layer NTK has a large eigenvalue of size $\Theta(n)$, while when $\vartheta_0=0$, its largest eigenvalue is only $O(\frac{n\log n}{d})$. 

We remark that if the data distribution is well-conditioned, we can also have a guarantee similar to Corollary~\ref{cor:first-layer-well-condioned}.

\subsection{Training Both Layers} \label{subsec:both-layers}

Finally we consider the case where both layers are trained, in which $\eta_1=\eta_2=\eta>0$ in~\eqref{eqn:two-layer-gd}.
Since the NTK for training both layers is simply the sum of the first-layer NTK and the second-layer NTK, the corresponding linear model should have its kernel being the sum of the kernels for linear models~\eqref{eqn:first-layer-linear-model} and~\eqref{eqn:second-layer-linear-model}, which can be derived easily:
\begin{equation} \label{eqn:both-layer-linear-model}
    \flin(\vx; \vdelta):=\vdelta^\top\vpsi(\vx), \quad
    \vpsi(\vx) :=
    \begin{bmatrix}
    \sqrt{\frac2d} \zeta \vx\\
    \sqrt{\frac{3}{2d}}\nu\\
    \vartheta_0 + \vartheta_1(\frac{\norm{\vx}}{\sqrt d}-1) +  \vartheta_2(\frac{\norm{\vx}}{\sqrt d}-1)^2
    \end{bmatrix},
\end{equation}
where the constants 
are from~\eqref{eqn:second-layer-linear-model}.
Note that $\langle \vpsi(\vx), \vpsi(\vx') \rangle = \langle \vpsi_1(\vx), \vpsi_1(\vx') \rangle + \langle \vpsi_2(\vx), \vpsi_2(\vx') \rangle$.


Again, we can show that the neural network is close to the linear model~\eqref{eqn:both-layer-linear-model} in early time. The guarantee is very similar to Theorems~\ref{thm:first-layer-main} and~\ref{thm:second-layer-main}, so we defer the formal theorem to \Cref{app:two-layer}; see Theorem~\ref{thm:both-layers-main}.
Note that our result can be directly generalized to the case where $\eta_1\not=\eta_2$, for which we just need to redefine the linear model using a weighted combination of the kernels for~\eqref{eqn:first-layer-linear-model} and~\eqref{eqn:second-layer-linear-model}.

\subsection{Empirical Verification} \label{subsec:two-layer-exp}

\begin{figure}[t]
\begin{center}
\begin{subfigure}[b]{0.3\textwidth}
                \centering
               \includegraphics[width=1\textwidth]{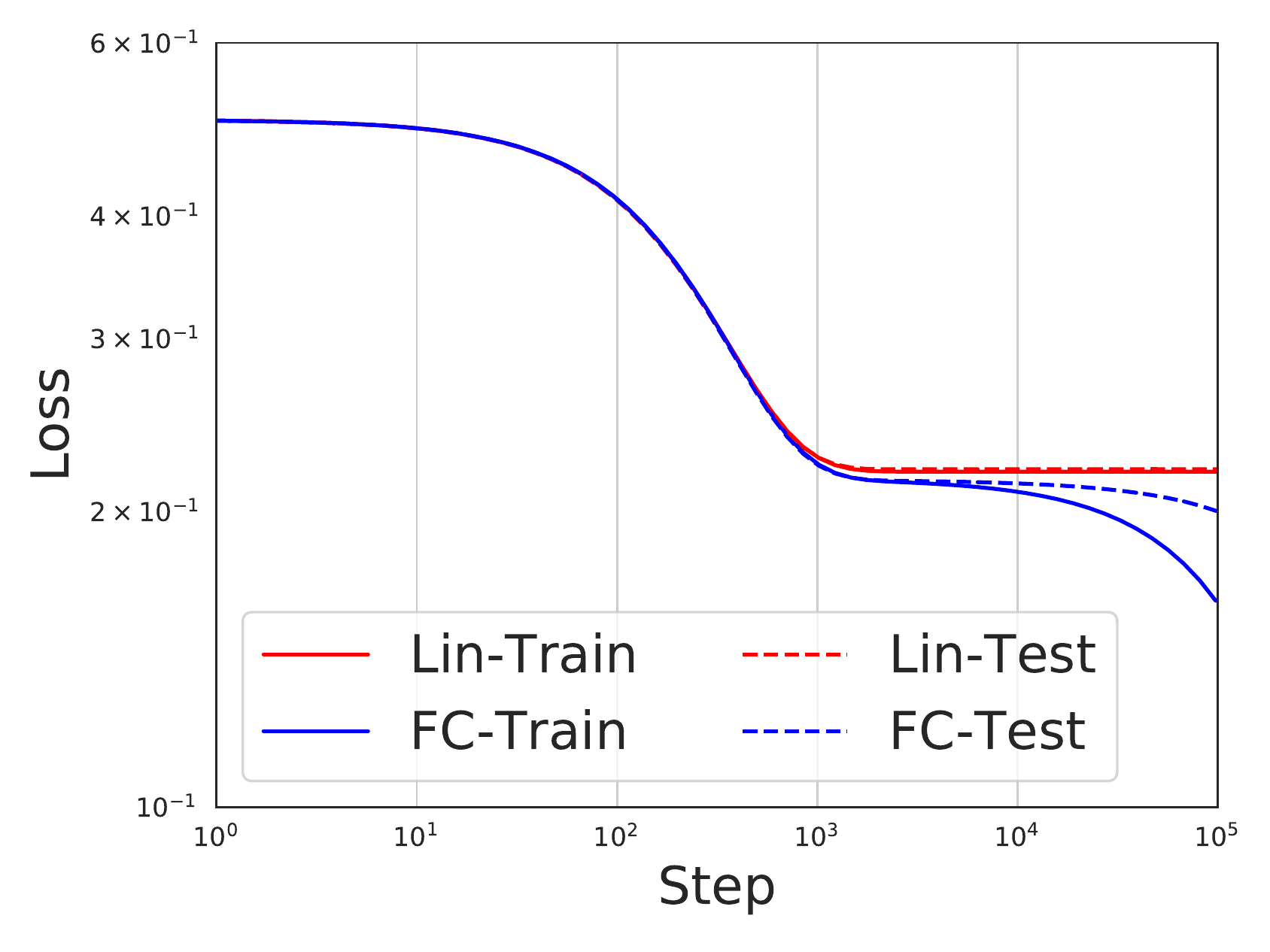}
                \caption{Loss}
                \label{fig:gaussian-loss}
        \end{subfigure}%
         \begin{subfigure}[b]{.3\textwidth}
                \centering
\includegraphics[width=1.\textwidth]{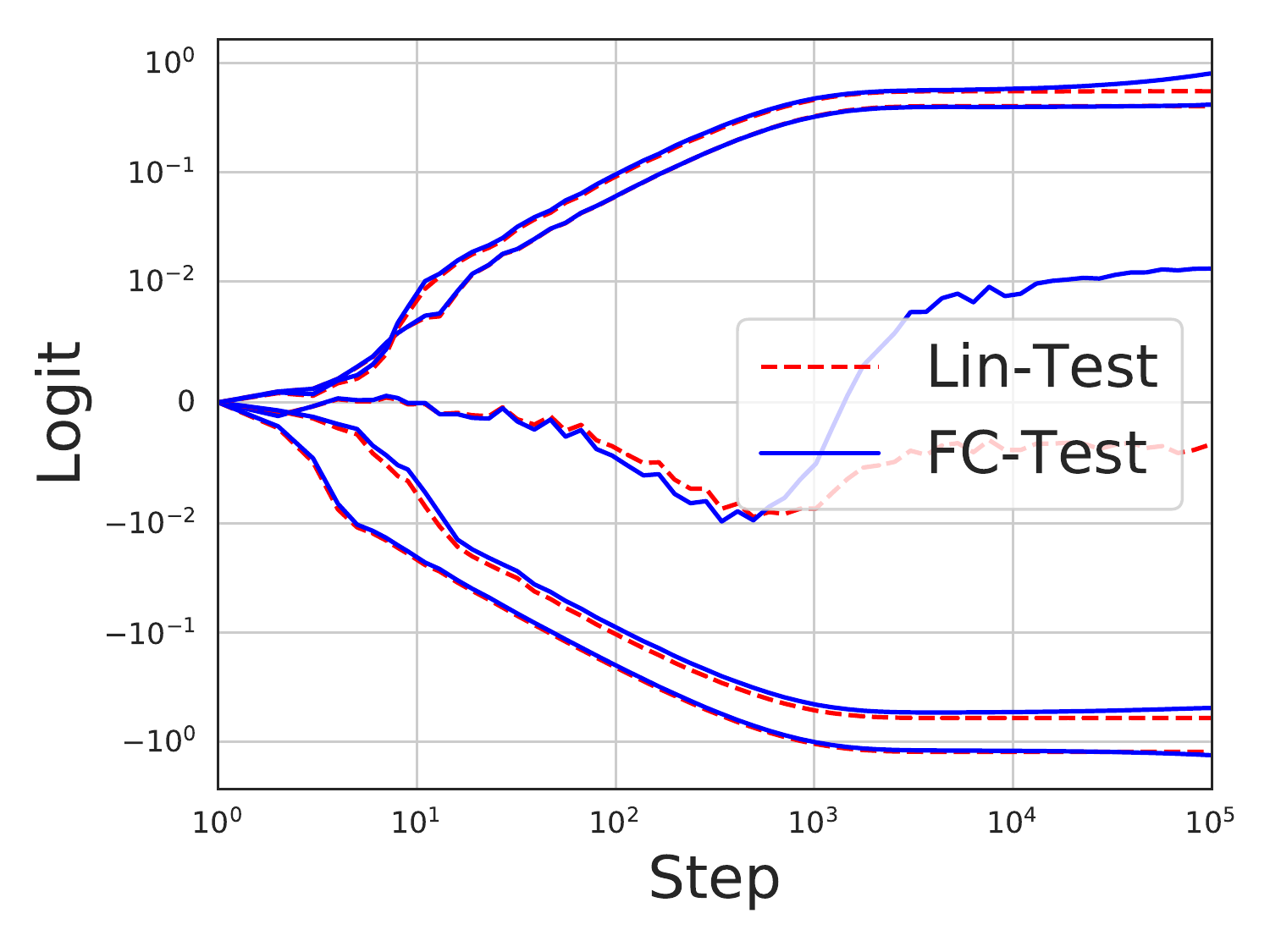}
                \caption{Test logits}
                \label{fig:Gaussian-logits-test}
        \end{subfigure}
        \begin{subfigure}[b]{0.3\textwidth}
                \centering
               \includegraphics[width=1\textwidth]{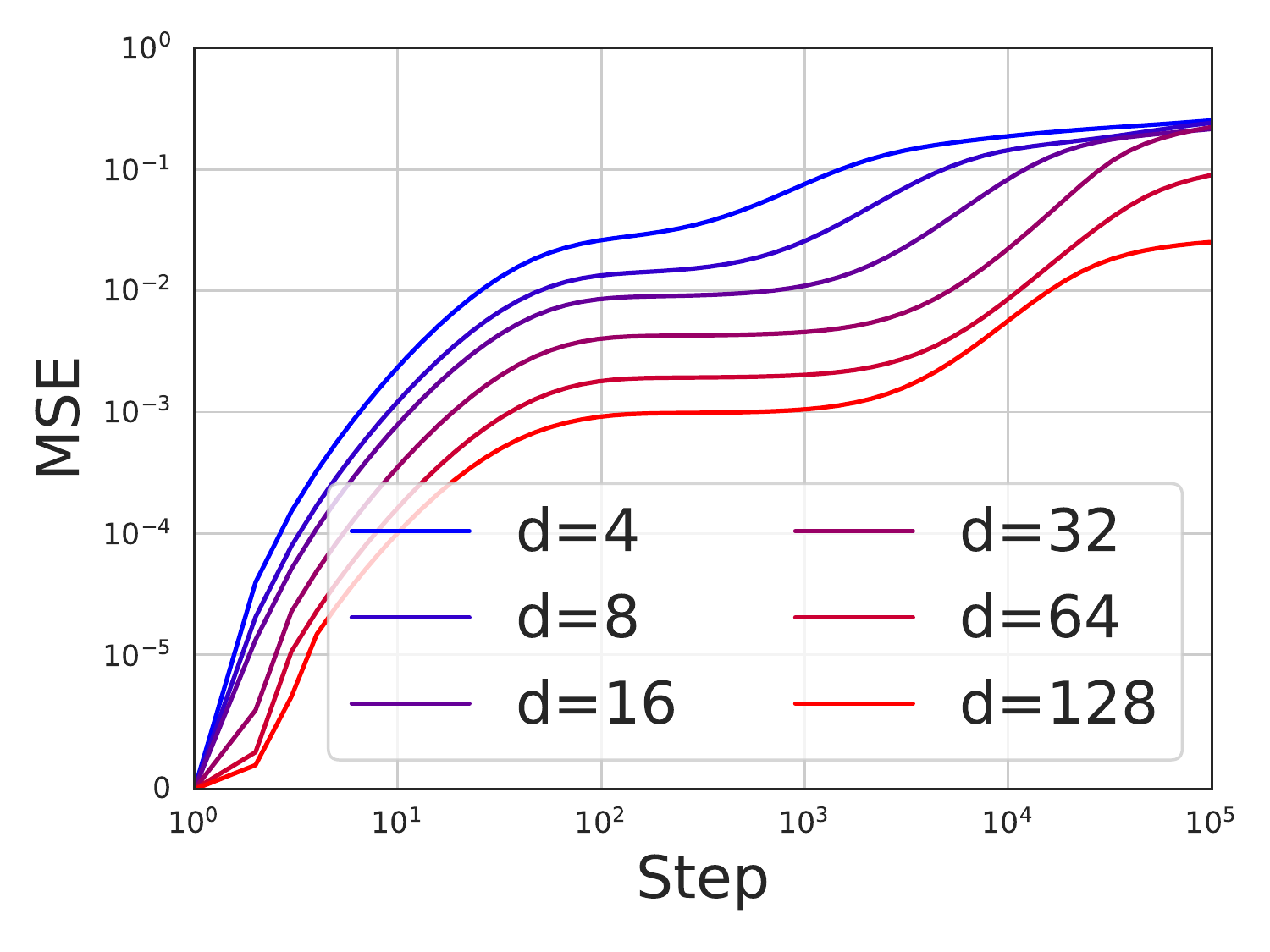}
                \caption{Discrepancy}
                \label{fig:Gaussian-discrepancy}
        \end{subfigure}%
\end{center}
\caption{{\bf Two-layer neural network learns a linear model early in training.} (a) Losses of a \textcolor{blue}{neural network} and the corresponding \textcolor{red}{linear model} predicted by~\eqref{eqn:both-layer-linear-model}. Solid (dashed) lines represent the training (test) losses. We have $d=50$, and use 20,000 training samples and 2,000 test samples.
The \textcolor{blue}{neural network} and the \textcolor{red}{linear model} are indistinguishable in the first 1,000 steps, after which linear learning finishes and the network continues to make progress.
(b) Evolution of logits (i.e., outputs) of $5$ random test examples. We see excellent agreement between the predictions of the \textcolor{blue}{neural network} and the \textcolor{red}{linear model} in early time.
(c) Discrepancy (in MSE) between the outputs of the network and the linear model for various values of $d$. As predicted, the discrepancy becomes smaller as $d$ increases.
}
\label{fig:two-layer-gaussian}
\end{figure}

\paragraph{Verifying the early-time agreement between neural network and linear model.}

We verify our theory by training a two-layer neural network with $\erf$ activation and width $256$ on synthetic data generated by $\vx\sim\N(\vzero,\mI)$ and $y =\sign(f^*(\vx))$, where $f^*$ is a ground-truth two-layer $\erf$ network with width~$5$. 
In \Figref{fig:gaussian-loss}, we plot the training and test losses of the neural network (colored in blue) and its corresponding linear model $\flin$ (in red).\footnote{For $\phi=\erf$, we have $\vartheta_0=\vartheta_1=\vartheta_2=0$, so $\flin$ in~\eqref{eqn:both-layer-linear-model} is a linear model in $\vx$ without the nonlinear feature.}
In the early training phase (up to 1,000 steps), the training/test losses of the network and the linear model are indistinguishable.
After that, the optimal linear model is reached, and the network continues to make progress.
In \Figref{fig:Gaussian-logits-test}, we plot the evolution of the outputs (logits) of the network and the linear model on $5$ random test examples, and we see excellent early-time agreement even on each individual sample.
Finally, in \Figref{fig:Gaussian-discrepancy}, we vary the input dimension $d$, and for each case plot the mean squared error (MSE) of the discrepancies between the outputs of the network and the linear model. We see that the discrepancy indeed becomes smaller as $d$ increases, matching our theoretical prediction.




\paragraph{The necessity of the norm-dependent feature.}
We now illustrate the necessity of including the norm-dependent feature in~\eqref{eqn:both-layer-linear-model} and~\eqref{eqn:second-layer-linear-model} through an example of learning a norm-dependent function.
We generate data from $\vx\sim\N(\vzero,\mI)$ and $y = \frac{\norm{\vx}}{\sqrt{d}} + \relu(\va^\top\vx)$ ($\norm{\va}=O(1)$), and train a two-layer network with $\relu$ activation. We also train the corresponding linear model $\flin$~\eqref{eqn:both-layer-linear-model} as well as a ``naive linear model'' which is identical to $\flin$ except $\vartheta_1$ and $\vartheta_2$ are replaced with $0$. \Figref{fig:norm_function_example} shows that $\flin$ is indeed a much better approximation to the neural network than the naive linear model.
\definecolor{ao(english)}{rgb}{0.0, 0.5, 0.0}
\begin{figure}[t]
\floatbox[{\capbeside\thisfloatsetup{capbesideposition={right,top},capbesidewidth=8cm}}]{figure}[\FBwidth]
{\caption{{\bf The norm-dependent feature is necessary.}
    For the task of learning a norm-dependent function, test losses are shown for a \textcolor{blue}{neural network} with $\relu$ activation, its corresponding \textcolor{red}{linear model} predicted by~\eqref{eqn:both-layer-linear-model}, and a {\color{ao(english)} naive linear model} by resetting $\vartheta_1=\vartheta_2=0$ in~\eqref{eqn:both-layer-linear-model}.
    Our predicted linear model is a much better approximation to the neural network than the naive linear model.
    }
    \label{fig:norm_function_example}}
    {\centering
    \includegraphics[width=0.3\textwidth]{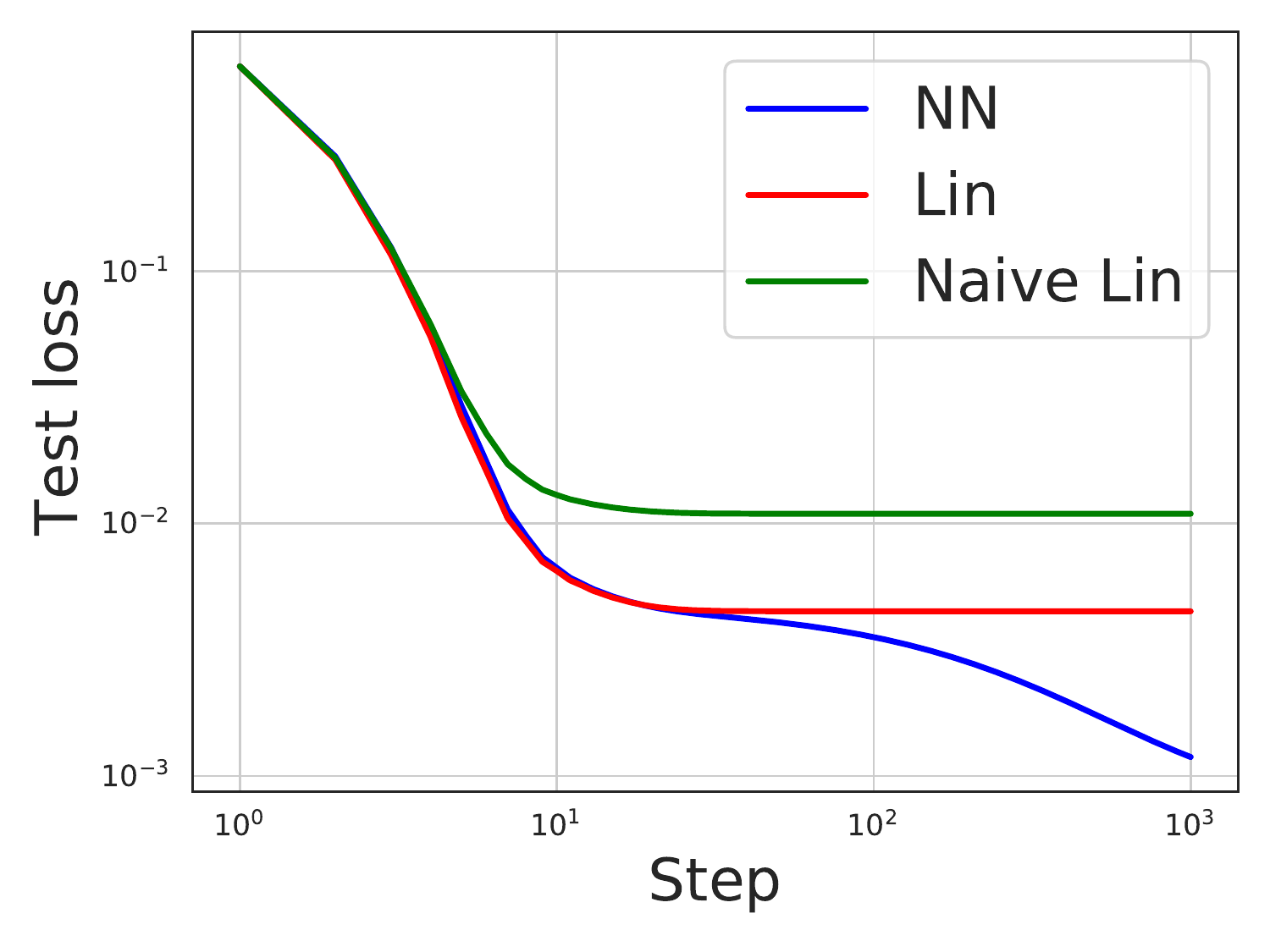}}
\end{figure}

\section{Extensions to Multi-Layer and Convolutional Neural Networks} \label{sec:extension}
In this section, we provide theoretical and empirical evidence supporting that the agreement between neural networks and linear models in the early phase of training may continue to hold for more complicated network architectures and datasets than what we analyzed in Section~\ref{sec:two-layer}.

\subsection{Theoretical Observations}

\paragraph{Multi-layer fully-connected (FC) neural networks.}
For multi-layer FC networks, it was known that their infinite-width NTKs have the form $K(\vx,\vx') = h(\tfrac{\norm{\vx}^2}{d},\tfrac{\|\vx'\|^2}{d},\tfrac{\langle\vx,\vx'\rangle}{d})$ ($\vx,\vx'\in\R^d$) for some function $h:\R^3\to\R$~\citep{yang2019fine}.
Let $\mTheta$ be the NTK matrix on the $n$ training data: $\index{\mTheta}{i,j}=K(\vx_i,\vx_j)$.
Under Assumption~\ref{asmp:input-distr}, we know from Claim~\ref{claim:data-concentration} that $\frac{\norm{\vx_i}^2}{d}\approx 1$ and $\frac{\langle\vx_i,\vx_j\rangle}{d}\approx0$ ($i\not=j$).
Hence we can Taylor expand $h$ around $(1,1,0)$ for the off-diagonal entries of $\mTheta$ and around $(1,1,1)$ for the diagonal entries.
Similar to our analysis of two-layer networks, we should be able to bound the higher-order components in the expansion, and only keep the simple ones like $\mX\mX^\top$, $\vone\vone^\top$, etc. This suggests that the early-time linear learning behavior which we showed for two-layer FC networks may persist in multi-layer FC networks.

\paragraph{Convolutional neural networks (CNNs).}
We consider a simple 1-dimensional CNN with one convolutional layer and without pooling (generalization to the commonly used 2-dimensional CNNs is straightforward):
\begin{equation} \label{eqn:cnn}
    f_\cnn(\vx; \mW, \mV) := \frac{1}{\sqrt{md}} \sum_{r=1}^m \vv_r^\top \phi\left(\vw_r * \vx / \sqrt{q}\right).
\end{equation}
Here $\vx \in \R^d$ is the input, $\mW = [\vw_1, \ldots, \vw_m]^\top \in \R^{m\times q}$ and $\mV = [\vv_1,\ldots,\vv_m]^\top \in \R^{m\times d}$ contain the weights, where $m$ is the number of channels (or width), and $q\le d$ is the filter size.
All the weights are initialized i.i.d from $\N(0,1)$.
The convolution operator $*$ is defined as: for input $\vx\in\R^d$ and filter $\vw\in\R^q$, we have $\vw*\vx\in\R^d$ with $\index{\vw*\vx}{i} := \sum_{j=1}^q \index{\vw}{j} \index{\vx}{i+j-1} $. We consider circular padding (as in~\citet{xiao2018dynamical,li2019enhanced}), so the indices in input should be understood as $\index{\vx}{i}=\index{\vx}{i+d}$. 

We have the following result concerning the NTK of this CNN:
\begin{prop}\label{prop:cnn}
    Let $\phi=\erf$.
    Suppose $n\gtrsim d^{1+\alpha}$ and $q\gtrsim d^{\frac12+2\alpha}$ for some constant $\alpha\in(0,\frac14)$.
	Consider $n$ datapoints $\vx_1,\ldots,\vx_n\simiid \unif(\{\pm1\}^d)$.
	Then the corresponding NTK matrix $\mTheta_\cnn\in\R^{n\times n}$ of the CNN~\eqref{eqn:cnn} in the infinite-width limit ($m\to\infty$) satisfies $\norm{\mTheta_\cnn - 2\zeta^2\mX\mX^\top/d}\lesssim \frac{n}{d^{1+\alpha}}$ with high probability, where $\zeta=\E[\phi'(g)]$.
\end{prop}

The proof is given in \Cref{app:proof-cnn}.
The above result shows that the NTK of a CNN can also be close to the (scaled) data kernel, which implies the linear learning behavior in the early time of training the CNN.
Our empirical results will show that this behavior can even persist to multi-layer CNNs and real data beyond our analysis.

\subsection{Empirical Results}

We perform experiments on a binary classification task from CIFAR-10 (``cats'' vs ``horses'') using a multi-layer FC network and a CNN. The numbers of training and test data are 10,000 and 2,000. The original size of the images is $32\times32\times3$, and we down-sample the images into size $8 \times 8 \times 3$ using a $4\times 4$ average pooling. Then we train a 4-hidden-layer FC net and a 4-hidden-layer CNN with $\erf$ activation. To have finer-grained examination of the evolution of the losses, we decompose the residual of the predictions on test data (namely, $f_t(\vx)-y$ for all test data collected as a vector in $\R^{2000}$) onto $V_{\text{lin}}$, the space spanned by the inputs (of dimension $d=192$), and its complement $V_{\text{lin}}^{\perp}$ (of dimension $2000-d$). For both networks, we observe in \Figref{fig:decompose loss} that the test losses of the networks and the linear model are almost identical up to 1,000 steps, and the networks start to make progress in $V_{\text{lin}}^{\perp}$ after that.
In \Figref{fig:cifar-logits} we plot the logit evolution of $3$ random test datapoints and again observe good agreement in early time.
In \Figref{fig:relative mse}, we plot the relative MSE between the network and the linear model (i.e., $\mathbb E_{\vx} \|f_t(\vx) - \flin_t(\vx)\|^2/ 
\mathbb E_{\vx} \| \flin_t(\vx)\|^2$ evaluated on test data). We observe that this quantity for either network is small in the first 1,000 steps and grows afterwards.
The detailed setup and additional results for full-size CIFAR-10 and MNIST are deferred to \Cref{app:exp}.

\begin{figure}[t]
\begin{center}
\begin{subfigure}[b]{0.33\textwidth}
                \centering
               \includegraphics[width=1\textwidth]{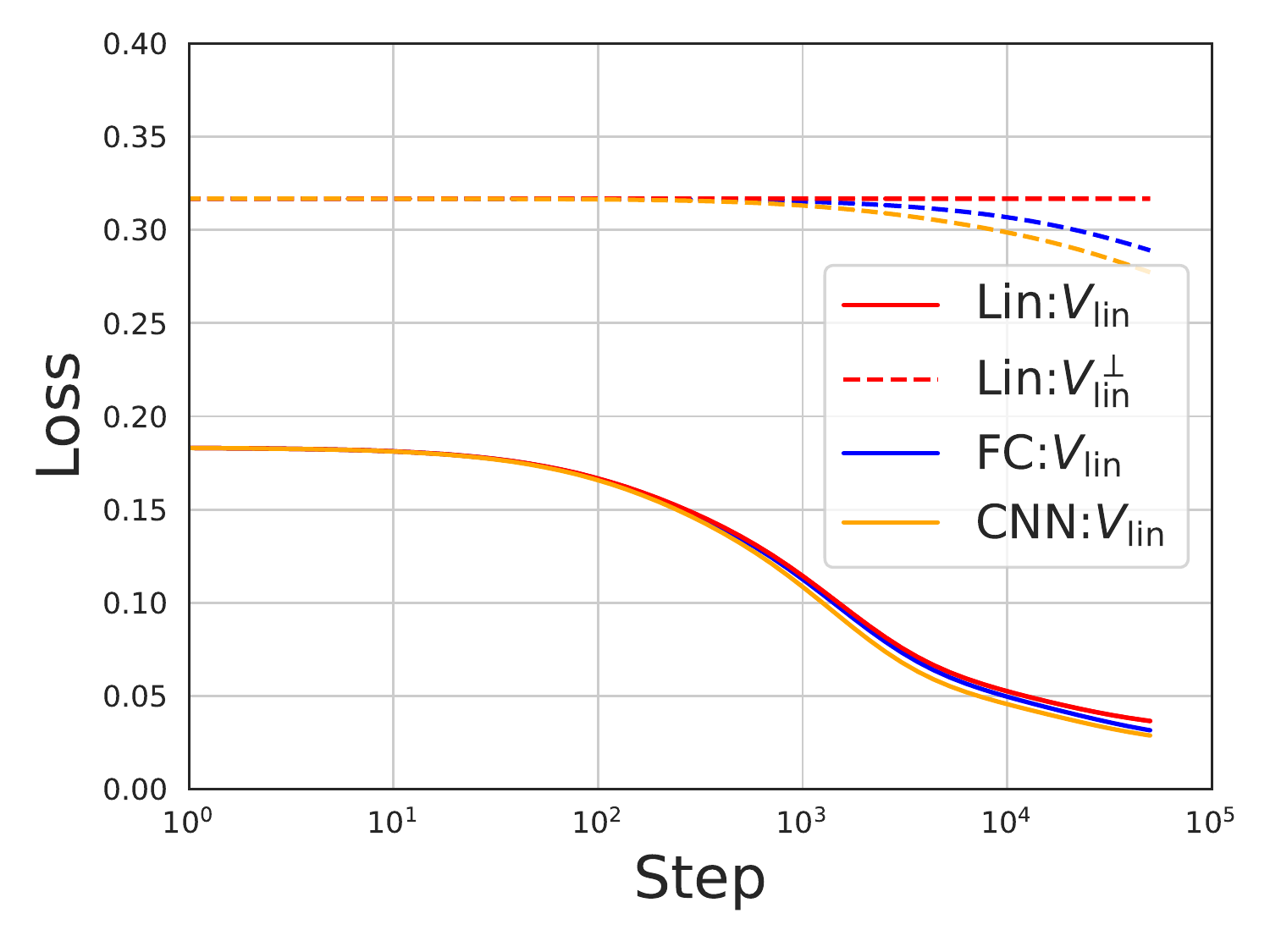}
                \caption{Test loss decomposition}
                \label{fig:decompose loss}
        \end{subfigure}%
        \begin{subfigure}[b]{0.33\textwidth}
                \centering
               \includegraphics[width=1\textwidth]{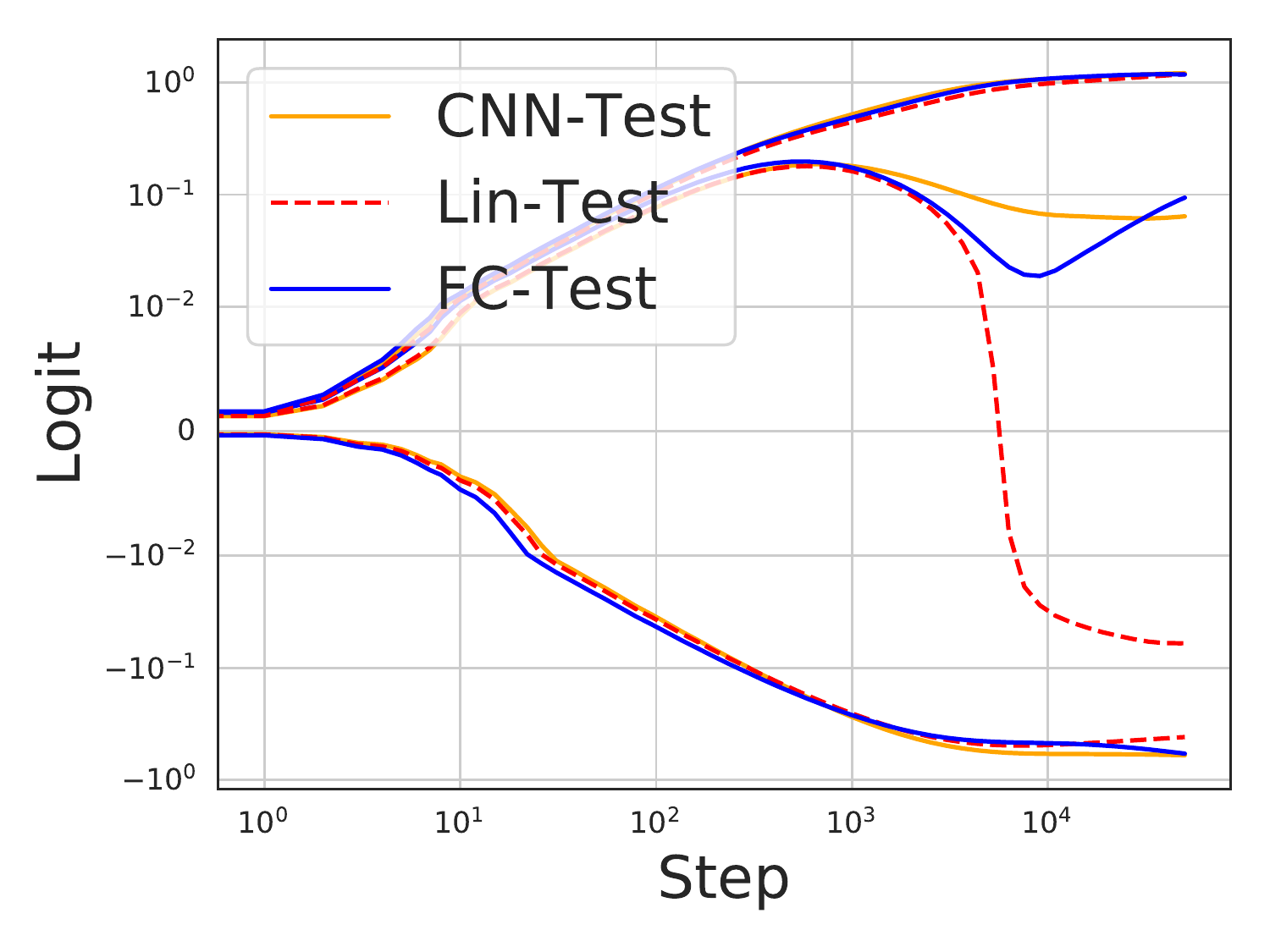}
                \caption{Test logits}
                \label{fig:cifar-logits}
        \end{subfigure}%
         \begin{subfigure}[b]{.33\textwidth}
                \centering
\includegraphics[width=1.\textwidth]{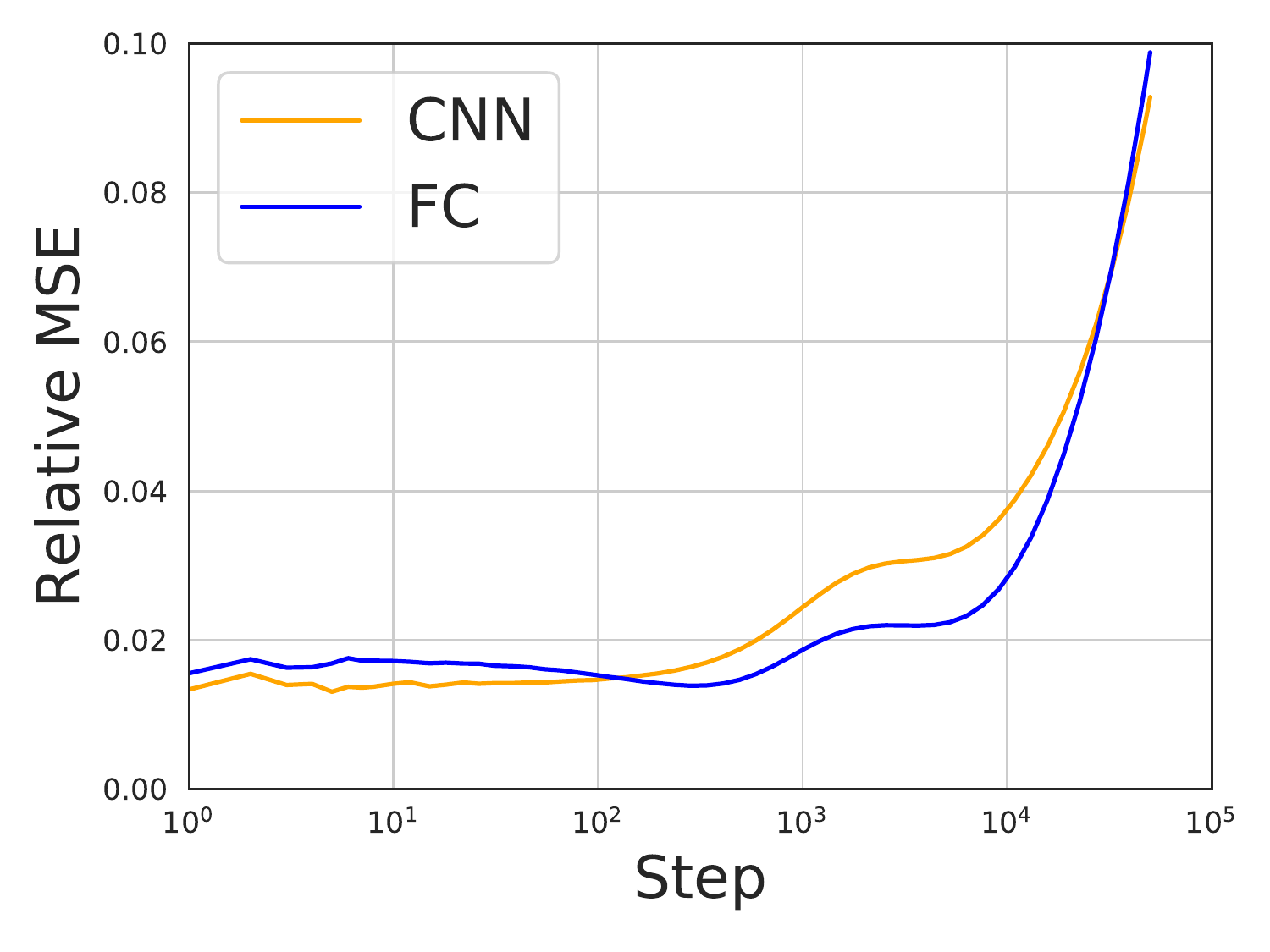}
                \caption{Relative MSE of discrepancy}
                \label{fig:relative mse}
        \end{subfigure}
\end{center}
\caption{{\bf Good agreement between 4-hidden-layer CNN/FC network and linear model on CIFAR-10 early in training.}
(a) Decomposition of the test losses onto $V_{\text{lin}}$ (solid lines) and $V_{\text{lin}}^{\perp}$ (dashed lines) for {\color{orange} CNN}, \textcolor{blue}{FC} and the corresponding {\color{red} linear model}.
(b) Three randomly selected test outputs for different models.
(c) The relative MSE between the networks and the linear model.
Note that we adjust the learning rates of {\color{orange}{CNN}} and {\color{blue}{FC}} so that their corresponding linear models are identical.
}
\label{fig:cifar-main-text}
\vspace{-2mm}
\end{figure}

\section{Conclusion} \label{sec:conclu}

This work gave a novel theoretical result rigorously showing that gradient descent on a neural network learns a simple linear function in the early phase. While we mainly focused on two-layer fully-connected neural networks, we further provided theoretical and empirical evidence suggesting that this phenomenon continues to exist in more complicated models. Formally extending our result to those settings is a direction of future work. Another interesting direction is to study the dynamics of neural networks after the initial linear learning phase.



\bibliography{ref}

\bibliographystyle{plainnat}

\newpage
\section*{\LARGE Appendices}
\appendix

In \Cref{app:exp}, we describe additional experiment details and provide additional plots.
In \Cref{app:notation}, we introduce additional notation and some lemmas that will be used in the proofs.
In \Cref{app:general-closeness}, we present a general result that shows how the GD trajectory of a non-linear least squares problem can be approximated by a linear one, which will be used in the proofs.
Finally, in \Cref{app:two-layer,app:proof-cnn} we provide omitted details and proofs in \Cref{sec:two-layer,sec:extension}, respectively.

\section{Experiment Setup and Additional Plots} \label{app:exp}

We provide additional plots and describe additional experiment details in this section.

In \Cref{fig:meta-plots}, we repeat the same experiments in \Cref{fig:cifar-main-text} on the full-size ($32\times32\times3$) CIFAR-10 as well as MNIST datasets, using the same $4$-hidden-layer FC and CNN architectures.
For both datasets we take two classes and perform binary classification.
We see very good early-time agreement except for CNN on CIFAR-10, where the agreement only lasts for a shorter time.

For the experiments in \Cref{fig:cifar-main-text,fig:meta-plots}, the FC network has width $512$ in each of the $4$ hidden layers, and the CNN uses circular padding and has $256$ channels in each of the $4$ hidden layers.
For CIFAR-10 and MNIST images, we use standard data pre-processing, i.e., normalizing each image to have zero mean and unit variance.
To ensure the initial outputs are always $0$, we subtract the function output at initialization for each datapoint (as discussed in \Cref{sec:two-layer}).
We train and test using the $\ell_2$ loss with $\pm1$ labels.
We use vanilla stochastic gradient descent with batch size $500$, and choose a small learning rate (roughly $\frac{0.01}{\norm{\mathrm{NTK}}}$) so that we can better observe early time of training (similar to~\citet{nakkiran2019sgd}). 


We use the Neural Tangents Library~\citep{novak2019neural} and JAX~\citep{jax2018github} for our experiments.

\begin{figure}[h]
\begin{center}
\begin{subfigure}[b]{0.33\textwidth}
                \centering
               \includegraphics[width=1\textwidth]{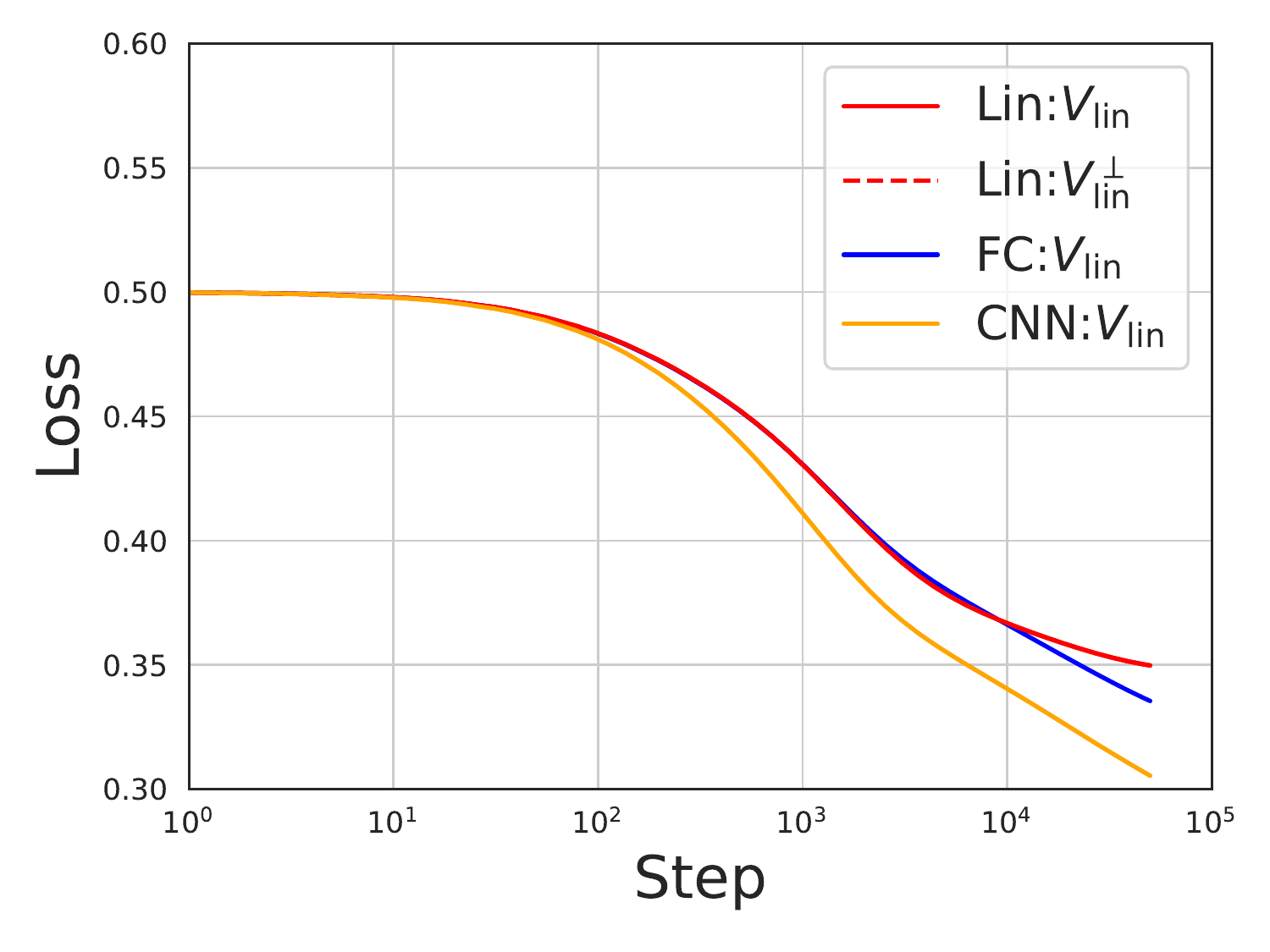}
                \caption{Test loss}
                \label{fig:cifar-full-size-decompose-loss}
        \end{subfigure}%
         \begin{subfigure}[b]{.33\textwidth}
                \centering
\includegraphics[width=1.\textwidth]{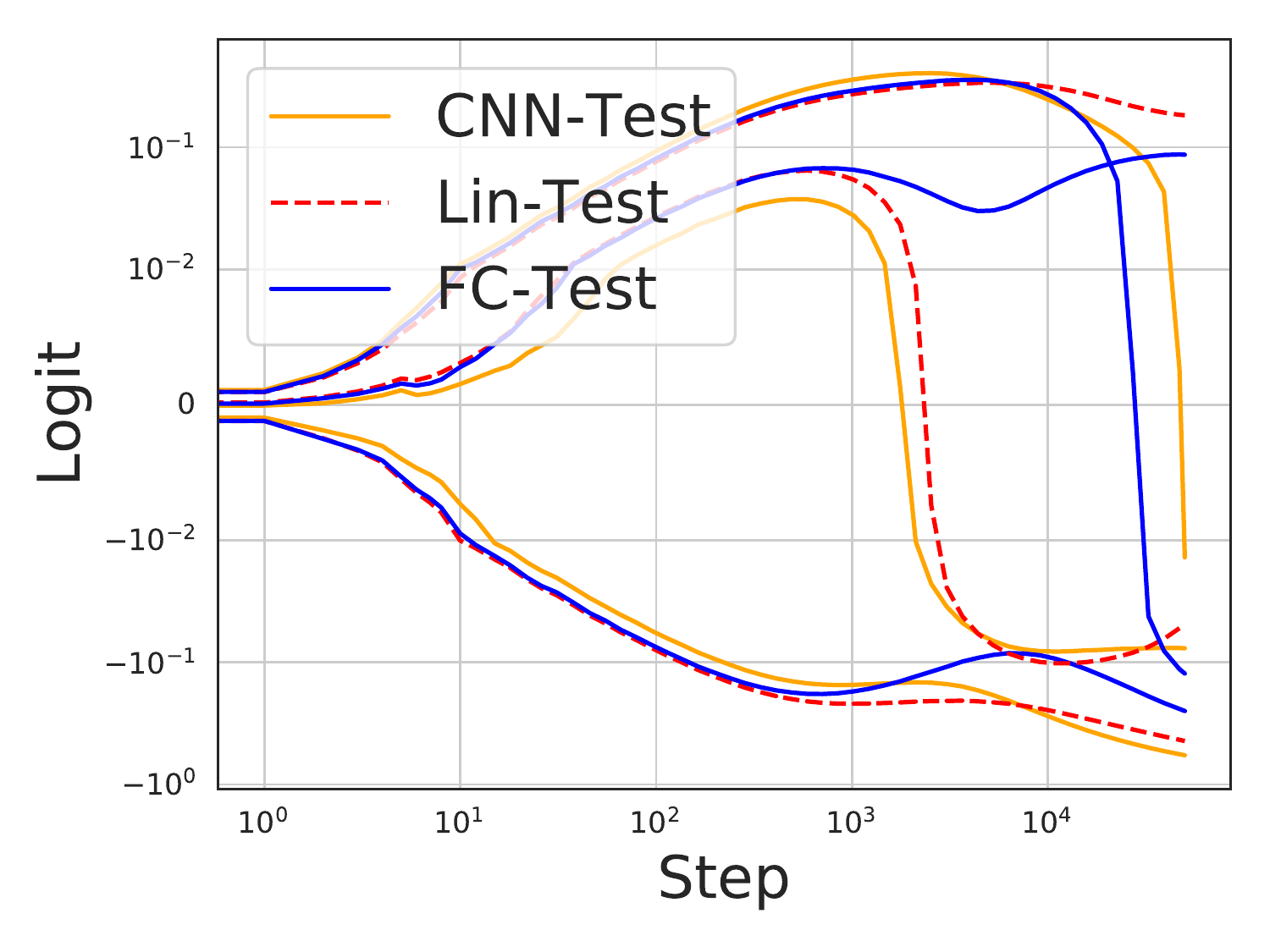}
                \caption{Test logits}
        \end{subfigure}
         \begin{subfigure}[b]{.33\textwidth}
                \centering
\includegraphics[width=1.\textwidth]{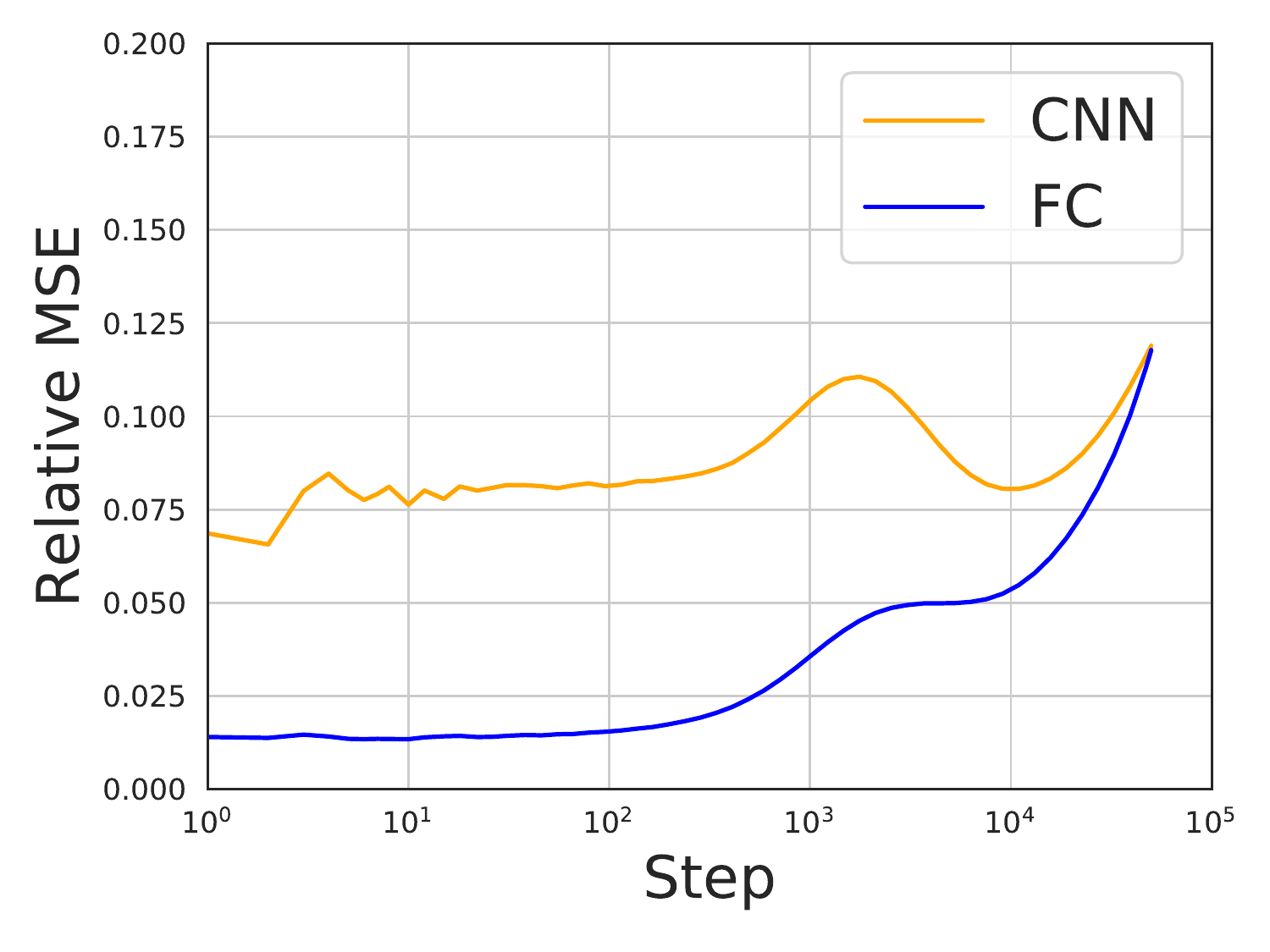}
                \caption{Relative MSE of discrepancy}
        \end{subfigure}
        \\
        \begin{subfigure}[b]{0.33\textwidth}
                \centering
               \includegraphics[width=1\textwidth]{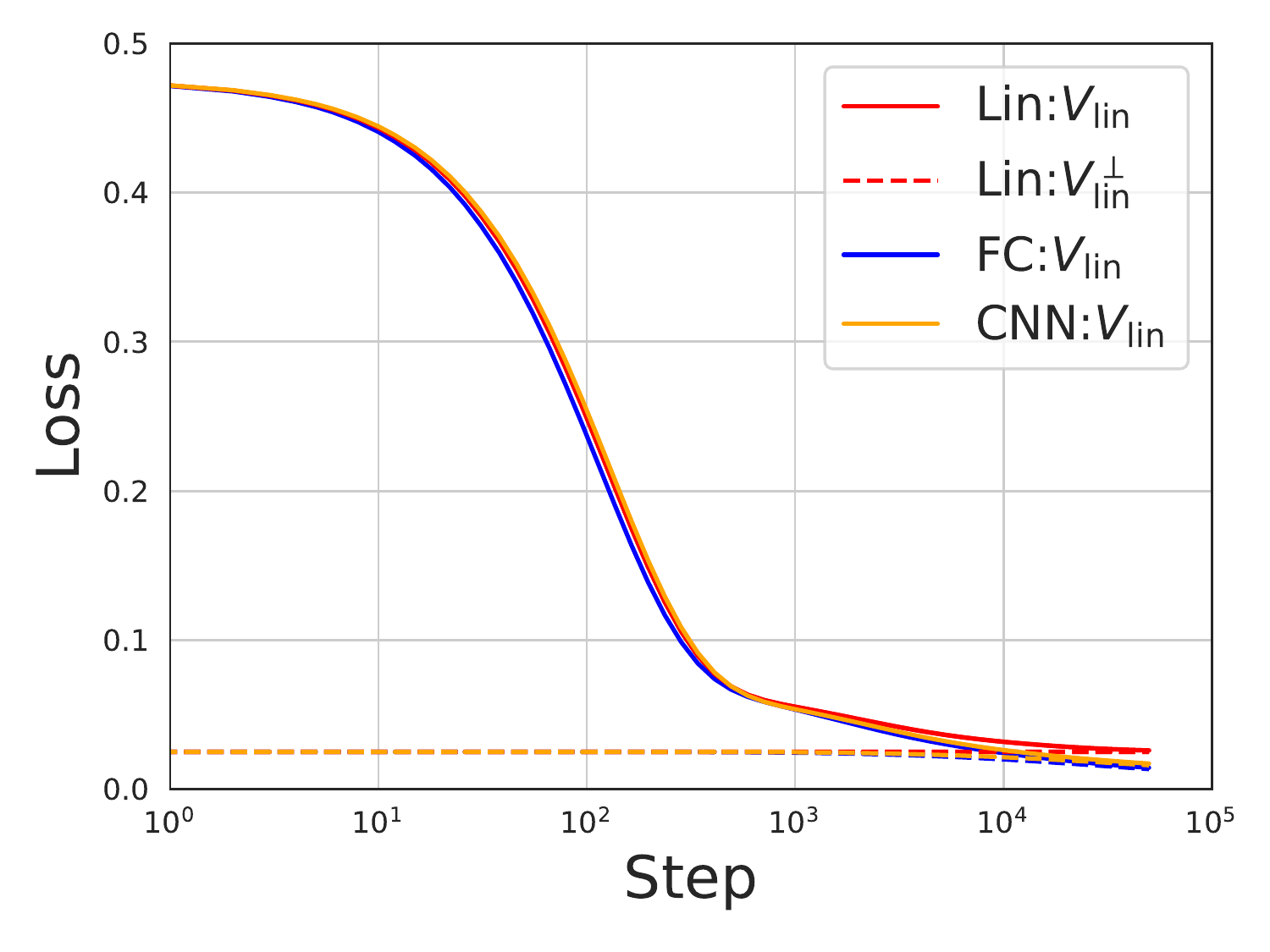}
                \caption{Test loss decomposition}
        \end{subfigure}%
         \begin{subfigure}[b]{.33\textwidth}
                \centering
\includegraphics[width=1.\textwidth]{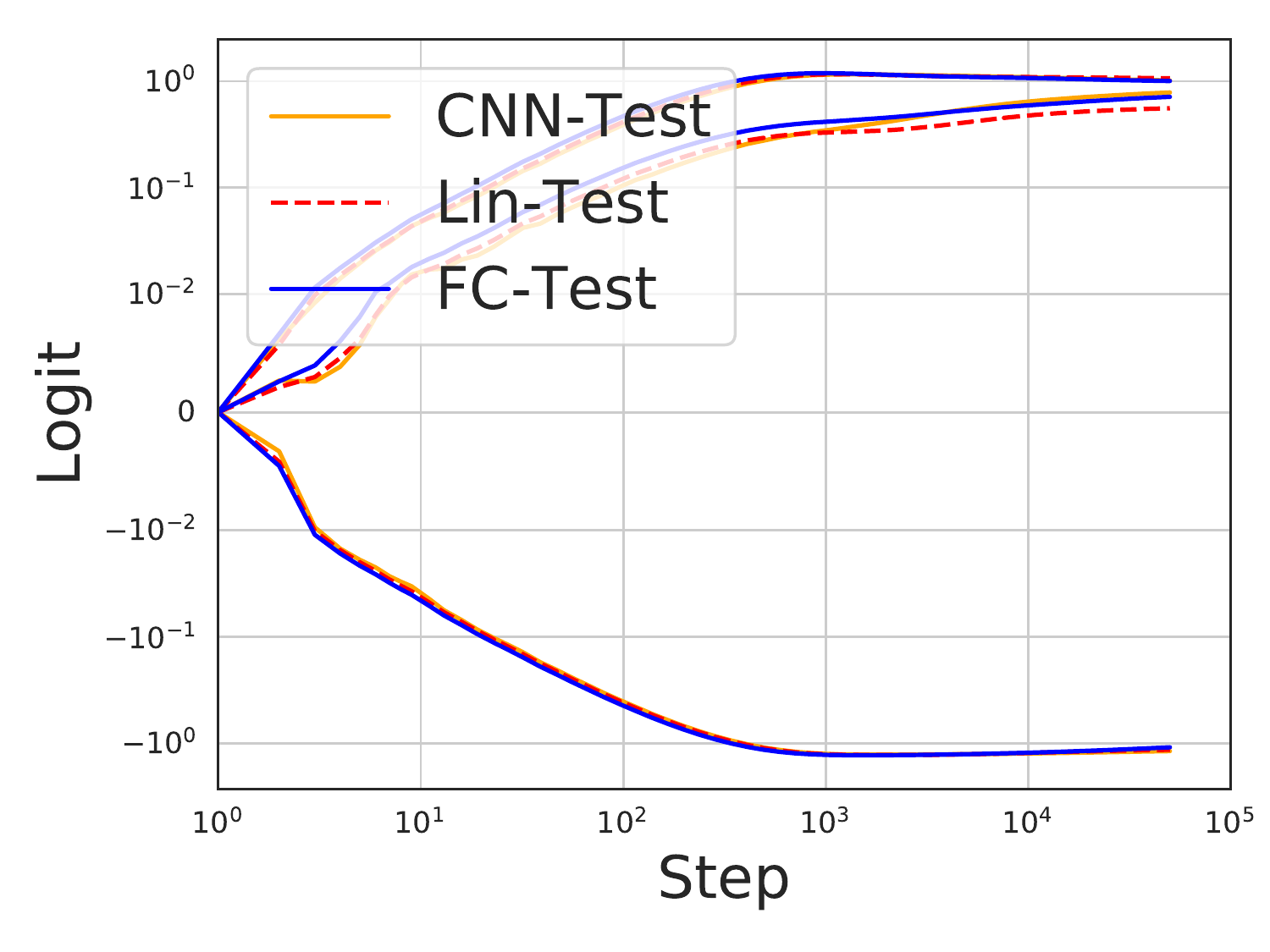}
                \caption{Test logits}
        \end{subfigure}
         \begin{subfigure}[b]{.33\textwidth}
                \centering
\includegraphics[width=1.\textwidth]{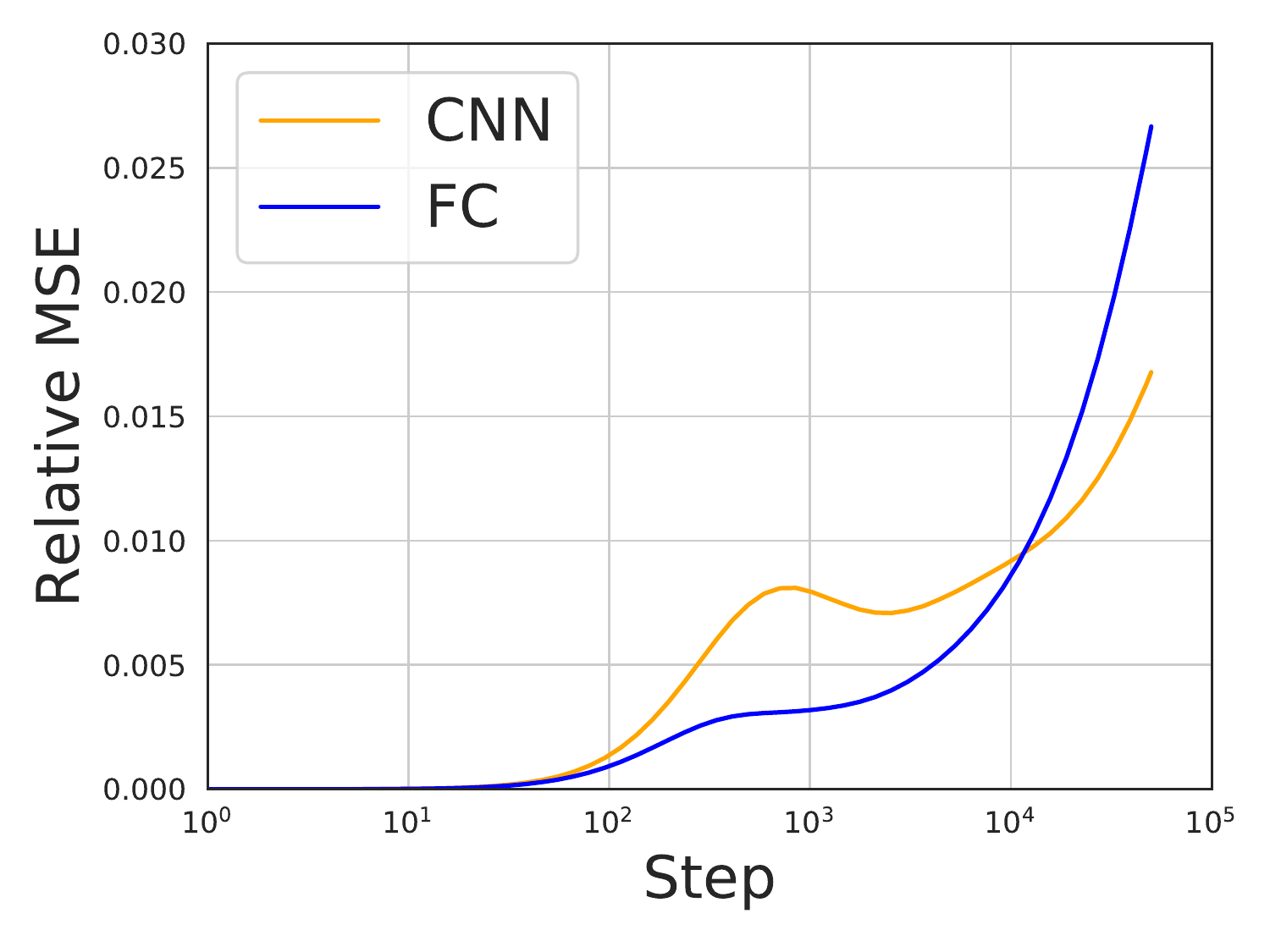}
                \caption{Relative MSE of discrepancy}
        \end{subfigure}
\end{center}
\caption{Replication of \Cref{fig:cifar-main-text} on full-size CIFAR-10 (top row) and MNIST (bottom row).
	In \Cref{fig:cifar-full-size-decompose-loss}, there is no projection onto $V_{\textrm{lin}}^\perp$ because the data dimension $32\times32\times3$ is larger than the number of test data 2,000.
}
\label{fig:meta-plots}
\end{figure}

\section{Additional Notation and Lemmas} \label{app:notation}

We introduce some additional notation and lemmas that will be used in the proofs.

We use $\tilde{O}(\cdot)$ to hide poly-logarithmic factors in $n$ (the number of training datapoints). Denote by $\ind{E}$ the indicator function for an event $E$.
For a vector $\va$, we let $\diag(\va)$ be a diagonal matrix whose diagonal entries constitute $\va$.
For a matrix $\mA$, we use $\vect{\mA}$ to denote the vectorization of $\mA$ in row-first order.

For a square matrix $\mA$, we denote its diagonal and off-diagonal parts as $\mA_\diag$ and $\mA_\offdiag$, respectively.
Namely, we have $\mA = \mA_{\diag} + \mA_{\offdiag}$, where $\index{\mA_{\diag}}{i, j} = \index{\mA}{i, j} \ind{i=j}$ and $\index{\mA_{\offdiag}}{i, j} = \index{\mA}{i, j} \ind{i\not=j}$.
Equivalently, $\mA_\diag = \mA \odot \mI$ and $\mA_\offdiag = \mA \odot (\vone\vone^\top - \mI)$.

\begin{lem} \label{lem:submatrix-spectral-norm}
    For any matrix $\mA$ and a submatrix $\mA_1$ of $\mA$, we have $\norm{\mA_1}\le\norm{\mA}$.
\end{lem}
\begin{proof}
    For simplicity we assume that $\mA_1$ is in the top-left corner of $\mA$, i.e. $\mA = \begin{bmatrix} \mA_1 & \mA_2\\ \mA_3&\mA_4\end{bmatrix}$. The same proof works when $\mA_1$ is any other submatrix of $\mA$.
    
    By the definition of spectral norm, we have
    \begin{align*}
        \norm{\mA} &= \max_{\norm{\vx}=\norm{\vy}=1} \vx^\top\mA\vy\\
        &= \max_{\norm{\vx}=\norm{\vy}=1} \vx^\top \begin{bmatrix} \mA_1 & \mA_2\\ \mA_3&\mA_4\end{bmatrix} \vy \\
        &\ge \max_{\norm{\vx_1}=\norm{\vy_1}=1} [\vx_1^\top, \vzero^\top] \begin{bmatrix} \mA_1 & \mA_2\\ \mA_3&\mA_4\end{bmatrix} \begin{bmatrix}\vy_1\\\vzero \end{bmatrix}\\
        &= \max_{\norm{\vx_1}=\norm{\vy_1}=1} \vx_1^\top\mA_1\vy_1\\
        &= \norm{\mA_1}. \qedhere
    \end{align*}
\end{proof}

\begin{lem} \label{lem:diag-and-offdiag-spectral-norm}
    For any square matrix $\mA$, we have $\norm{\mA_\diag} \le \norm{\mA}$ and $\norm{\mA_\offdiag} \le 2\norm{\mA}$.
\end{lem}
\begin{proof}
    From Lemma~\ref{lem:submatrix-spectral-norm} we know that $\abs{\index{\mA}{i,i}} \le \norm{\mA}$ for all $i$ since $\index{\mA}{i, i}$ can be viewed as a submatrix of $\mA$.
    Thus we have
    \begin{align*}
        \norm{\mA_\diag} = \max_i \abs{\index{\mA}{i,i}} \le \norm{\mA}. 
    \end{align*}
    It follows that
    \begin{equation*}
        \norm{\mA_{\offdiag}} = \norm{\mA - \mA_\diag} \le \norm{\mA} + \norm{\mA_\diag} \le 2\norm{\mA} . \qedhere
    \end{equation*}
\end{proof}

\begin{lem}[\citet{schur1911bemerkungen}] \label{lem:hadamard-product-bound}
    For any two positive semidefinite matrices $\mA,\mB$, we have $$\norm{\mA\odot\mB} \le \norm{\mA} \cdot \max_i \index{\mB}{i,i}.$$
\end{lem}

\section{General Result on the Closeness between Two Dynamics}
\label{app:general-closeness}

We present a general result that shows how the GD trajectory for a non-linear least squares problem can be simulated by a linear one. Later we will specialize this result to the settings considered in the paper.

We consider an objective function of the form:
\begin{align*} 
    F(\vtheta) = \frac{1}{2n} \norm{\vf(\vtheta) - \vy}^2,
\end{align*}
where $\vf: \R^N \mapsto \R^n$ is a general differentiable function, and $\vy\in\R^n$ satisfies $\norm{\vy} \le \sqrt{n}$.
We denote by $\mJ: \R^N \mapsto \R^{n\times N}$ the Jacobian map of $\vf$.
Then starting from some $\vtheta(0) \in \R^N$, the GD updates for minimizing $F$ can be written as:
\begin{align*}
    \vtheta(t+1) = \vtheta(t) - \eta \nabla F(\vtheta(t))
    = \vtheta(t) - \frac{1}{n}\eta \mJ(\vtheta(t))^\top (\vf(\vtheta(t))-\vy) .
\end{align*}

Consider another linear least squares problem:
\begin{align*}
    G(\vomega) = \frac{1}{2n} \norm{\mPhi \vomega - \vy}^2 ,
\end{align*}
where $\mPhi \in \R^{n\times M}$ is a fixed matrix.
Its GD dynamics started from $\vomega(0)\in\R^M$ can be written as:
\begin{align*}
    \vomega(t+1) = \vomega(t) - \eta \nabla G(\vomega(t))
    = \vomega(t) - \frac{1}{n}\eta \mPhi^\top (\mPhi\vomega(t)-\vy) .
\end{align*}
Let $\mK := \mPhi\mPhi^\top$, and let
\begin{align*}
    \vu(t) &:= \vf(\vtheta(t)),\\
    \vulin(t) &:= \mPhi\vomega(t),
\end{align*}
which stand for the predictions of these two models at iteration $t$.

The linear dynamics admit a very simple analytical form, summarized below.

\begin{claim} \label{claim:general-linear-dynamics}
    For all $t\ge0$ we have
    $
        \vulin(t) - \vy = \left(\mI - \frac{1}{n}\eta \mK\right)^t (\vulin(0) - \vy)
    $.
    As a consequence, if $\eta \le \frac{2n}{\norm{\mK}}$, then we have $\norm{\vulin(t) - \vy}\le\norm{\vulin(0) - \vy}$ for all $t\ge0$.
\end{claim}
\begin{proof}
    By definition we have
    $
        \vulin(t+1) = \vulin(t) - \frac1n\eta \mK (\vulin(t)-\vy)
    $,
    which implies
    $
        \vulin(t+1) - \vy = \left(\mI - \frac{1}{n}\eta \mK\right) (\vulin(t) - \vy)
    $.
    Thus the first statement follows directly.
    Then the second statement can be proved by noting that $\norm{\mI - \frac{1}{n}\eta \mK}\le 1$ when $\eta \le \frac{2n}{\norm{\mK}}$.
\end{proof}

We make the following assumption that connects these two problems:
\begin{asmp} \label{asmp:general-jacobian-closeness}
    There exist $0<\epsilon<\norm{\mK}, R>0$ such that for any $\vtheta, \vtheta' \in \R^N$, as long as $\norm{\vtheta-\vtheta(0)}\le R$ and $\norm{\vtheta' - \vtheta(0)}\le R$, we have
    \begin{align*}
        \norm{\mJ(\vtheta)\mJ(\vtheta')^\top - \mK} \le \epsilon.
    \end{align*}
\end{asmp}

Based on the above assumption, we have the following theorem showing the agreement between $\vu(t)$ and $\vulin(t)$ as well as the parameter boundedness in early time.


\begin{thm}\label{thm:general-closeness}
    Suppose that the initializations are chosen so that $\vu(0)=\vulin(0) = \vzero$, and that the learning rate satisfies $\eta \le \frac{n}{\norm{\mK}}$.
    Suppose that Assumption~\ref{asmp:general-jacobian-closeness} is satisfied with $R^2\epsilon < n$.
    Then there exists a universal constant $c>0$ such that for all $0\le t \le c \frac{R^2}{\eta}$:
    \begin{itemize}
        \item (closeness of predictions) $\norm{\vu(t) - \vulin(t)} \lesssim \frac{\eta t\epsilon}{\sqrt{n}}$;
        \item (boundedness of parameter movement) $\norm{\vtheta(t) - \vtheta(0)}\le R, \norm{\vomega(t) - \vomega(0)}\le R$.
    \end{itemize}
\end{thm}
\begin{proof}
    We first prove the first two properties, and will prove the last property $\norm{\vomega(t) - \vomega(0)}\le R$ at the end.
    
    We use induction to prove $\norm{\vu(t) - \vulin(t)} \lesssim \frac{\eta t\epsilon}{\sqrt{n}}$ and $\norm{\vtheta(t) - \vtheta(0)}\le R$.
    For $t=0$, these statements are trivially true.
    Now suppose for some $1\le t \le c\frac{R^2}{\eta}$ we have $\norm{\vu(\tau) - \vulin(\tau)} \lesssim \frac{\eta \tau\epsilon}{\sqrt{n}}$ and $\norm{\vtheta(\tau) - \vtheta(0)}\le R$ for $\tau=0,1,\ldots,t-1$.
    We will now prove $\norm{\vu(t) - \vulin(t)} \lesssim \frac{\eta t\epsilon}{\sqrt{n}}$ and $\norm{\vtheta(t) - \vtheta(0)}\le R$ under these induction hypotheses.
    
    Notice that from $\norm{\vu(\tau)-\vulin(\tau)} \lesssim \frac{\eta\tau\epsilon}{\sqrt{n}} \le \frac{cR^2\epsilon}{\sqrt{n}} \lesssim \sqrt{n}$ and Claim~\ref{claim:general-linear-dynamics} we know $\norm{\vu(\tau)-\vy} \lesssim \sqrt{n}$ for all $\tau<t$.
    
    \paragraph{Step 1: proving $\norm{\vtheta(t) - \vtheta(0)}\le R$.}
    We define
    \begin{align*}
        \mJ(\vtheta\to \vtheta') := \int_0^1 \mJ(\vtheta + x(\vtheta'-\vtheta)) dx.
    \end{align*}
    
    We first prove $\norm{\vtheta(t-1)-\vtheta(0)} \le \frac{R}{2}$.
    If $t=1$, this is trivially true.
    Now we assume $t\ge2$.
    For each $0\le\tau<t-1$, by the fundamental theorem for line integrals we have
    \begin{align*}
        \vu(\tau+1) - \vu(\tau) &= \mJ(\vtheta(\tau)\to \vtheta(\tau+1)) \cdot (\vtheta(\tau+1) - \vtheta(\tau)) \\
        &= -\frac{\eta}{n}\mJ(\vtheta(\tau)\to \vtheta(\tau+1)) \mJ(\vtheta(\tau))^\top (\vu(\tau)-\vy) .
    \end{align*}
    Let $\mE(\tau) := \mJ(\vtheta(\tau)\to \vtheta(\tau+1)) \mJ(\vtheta(\tau))^\top - \mK$.
    Since $\norm{\vtheta(\tau)-\vtheta(0)}\le R$ and $\norm{\vtheta(\tau+1)-\vtheta(0)}\le R$, from Assumption~\ref{asmp:general-jacobian-closeness} we know that $\norm{\mE(\tau)}\le\epsilon$.
    We can write
    \begin{equation} \label{eqn:prediction-dynamics}
    \begin{aligned}
        \vu(\tau+1)-\vy &= \left( \mI - \frac{\eta}{n} \mJ(\vtheta(\tau)\to \vtheta(\tau+1)) \mJ(\vtheta(\tau))^\top  \right) (\vu(\tau)-\vy) \\
        &= \left( \mI - \frac{\eta}{n}\mK \right)(\vu(\tau)-\vy) - \frac{\eta}{n}\mE(\tau)(\vu(\tau)-\vy) .
    \end{aligned} 
    \end{equation}
    It follows that
    \begin{align*}
        &\norm{\vu(\tau+1)-\vy}^2\\
        \le\,& \norm{\left( \mI - \frac{\eta}{n}\mK \right)(\vu(\tau)-\vy)}^2 + 2\norm{\left( \mI - \frac{\eta}{n}\mK \right)(\vu(\tau)-\vy)}\cdot\norm{\frac{\eta}{n}\mE(\tau)(\vu(\tau)-\vy)}\\&\quad+ \norm{\frac{\eta}{n}\mE(\tau)(\vu(\tau)-\vy)}^2\\
        \le\,& \norm{\left( \mI - \frac{\eta}{n}\mK \right)(\vu(\tau)-\vy)}^2 + O\left( \sqrt{n}\cdot\frac{\eta}{n}\epsilon\sqrt{n} + \left(\frac{\eta}{n}\epsilon\sqrt{n}\right)^2\right)\\
        =\,& \norm{\left( \mI - \frac{\eta}{n}\mK \right)(\vu(\tau)-\vy)}^2 + O(\eta\epsilon) \tag{$\eta\epsilon\lesssim n$} \\
        =\,& \norm{\vu(\tau)-\vy}^2 - \frac{2\eta}{n} (\vu(\tau)-\vy)^\top\mK(\vu(\tau)-\vy) + \frac{\eta^2}{n^2} \norm{\mK(\vu(\tau)-\vy)}^2 + O(\eta\epsilon) \\
        \le\,& \norm{\vu(\tau)-\vy}^2 - \frac{2\eta}{n} (\vu(\tau)-\vy)^\top\mK(\vu(\tau)-\vy) + \frac{\eta^2}{n^2}\norm{\mK}\cdot \norm{\mK^{1/2}(\vu(\tau)-\vy)}^2 + O(\eta\epsilon) \\
        \le\,& \norm{\vu(\tau)-\vy}^2 - \frac{\eta}{n} (\vu(\tau)-\vy)^\top\mK(\vu(\tau)-\vy) + O(\eta\epsilon)
        \tag{$\frac{\eta^2\norm{\mK}}{n^2}\le \frac{\eta}{n}$}.
    \end{align*}
    On the other hand, we have
    \begin{equation} \label{eqn:one-step-param-movement}
    \begin{aligned}
        &\norm{\vtheta(\tau+1)-\vtheta(\tau)}^2\\
        =\,& \frac{\eta^2}{n^2} \norm{\mJ(\vtheta(\tau))^\top(\vu(\tau)-\vy)}^2\\
        =\,& \frac{\eta^2}{n^2} (\vu(\tau)-\vy)^\top \mJ(\vtheta(\tau))\mJ(\vtheta(\tau))^\top (\vu(\tau)-\vy)\\
        \le\,& \frac{\eta^2}{n^2} \left( (\vu(\tau)-\vy)^\top \mK (\vu(\tau)-\vy) + \norm{\vu(\tau)-\vy}^2\norm{\mJ(\vtheta(\tau))\mJ(\vtheta(\tau))^\top-\mK}\right) \\
        \le\,& \frac{\eta^2}{n^2} \left( (\vu(\tau)-\vy)^\top \mK (\vu(\tau)-\vy) + O(n\epsilon) \right).
    \end{aligned}
    \end{equation}
    Combining the above two inequalities, we obtain
    \begin{align*}
        &\norm{\vu(\tau+1)-\vy}^2 - \norm{\vu(\tau)-\vy}^2\\
        \le\,& - \frac{\eta}{n} \left( \frac{n^2}{\eta^2}\norm{\vtheta(\tau+1)-\vtheta(\tau)}^2 - O(n\epsilon) \right) + O(\eta\epsilon)\\
        =\,& -\frac{n}{\eta}\norm{\vtheta(\tau+1)-\vtheta(\tau)}^2 + O(\eta\epsilon).
    \end{align*}
    Taking sum over $\tau=0,\ldots,t-2$, we get
    \begin{align*}
        \norm{\vu(t-1)-\vy}^2 - \norm{\vu(0)-\vy}^2
        \le -\frac{n}{\eta}\sum_{\tau=0}^{t-2}\norm{\vtheta(\tau+1)-\vtheta(\tau)}^2 + O(\eta t\epsilon),
    \end{align*}
    which implies
    \begin{align*}
        \frac{n}{\eta}\sum_{\tau=0}^{t-2}\norm{\vtheta(\tau+1)-\vtheta(\tau)}^2
        \le \norm{\vy}^2 + O(\eta t\epsilon)
        \le \norm{\vy}^2 + O(R^2\epsilon)= O(n).
    \end{align*}
    Then by the Cauchy-Schwartz inequality we have
    \begin{align*}
        \norm{\vtheta(t-1) - \vtheta(0)}
        &\le \sum_{\tau=0}^{t-2}\norm{\vtheta(\tau+1)-\vtheta(\tau)}
        \le \sqrt{(t-1)\sum_{\tau=0}^{t-2}\norm{\vtheta(\tau+1)-\vtheta(\tau)}^2} \\
        &\le \sqrt{t \cdot O(\eta)}
        \le \sqrt{c\frac{R^2}{\eta} \cdot O(\eta)}.
    \end{align*}
    Choosing $c$ sufficiently small, we can ensure $\norm{\vtheta(t-1)-\vtheta(0)} \le \frac{R}{2}$.
    
    Now that we have proved $\norm{\vtheta(t-1)-\vtheta(0)} \le \frac{R}{2}$, to prove $\norm{\vtheta(t)-\vtheta(0)} \le R$ it suffices to bound the one-step deviation $\norm{\vtheta(t)-\vtheta(t-1)}$ by $\frac R2$. Using the exact same method in~\eqref{eqn:one-step-param-movement}, we have
    \begin{align*}
        \norm{\vtheta(t)-\vtheta(t-1)}
        \le \frac{\eta}{n}\sqrt{n\norm{\mK} + O(n\epsilon)}
        \lesssim \eta \sqrt{\norm{\mK}/n}
        = \sqrt{\eta\norm{\mK}/n} \sqrt{\eta}
        \le \sqrt{c} R,
    \end{align*}
    where we have used $\eta \le \frac{n}{\norm{\mK}}$ and $\eta \le \eta t \le cR^2$.
    Choosing $c$ sufficiently small, we can ensure $\norm{\vtheta(t)-\vtheta(t-1)}\le \frac{R}{2}$.
    Therefore we conclude that $\norm{\vtheta(t)-\vtheta(0)}\le R$.
    
    \paragraph{Step 2: proving $\norm{\vu(t) - \vulin(t)} \lesssim \frac{\eta t\epsilon}{\sqrt{n}}$.}
    Same as~\eqref{eqn:prediction-dynamics} we have
    \begin{align*}
        \vu(t)-\vy = \left( \mI - \frac{\eta}{n}\mK \right)(\vu(t-1)-\vy) - \frac{\eta}{n}\mE(t-1)(\vu(t-1)-\vy),
    \end{align*}
    where $\mE(t-1) = \mJ(\vtheta(t-1), \vtheta(t)) \mJ(\vtheta(t-1))^\top - \mK$.
    Since $\norm{\vtheta(t-1)-\vtheta(0)}\le R$ and $\norm{\vtheta(t)-\vtheta(0)}\le R$, we know from Assumption~\ref{asmp:general-jacobian-closeness} that $\norm{\mE(t-1)}\le\epsilon$.
    Moreover, from Claim~\ref{claim:general-linear-dynamics} we know
    \begin{align*}
        \vulin(t)-\vy = \left( \mI - \frac{\eta}{n}\mK \right)(\vulin(t-1)-\vy).
    \end{align*}
    It follows that
    \begin{align*}
        \vu(t) - \vulin(t) = \left( \mI - \frac{\eta}{n}\mK \right)(\vu(t-1)-\vulin(t-1)) - \frac{\eta}{n}\mE(t-1)(\vu(t-1)-\vy),
    \end{align*}
    which implies
    \begin{align*}
        \norm{\vu(t) - \vulin(t)} &\le \norm{\left( \mI - \frac{\eta}{n}\mK \right)(\vu(t-1)-\vulin(t-1))} + \norm{\frac{\eta}{n}\mE(t-1)(\vu(t-1)-\vy)}\\
        &\le \norm{\vu(t-1)-\vulin(t-1)} + O\left(\frac{\eta}{n}\epsilon\sqrt{n}\right)\\
        &= \norm{\vu(t-1)-\vulin(t-1)} + O\left(\frac{\eta\epsilon}{\sqrt{n}}\right).
    \end{align*}
    Therefore from $\norm{\vu(t-1) - \vulin(t-1)} \lesssim \frac{\eta (t-1)\epsilon}{\sqrt{n}}$ we know $\norm{\vu(t) - \vulin(t)} \lesssim \frac{\eta t\epsilon}{\sqrt{n}}$,
    completing the proof.
    
    Finally, we prove the last statement in the theorem, i.e., $\norm{\vomega(t) - \vomega(0)}\le R$.
    In fact we have already proved this -- notice that we have proved $\norm{\vtheta(t)-\vtheta(0)}\le R$ and that a special instance of this problem is when $\vtheta(t)=\vomega(t)$, i.e., the two dynamics are the same. Applying our result on that problem instance, we obtain $\norm{\vomega(t) - \vomega(0)}\le R$.
\end{proof}

\section{Omitted Details in Section~\ref{sec:two-layer}} \label{app:two-layer}

In \Cref{app:two-layer-both-layers-guarantee}, we present the formal theoretical guarantee (Theorem~\ref{thm:both-layers-main}) for the case of training both layers.

In \Cref{app:two-layer-jacobian-NTK}, we calculate the formulae of various Jacobians and NTKs that will be used in the analysis.

In \Cref{app:proof-first-layer}, we prove Theorem~\ref{thm:first-layer-main} (training the first layer).

In \Cref{app:proof-first-layer-well-condition}, we prove Corollary~\ref{cor:first-layer-well-condioned} (training the first layer with well-conditioned data).

In \Cref{app:proof-second-layer}, we prove Theorem~\ref{thm:second-layer-main} (training the second layer).

In \Cref{app:proof-both-layers}, we prove Theorem~\ref{thm:both-layers-main} (training both layers).

In \Cref{app:two-layer-data-concentration-proof}, we prove Claim~\ref{claim:data-concentration} (data concentration properties).

\subsection{Guarantee for Training Both Layers} \label{app:two-layer-both-layers-guarantee}

Now we state our guarantee for the case of training both layers, continuing from \Cref{subsec:both-layers}.
Recall that the neural network weights $(\mW(t),\vv(t))$ are updated according to GD~\eqref{eqn:two-layer-gd} with learning rate $\eta_1=\eta_2=\eta$.
The linear model $\flin(\vx;\delta)$ in~\eqref{eqn:both-layer-linear-model} is also trained with GD:
\begin{align*}
    \vdelta(0) = \vzero_{d+2}, \quad \vdelta(t+1) = \vdelta(t) - \eta \nabla_\vdelta \frac{1}{2n}\sum_{i=1}^n (\flin(\vx_i;\vdelta(t)) - y_i)^2.
\end{align*}
We let $f_t$ and $\flin_t$ be the neural network and the linear model at iteration $t$, i.e., $f_t(\vx):=f(\vx;\mW(t),\vv(t))$ and $\flin_t(\vx):=\flin(\vx;\vdelta(t))$.

\begin{thm}[main theorem for training both layers] \label{thm:both-layers-main}
    Let $\alpha \in (0, \frac14)$ be a fixed constant.
    Suppose
    $ n\gtrsim d^{1+\alpha} $
    and $m \gtrsim d^{2+\alpha}$. 
    Suppose $\begin{cases}
    \eta \ll d/\log n, \text{ if }\,\E[\phi(g)] = 0\\
    \eta \ll 1, \ \ \ \ \ \ \ \ \ \ \ \text{ otherwise}
    \end{cases}$.
    Then there exists a universal constant $c>0$ such that with high probability, for all $0\le t \le T=c\cdot \frac{ d\log d}{\eta}$ simultaneously, we have 
    \begin{align*} 
        \frac1n \sum_{i=1}^n \left( f_t(\vx_i) - \flin_t(\vx_i) \right)^2 \lesssim d^{-\Omega(\alpha)}, \quad \E_{\vx\sim\D}\left[ \min\{ (f_t(\vx) - \flin_t(\vx))^2, 1\} \right] \lesssim d^{-\Omega(\alpha)} + \sqrt{\tfrac{\log T}{n}} .
    \end{align*}
\end{thm}

We remark that if the data distribution is well-conditioned, we can also have a guarantee similar to Corollary~\ref{cor:first-layer-well-condioned}.

\subsection{Formulae of Jacobians and NTKs} \label{app:two-layer-jacobian-NTK}

We first calculate the Jacobian of the network outputs at the training data $\mX$ with respect to the weights in the network.
The Jacobian for the first layer is:
\begin{equation} \label{eqn:first-layer-jacobian}
    \mJ_1(\mW,\vv) := \left[ \mJ_1(\vw_1,v_1), \mJ_1(\vw_2,v_2), \ldots, \mJ_1(\vw_m,v_m) \right] \in \R^{n\times md},
\end{equation}
where
\begin{equation*}
    \mJ_1(\vw_r,v_r) := \frac{1}{\sqrt{md}} v_r  \diag\left(\phi'(\mX\vw_r/\sqrt{d})\right)  \mX \in \R^{n\times d}, \qquad r\in[m].
\end{equation*}
The Jacobian for the second layer is:
\begin{equation} \label{eqn:second-layer-jacobian}
    \mJ_2(\mW) := \frac{1}{\sqrt{m}} \phi(\mX\mW^\top/\sqrt{d}) \in \R^{n\times m}.
\end{equation}
Here we omit $\vv$ in the notation since it does not affect the Jacobian.
The Jacobian for both layers is simply $\mJ(\mW,\vv):= [\mJ_1(\mW,\vv), \mJ_2(\mW)] \in \R^{n\times(md+m)}$.

After calculating the Jacobians, we can calculate the NTK matrices for the first layer, the second layer, and both layers as follows:
\begin{equation} \label{eqn:two-layer-NTKs}
\begin{aligned}
    \mTheta_1(\mW,\vv)
    :=\,& \mJ_1(\mW,\vv) \mJ_1(\mW,\vv)^\top
    = \frac{1}{m} \sum_{r=1}^m v_r^2 \left( \phi'(\mX\vw_r/\sqrt{d})\phi'(\mX\vw_r/\sqrt{d})^\top \right)\odot \frac{\mX\mX^\top}{d}, \\
    \mTheta_2(\mW) :=\,& \mJ_2(\mW)\mJ_2(\mW)^\top
    = \frac{1}{m} \phi(\mX\mW^\top/\sqrt{d})\phi(\mX\mW^\top/\sqrt{d})^\top ,\\
    \mTheta(\mW,\vv):=\,& \mJ(\mW,\vv)\mJ(\mW,\vv)^\top = \mTheta_1(\mW,\vv)+\mTheta_2(\mW) .
\end{aligned}
\end{equation}
We also denote the expected NTK matrices at random initialization as:
\begin{equation} \label{eqn:two-layer-expected-NTKs}
\begin{aligned}
    \mTheta_1^* :=\,&  \E_{\vw\sim\N(\vzero,\mI),v\sim\unif\{\pm1\}}\left[ v^2 \left( \phi'(\mX\vw/\sqrt{d})\phi'(\mX\vw/\sqrt{d})^\top \right) \right] \odot \frac{\mX\mX^\top}{d}\\
    =\,& \E_{\vw\sim\N(\vzero,\mI)}\left[ \left( \phi'(\mX\vw/\sqrt{d})\phi'(\mX\vw/\sqrt{d})^\top \right) \right] \odot \frac{\mX\mX^\top}{d},\\
    \mTheta_2^* :=\,& \E_{\vw\sim\N(\vzero,\mI)} \left[ \phi(\mX\vw/\sqrt{d})\phi(\mX\vw/\sqrt{d})^\top \right], \\
    \mTheta^* :=\,& \mTheta_1^* + \mTheta_2^*.
\end{aligned}
\end{equation}
These are also the NTK matrices at infinite width ($m\to\infty$).

Next, for the three linear models~\eqref{eqn:first-layer-linear-model},~\eqref{eqn:second-layer-linear-model} and~\eqref{eqn:both-layer-linear-model} defined in \Cref{sec:two-layer}, denote their feature/Jacobian matrices by:
\begin{equation} \label{eqn:lin-feature-matrices}
\begin{aligned}
    \mPsi_1 &:= [\vpsi_1(\vx_1), \ldots, \vpsi_1(\vx_n)]^\top, \\
    \mPsi_2 &:= [\vpsi_2(\vx_1), \ldots, \vpsi_2(\vx_n)]^\top, \\
    \mPsi &:= [\vpsi(\vx_1), \ldots, \vpsi(\vx_n)]^\top.
\end{aligned}
\end{equation}
Consequently, their corresponding kernel matrices are:
\begin{equation} \label{eqn:lin-kernels}
\begin{aligned}
    \mThetalinone &:= \mPsi_1\mPsi_1^\top = \frac1d(\zeta^2\mX\mX^\top+\nu^2\vone\vone^\top), \\
    \mThetalintwo &:= \mPsi_2\mPsi_2^\top = \frac1d\left(\zeta^2\mX\mX^\top+\frac12\nu^2\vone\vone^\top\right) + \vq\vq^\top, \\
    \mThetalin &:= \mPsi\mPsi ^\top = \frac1d\left(2\zeta^2\mX\mX^\top+\frac32\nu^2\vone\vone^\top\right) + \vq\vq^\top .
\end{aligned}
\end{equation}
Here the constants are defined in~\eqref{eqn:second-layer-linear-model}, and $\vq \in \R^n$ is defined as $\index{\vq}{i}:= \vartheta_0 + \vartheta_1(\frac{\norm{\vx_i}}{\sqrt d}-1) +  \vartheta_2(\frac{\norm{\vx_i}}{\sqrt d}-1)^2$  for each $i\in[n]$.

\subsection{Proof of Theorem~\ref{thm:first-layer-main} (Training the First Layer)} \label{app:proof-first-layer}

For convenience we let $\vv = \vv(0)$ which is the fixed second layer.
Since we have $v_r\in\{\pm1\}$ ($\forall r\in[m]$), we can write the first-layer NTK matrix as 
\begin{align*}
\mTheta_1(\mW, \vv) = \frac{1}{m} \sum_{r=1}^m \left( \phi'(\mX\vw_r/\sqrt{d})\phi'(\mX\vw_r/\sqrt{d})^\top \right)\odot \frac{\mX\mX^\top}{d}.
\end{align*}
Because it does not depend on $\vv$, we denote $\mTheta_1(\mW) := \mTheta_1(\mW,\vv)$ for convenience.

\subsubsection{The NTK at Initialization}

\begin{figure}[t]
    \centering
    \includegraphics[width=0.6\textwidth]{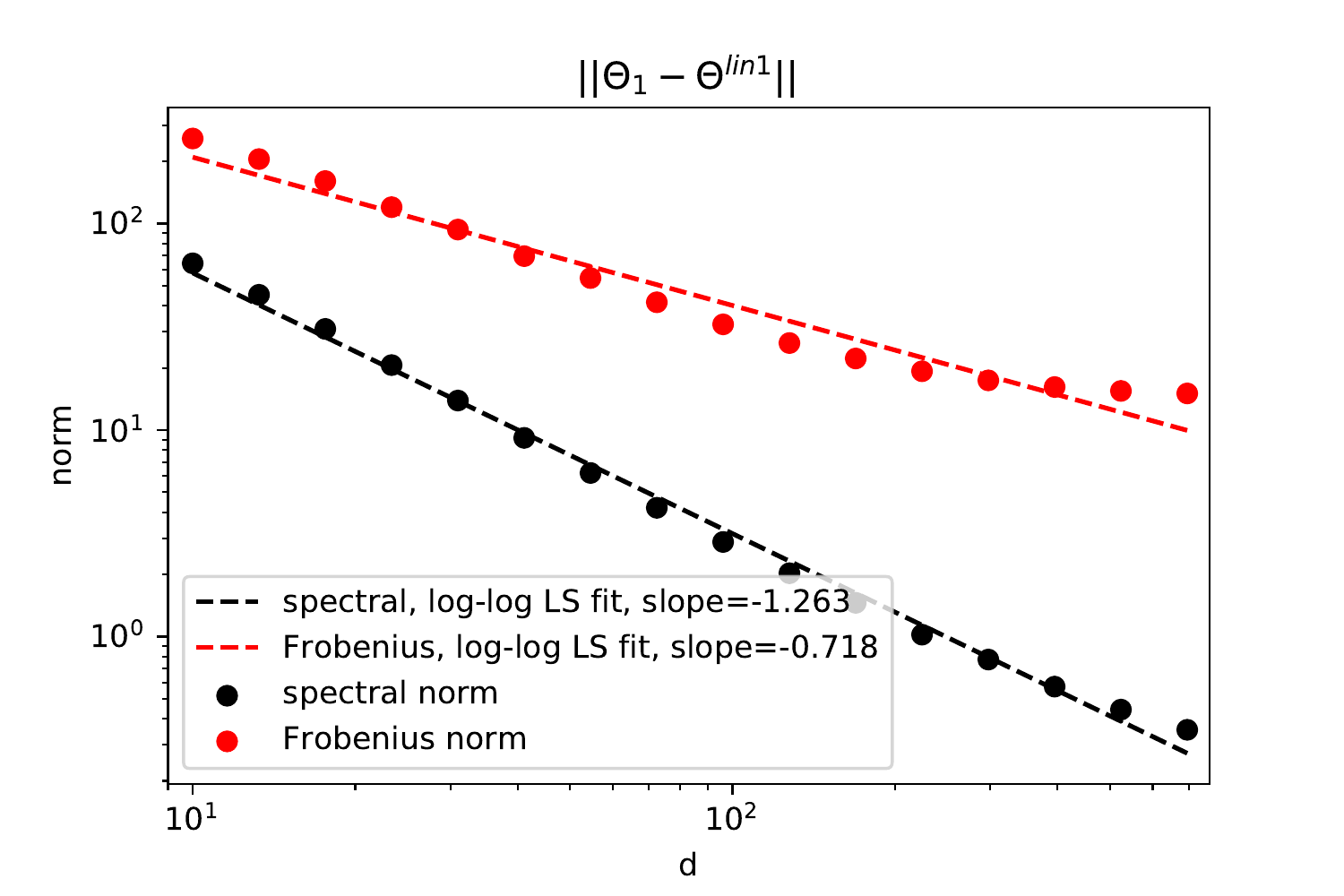}
    \caption{{\bf Verification of Proposition~\ref{prop:first-layer-ntk-init-approx}/\ref{prop:first-layer-ntk-init-approx-copy}.}
    We simulate the dependence of the spectral and Frobenius norms of $\mTheta_1(\mW(0)) - \mThetalinone$ on $d$. We set $\phi=\erf$, $n=10^4$ and  $m=2\times10^4$, and generate data from $\N(\vzero,\mI)$ for various $d$. We perform a linear least-squares fit on the log mean norms against $\log(d)$. Numerically we find $\norm{\mTheta_1(\mW(0)) - \mThetalinone} \propto d^{-1.263}$ and $\norm{\mTheta_1(\mW(0)) - \mThetalinone}_F \propto d^{-0.718}$.}
    \label{fig:spec_norm_decay}
\end{figure}

Now we prove Proposition~\ref{prop:first-layer-ntk-init-approx}, restated below:
\begin{prop}[restatement of Proposition~\ref{prop:first-layer-ntk-init-approx}] \label{prop:first-layer-ntk-init-approx-copy}
    With high probability over the random initialization $\mW(0)$ and the training data $\mX$, we have
    \[
    \norm{\mTheta_1(\mW(0)) - \mThetalinone} \lesssim \frac{n}{d^{1+\alpha}} .
    \]
\end{prop}

We perform a simulation to empirically verify Proposition~\ref{prop:first-layer-ntk-init-approx-copy} in \Cref{fig:spec_norm_decay}. Here we fix $n$ and $m$ to be large and look at the dependence of $\norm{\mTheta_1(\mW(0)) - \mThetalinone}$ on $d$. We find that $\norm{\mTheta_1(\mW(0)) - \mThetalinone}$ indeed decays faster than $\frac1d$. In contrast, $\norm{\mTheta_1(\mW(0)) - \mThetalinone}_F$ decays slower than $\frac{1}{d}$, indicating that bounding the Frobenius norm is insufficient.

To prove Proposition~\ref{prop:first-layer-ntk-init-approx-copy}, we will prove $\mTheta_1(\mW(0))$ is close to its expectation $\mTheta^*_1$ (defined in~\eqref{eqn:two-layer-expected-NTKs}), and then prove $\mTheta^*_1$ is close to $\mThetalinone$. We do these steps in the next two propositions.

\begin{prop} \label{prop:first-layer-ntk-concentration}
    With high probability over the random initialization $\mW(0)$ and the training data $\mX$, we have
    \[
    \norm{\mTheta_1(\mW(0)) - \mTheta_1^*} \le \frac{n}{d^{1+\alpha}} .
    \]
\end{prop}
\begin{proof}
    For convenience we denote $\mW = \mW(0)$ and $\mTheta_1 = \mTheta_1(\mW) = \mTheta_1(\mW(0))$ in this proof.

    From Claim~\ref{claim:data-concentration} we know $\norm{\mX\mX^\top} = O(n)$ with high probability. For the rest of the proof we will be conditioned on $\mX$ and on Claim~\ref{claim:data-concentration}, and only consider the randomness in $\mW$.
    
    We define $\mTheta_1^{(r)} := \left( \phi'(\mX\vw_r/\sqrt{d})\phi'(\mX\vw_r/\sqrt{d})^\top \right)\odot \frac{\mX\mX^\top}{d}$ for each $r\in[m]$.
    Then we have $\mTheta_1 = \frac1m \sum_{r=1}^m \mTheta_1^{(r)}$.
    According to the initialization scheme~\eqref{eqn:two-layer-symm-init}, we know that $\mTheta_1^{(1)}, \mTheta_1^{(2)}, \ldots, \mTheta_1^{(m/2)}$ are independent, $\mTheta_1^{(m/2+1)}, \mTheta_1^{(m/2+2)}, \ldots, \mTheta_1^{(m)}$ are independent, and $\E[\mTheta_1^{(r)}]=\mTheta_1^*$ for all $r\in[m]$.
    
    Next we will apply the matrix Bernstein inequality (Theorem 1.6.2 in~\citet{tropp2015introduction}) to bound $\norm{\mTheta_1 - \mTheta_1^*}$.
    We will first consider the first half of independent neurons, i.e. $r\in[m/2]$.
    For each $r$ we have
    \begin{align*}
        \norm{\mTheta_1^{(r)}} &= \norm{ \diag\left( \phi'(\mX\vw_r/\sqrt{d}) \right) \cdot \frac{\mX\mX^\top}{d} \cdot \diag\left( \phi'(\mX\vw_r/\sqrt{d}) \right) } \\
        &\le \norm{ \diag\left( \phi'(\mX\vw_r/\sqrt{d}) \right)} \cdot \norm{ \frac{\mX\mX^\top}{d} } \cdot \norm{\diag\left( \phi'(\mX\vw_r/\sqrt{d}) \right) } \\
        &\le O(1) \cdot O(n/d) \cdot O(1) \\
        &= O(n/d) .
    \end{align*}
    Here we have used the boundedness of $\phi'(\cdot)$ (Assumption~\ref{asmp:activation}).
    Since $\mTheta^*_1 = \E[\mTheta_1^{(r)}]$, it follows that
    \begin{align*}
        \norm{\mTheta_1^*} &\le O(n/d),\\
        \norm{\mTheta_1^{(r)} - \mTheta_1^*} &\le O(n/d), \qquad \forall r\in[m/2]\\
        \norm{\sum_{r=1}^{m/2} \E[(\mTheta_1^{(r)} - \mTheta_1^*)^2]} &\le \sum_{r=1}^{m/2} \norm{\E[(\mTheta_1^{(r)} - \mTheta_1^*)^2]} \le O(mn^2/d^2) .
    \end{align*}
    Therefore, from the the matrix Bernstein inequality, for any $s\ge0$ we have:
    \begin{align*}
        \Pr\left[ \norm{\sum_{r=1}^{m/2}(\mTheta_1^{(r)}-\mTheta_1^*)} \ge s \right] \le 2n \cdot \exp\left( \frac{-s^2/2}{O(mn^2/d^2 + s n/d )} \right) .
    \end{align*}
    Letting $s = \frac m2 \cdot \frac{n}{d^{1+\alpha}}$, we obtain
    \begin{align*}
        \Pr\left[ \norm{\sum_{r=1}^{m/2}(\mTheta_1^{(r)}-\mTheta_1^*)} \ge  \frac m2 \cdot \frac{n}{d^{1+\alpha}} \right]
        &\le 2n \cdot \exp\left( -\Omega\left( \frac{m^2n^2/d^{2+2\alpha}}{mn^2/d^2 + mn^2/d^{2+\alpha} } \right) \right) \\
        &= 2n \cdot \exp\left( -\Omega\left( \frac{m}{ d^{2\alpha}  } \right) \right) \\
        &= d^{O(1)} \cdot e^{-\Omega(d^{1-\alpha})}\\
        &\ll 1,
    \end{align*}
    where we have used $m = \Omega(d^{1+\alpha})$ and $n= d^{O(1)}$. Therefore with high probability we have
    \[
    \norm{\sum_{r=1}^{m/2}(\mTheta_1^{(r)}-\mTheta_1^*)} \le \frac m2 \cdot \frac{n}{d^{1+\alpha}}.
    \]
    Similarly, for the second half of the neurons we also have with high probability
    \[
    \norm{\sum_{r=m/2+1}^{m}(\mTheta_1^{(r)}-\mTheta_1^*)} \le \frac m2 \cdot \frac{n}{d^{1+\alpha}}.
    \]
    Finally, by the triangle inequality we have
    \begin{align*}
        \norm{\mTheta_1-\mTheta_1^*} = \frac1m \norm{\sum_{r=1}^{m}(\mTheta_1^{(r)}-\mTheta_1^*)}
        \le \frac1m \left(\frac m2 \cdot \frac{n}{d^{1+\alpha}}+ \frac m2 \cdot \frac{n}{d^{1+\alpha}} \right)
        = \frac{n}{d^{1+\alpha}}
    \end{align*}
    with high probability, completing the proof.
\end{proof}

\begin{prop} \label{prop:first-layer-ana-ntk-approx}
    With high probability over the training data $\mX$, we have
    \begin{align*}
        \norm{\mTheta_1^* - \mThetalinone} \lesssim \frac{n}{d^{1+\alpha}}.
    \end{align*}
\end{prop}
\begin{proof}
    We will be conditioned on the high probability events stated in Claim~\ref{claim:data-concentration}.
    
    By the definition of $\mTheta_1^*$, we know
    \begin{align*}
        \index{\mTheta_1^*}{i, j} = \frac1d \vx_i^\top\vx_j \cdot \E_{\vw\sim\N(\vzero,\mI)}\left[  \phi'(\vw^\top\vx_i/\sqrt{d})\phi'(\vw^\top\vx_j/\sqrt{d})^\top  \right], \qquad i, j\in[n].
    \end{align*}
    We define
    \begin{align*}
        \Phi(a, b, c) := \E_{(z_1, z_2)\sim\N\left(\vzero, \mLambda\right)} [\phi'(z_1)\phi'(z_2)], \text{  where } \mLambda = \begin{pmatrix}a\ \ c\\c \ \ b\end{pmatrix}, \quad a\ge0, b\ge0, |c|\le \sqrt{ab} .
    \end{align*}
    Then we can write
    \begin{align*}
        \index{\mTheta_1^*}{i, j} = \frac1d \vx_i^\top\vx_j \cdot \Phi\left( \frac{\norm{\vx_i}^2}{d}, \frac{\norm{\vx_j}^2}{d}, \frac{\vx_i^\top\vx_j}{d} \right) .
    \end{align*}
    
    We consider the diagonal and off-diagonal entries of $\mTheta_1^*$ separately.

    For $i\not=j$, from Claim~\ref{claim:data-concentration} we know $\frac{\norm{\vx_i}^2}{d} = 1\pm \tilde{O}(\frac{1}{\sqrt d})$, $\frac{\norm{\vx_j}^2}{d} = 1\pm \tilde{O}(\frac{1}{\sqrt d})$ and $\frac{\vx_i^\top\vx_j}{d} = \pm \tilde{O}(\frac{1}{\sqrt d})$.
    Hence we apply Taylor expansion of $\Phi$ around $(1, 1, 0)$:
    \begin{align*}
        &\Phi\left( \frac{\norm{\vx_i}^2}{d}, \frac{\norm{\vx_j}^2}{d}, \frac{\vx_i^\top\vx_j}{d} \right)\\
        =\,& \Phi(1, 1, 0) + c_1\left(\frac{\norm{\vx_i}^2}{d}-1\right) + c_2\left(\frac{\norm{\vx_j}^2}{d}-1\right) + c_3\frac{(\vx_i^\top\vx_j)^2}{d} \\ &\pm O\left( \left(\frac{\norm{\vx_i}^2}{d}-1\right)^2 + \left(\frac{\norm{\vx_j}^2}{d}-1\right)^2 + \left(\frac{(\vx_i^\top\vx_j)^2}{d}\right)^2 \right) \\
        =\,& \Phi(1, 1, 0) + c_1\left(\frac{\norm{\vx_i}^2}{d}-1\right) + c_2\left(\frac{\norm{\vx_j}^2}{d}-1\right) + c_3\frac{(\vx_i^\top\vx_j)^2}{d} \pm \tilde{O}\left(\frac{1}{d}\right) .
    \end{align*}
    Here $(c_1, c_2, c_3) := \nabla \Phi(1,1,0)$. Note that $\Phi(1,1,0)$ and all first and second order derivatives of $\Phi$ at $(1,1,0)$ exist and are bounded for activation $\phi$ that satisfies Assumption~\ref{asmp:activation}.
    In particular, we have $\Phi(1,1,0) = \left( \E[\phi'(g)] \right)^2 = \zeta^2$ and $c_3 = \left( \E[g\phi'(g)] \right)^2$.
    Using the above expansion, we can write 
    \begin{equation} 
    \begin{aligned}
        (\mTheta^*_1)_{\offdiag} =\,& \zeta^2\left(\frac{\mX\mX^\top}{d}\right)_\offdiag
        + c_1 \left( \diag(\veps)\cdot \frac{\mX\mX^\top}{d} \right)_\offdiag
        + c_2 \left( \frac{\mX\mX^\top}{d} \cdot \diag(\veps) \right)_\offdiag \\
        & + c_3 \left( \frac{\mX\mX^\top}{d} \odot \frac{\mX\mX^\top}{d} \right)_\offdiag + \mE,
        \label{eqn:first-layer-ntk-offdiag-approx}
        \end{aligned}
    \end{equation}
    where $\veps\in\R^n$ is defined as $\index{\veps}{i} = \frac{\norm{\vx_i}^2}{d}-1$, and $\index{\mE}{i,j}=\pm\tilde{O}(\frac1d)\cdot \frac{\vx_i^\top\vx_j}{d} \ind{i\not=j} = \pm \tilde{O}(\frac{1}{d^{1.5}})$.
    
    Now we treat the terms in~\eqref{eqn:first-layer-ntk-offdiag-approx} separately.
    First, we have
    \begin{align*}
        \norm{\diag(\veps)\cdot \frac{\mX\mX^\top}{d}}
        &\le \norm{\diag(\veps)} \cdot \norm{\frac{\mX\mX^\top}{d}}
        = \max_{i\in[n]} \Big|\index{\veps}{i}\Big| \cdot \norm{\frac{\mX\mX^\top}{d}}
        \le \tilde{O}\left(\frac{1}{\sqrt d}\right) \cdot O\left(\frac nd\right)\\
        &= \tilde O\left(\frac{n}{d^{1.5}}\right) .
    \end{align*}
    Similarly, we have $\norm{\frac{\mX\mX^\top}{d} \cdot \diag(\veps)} \le \tilde O\left(\frac{n}{d^{1.5}}\right)$.
    
    Next, for $\left( \frac{\mX\mX^\top}{d} \odot \frac{\mX\mX^\top}{d} \right)_\offdiag$, we can use the 4th moment method in~\citet{el2010spectrum} to show that it is close to its mean. Specifically, the mean at each entry is $\E\left[ \left(\frac{\vx_i^\top\vx_j}{d}\right)^2 \right] = \frac{\Tr[\mSigma^2]}{d^2}$ ($i\not=j$), and the moment calculation in~\citet{el2010spectrum} shows the following bound on the error matrix $\mF = \left( \frac{\mX\mX^\top}{d} \odot \frac{\mX\mX^\top}{d} - \frac{\Tr[\mSigma^2]}{d^2}\vone\vone^\top \right)_\offdiag$
    :
    \begin{align*}
        \E\left[ \norm{\mF}^4 \right] \le 
        \E\left[\Tr[\mF^4]\right] \le \tilde{O}\left( \frac{n^4}{d^6} + \frac{n^3}{d^4} \right) \le \tilde O\left(\frac{n^4}{d^5}\right),
    \end{align*}
    where we have used $n\gtrsim d$.
    Therefore by Markov inequality we know that with high probability, $\norm{\mF} \le \tilde{O}\left( \frac{n}{d^{1.25}} \right)$.
    
    For the final term $\mE$ in~\eqref{eqn:first-layer-ntk-offdiag-approx}, we have
    \begin{align*}
        \norm{\mE} \le \norm{\mE}_F \le \sqrt{n^2 \cdot \tilde{O}\left(\frac{1}{d^3}\right)} = \tilde{O}\left(\frac{n}{d^{1.5}}\right) .
    \end{align*}
    
    Put together, we can obtain the following bound regarding $(\mTheta^*_1)_\offdiag$:
    \begin{equation} \label{eqn:first-layer-ntk-offdiag-approx-2}
    \begin{aligned}
        &\norm{\left(\mTheta_1^* - \zeta^2 \frac{\mX\mX^\top}{d} - c_3\frac{\Tr[\mSigma^2]}{d^2}\vone\vone^\top \right)_\offdiag }\\
        \le\,& c_1\cdot\tilde{O}\left( \frac{n}{d^{1.5}} \right) + c_2\cdot\tilde{O}\left( \frac{n}{d^{1.5}} \right) + c_3\cdot\tilde{O}\left( \frac{n}{d^{1.25}} \right) + \tilde{O}\left(\frac{n}{d^{1.5}}\right)\\
        =\,& \tilde{O}\left(\frac{n}{d^{1.25}}\right).
    \end{aligned}
    \end{equation}
    Here we have used Lemma~\ref{lem:diag-and-offdiag-spectral-norm} to bound the spectral norm of the off-diagonal part of a matrix by the spectral norm of the matrix itself.
    Notice $c_3\frac{\Tr[\mSigma^2]}{d} = \left( \E[g\phi'(g)] \right)^2\cdot\frac{\Tr[\mSigma^2]}{d} = \nu^2$ (c.f.~\eqref{eqn:first-layer-linear-model}).
    Hence~\eqref{eqn:first-layer-ntk-offdiag-approx-2} becomes
    \begin{equation}\label{eqn:first-layer-ntk-offdiag-approx-final}
        \norm{\left( \mTheta_1^* - \mThetalinone \right)_\offdiag} = \tilde{O}\left(\frac{n}{d^{1.25}}\right).
    \end{equation}
    
    For the diagonal entries of $\mTheta_1^*$, we have $\index{\mTheta_1^*}{i, i} = \frac{\norm{\vx_i}^2}{d} \cdot \Phi\left( \frac{\norm{\vx_i}^2}{d}, \frac{\norm{\vx_i}^2}{d}, \frac{\norm{\vx_i}^2}{d} \right)$.
    We denote $\bar{\Phi}(a):= \Phi(a, a, a)$ ($a\ge0$). When $\phi$ is a smooth activation as in Assumption~\ref{asmp:activation}, we know that $\bar{\Phi}$ has bounded derivative, and thus we get
    \begin{equation} \label{eqn:first-layer-ntk-diag-approx}
        \index{\mTheta_1^*}{i, i} = \frac{\norm{\vx_i}^2}{d} \cdot \bar{\Phi}\left(\frac{\norm{\vx_i}^2}{d} \right)
        = \left( 1\pm \tilde{O}\left(\frac{1}{\sqrt d}\right)\right) \cdot \left( \bar{\Phi}(1) \pm \tilde{O}\left(\frac{1}{\sqrt d}\right) \right)
        = \bar{\Phi}(1) \pm \error .
    \end{equation}
    When $\phi$ is a piece-wise linear activation as in Assumption~\ref{asmp:activation}, $\bar{\Phi}(a)$ is a constant, so we have $\bar{\Phi}\left(\frac{\norm{\vx_i}^2}{d} \right) = \bar{\Phi}(1)$. Therefore~\eqref{eqn:first-layer-ntk-diag-approx} also holds.
    Notice that $\bar{\Phi}(1) = \E[(\phi'(g))^2]=:\gamma$.
    It follows from~\eqref{eqn:first-layer-ntk-diag-approx} that
    \begin{align*}
        \norm{ (\mTheta_1^*)_\diag - \gamma  \mI } = \error. 
    \end{align*}
    Also note that
    \begin{align*}
        \norm{\mThetalinone_\diag - \zeta^2\mI} = \error .
    \end{align*}
    Therefore we obtain
    \begin{equation} \label{eqn:first-layer-ntk-diag-approx-final}
        \norm{\left( \mTheta_1^* - \mThetalinone \right)_\diag - (\gamma-\zeta^2)\mI} = \error .
    \end{equation}
    
    Combining the off-diagonal and diagonal approximations~\eqref{eqn:first-layer-ntk-offdiag-approx-final} and~\eqref{eqn:first-layer-ntk-diag-approx-final}, we obtain
    \begin{equation*}
        \norm{\mTheta_1^* - \mThetalinone - (\gamma-\zeta^2)\mI} = \tilde{O}\left(\frac{n}{d^{1.25}}\right) .
    \end{equation*}
    
    Finally, when $n\gtrsim d^{1+\alpha}$ ($0<\alpha<\frac14$), we have $\norm{\mI} = 1 \lesssim \frac{n}{d^{1+\alpha}}$.
    Hence we can discard the identity component above and get
    \begin{equation*}
        \norm{\mTheta_1^* - \mThetalinone} = {O}\left(\frac{n}{d^{1+\alpha}}\right) .
    \end{equation*}
    This completes the proof.
\end{proof}

Combining Propositions~\ref{prop:first-layer-ntk-concentration} and~\ref{prop:first-layer-ana-ntk-approx} directly gives Proposition~\ref{prop:first-layer-ntk-init-approx-copy}.

\subsubsection{Agreement on Training Data} \label{app:proof-first-layer-training}

Now we prove the first part of Theorem~\ref{thm:first-layer-main}, i.e.,  \eqref{eqn:first-layer-training-guarantee}, which says that the neural network $f_t^1$ and the linear model $\flinone_t$ are close on the training data.
We will use Theorem~\ref{thm:general-closeness}, and the most important step is to verify Assumption~\ref{asmp:general-jacobian-closeness}.
To this end we prove the following Jacobian perturbation lemma.


\begin{lem}[Jacobian perturbation for the first layer] \label{lem:jacobian-perturb-first-layer}
    If $\phi$ is a smooth activation as in Assumption~\ref{asmp:activation}, then with high probability over the training data $\mX$, we have
    \begin{equation} \label{eqn:jacobian-perturb-smooth}
        \norm{\mJ_1(\mW,\vv) - \mJ_1(\tmW,\vv)} \lesssim \sqrt{\frac{n}{md}} \norm{\mW-\tmW}_F, \qquad \forall \mW, \tmW \in \R^{m\times d} .
    \end{equation}
    If $\phi$ is a piece-wise linear activation as in Assumption~\ref{asmp:activation}, then with high probability over the random initialization $\mW(0)$ and the training data $\mX$, we have 
    \begin{align} \label{eqn:jacobian-perturb-relu}
        \norm{\mJ_1(\mW,\vv) - \mJ_1(\mW(0),\vv)} \lesssim \sqrt{\frac{n}{d}} \left( \frac{\norm{\mW - \mW(0)}^{1/3}}{m^{1/6}} + \left(\frac{\log n}{m}\right)^{1/4} \right), \qquad \forall \mW\in\R^{m\times d}.
    \end{align}
\end{lem}
\begin{proof}
    Throughout the proof we will be conditioned on $\mX$ and on the high-probability events in Claim~\ref{claim:data-concentration}.
    
    By the definition of $\mJ_1(\mW,\vv)$ in~\eqref{eqn:first-layer-jacobian}, we have
    \begin{equation} \label{eqn:jacobian-diff}
    \begin{aligned}
        &(\mJ_1(\mW,\vv) - \mJ_1(\tmW,\vv))(\mJ_1(\mW,\vv) - \mJ_1(\tmW,\vv))^\top \\
        =\,& \frac{1}{md} \left(\phi'\left(\mX\mW^\top/\sqrt{d}\right)-\phi'\left(\mX\tmW^\top/\sqrt{d}\right)\right) \left(\phi'\left(\mX\mW^\top/\sqrt{d}\right)-\phi'\left(\mX\tmW^\top/\sqrt{d}\right)\right)^\top \odot (\mX\mX^\top) .
    \end{aligned}
    \end{equation}
    Then if $\phi$ is a smooth activation, we have with high probability,
    \begin{align*}
        &\norm{\mJ_1(\mW,\vv) - \mJ_1(\tmW,\vv)}^2 \\
        \le\,& \frac{1}{md} \norm{\phi'\left(\mX\mW^\top/\sqrt{d}\right)-\phi'\left(\mX\tmW^\top/\sqrt{d}\right)}^2 \cdot \max_{i\in[n]} \norm{\vx_i}^2 \tag{\eqref{eqn:jacobian-diff} and Lemma~\ref{lem:hadamard-product-bound}} \\
        \lesssim \,& \frac{1}{md} \norm{\phi'\left(\mX\mW^\top/\sqrt{d}\right)-\phi'\left(\mX\tmW^\top/\sqrt{d}\right)}_F^2 \cdot d \tag{Claim~\ref{claim:data-concentration}} \\
        \lesssim \,& \frac{1}{md} \norm{\mX\mW^\top/\sqrt{d}-\mX\tmW^\top/\sqrt{d}}_F^2 \cdot d \tag{$\phi''$ is bounded} \\
        = \,& \frac{1}{md} \norm{\mX(\mW-\tmW)^\top}_F^2  \\
        \le \,& \frac{1}{md} \norm{\mX}^2\norm{\mW-\tmW}_F^2  \\
        \lesssim\,& \frac{n}{md} \norm{\mW-\tmW}_F^2 \tag{Claim~\ref{claim:data-concentration}}.
    \end{align*}
    This proves~\eqref{eqn:jacobian-perturb-smooth}.
    
    Next we consider the case where $\phi$ is a piece-wise linear activation.
    From~\eqref{eqn:jacobian-diff} and Lemma~\ref{lem:hadamard-product-bound} we have
    \begin{equation} \label{eqn:jacobian-perturb-relu-step-1}
    \begin{aligned}
        \norm{\mJ_1(\mW,\vv) - \mJ_1(\mW(0),\vv)}^2
        &\le \frac{1}{md} \norm{\mX\mX^\top} \cdot \max_{i\in[n]} \norm{\phi'(\mW\vx_i/\sqrt{d}) - \phi'(\mW(0)\vx_i/\sqrt{d})}^2\\
        &\lesssim \frac{n}{md}  \cdot \max_{i\in[n]} \norm{\phi'(\mW\vx_i/\sqrt{d}) - \phi'(\mW(0)\vx_i/\sqrt{d})}^2 .
    \end{aligned}
    \end{equation}
    For each $i\in[n]$, let
    \begin{align*}
        M_i = \{ r\in[m]: \sign(\vw_r^\top\vx_i) \not= \sign(\vw_r(0)^\top\vx_i)  \}
    \end{align*}
    Since $\phi'$ is a step function that only depends on the sign of the input, we have
    \begin{align} \label{eqn:jacobian-perturb-rely-step-2}
        \norm{\phi'(\mW\vx_i/\sqrt{d}) - \phi'(\mW(0)\vx_i/\sqrt{d})}^2
        \lesssim \left| M_i \right|, \qquad \forall i\in[n].
    \end{align}
    Therefore we need to bound $|M_i|$, i.e. how many coordinates in $\mW\vx_i$ and $\mW(0)\vx_i$ differ in sign for each $i\in[n]$.
    
    Let $\lambda>0$ be a parameter whose value will be determined later.
    For each $i\in[n]$, define $$N_i := \{ r\in[m]: |\vw_r(0)^\top \vx_i| \le \lambda\norm{\vx_i} \}.$$
    We have
    \begin{align*}
        |N_i| = \sum_{r=1}^m \ind{|\vw_r(0)^\top \vx_i| \le \lambda\norm{\vx_i}}
        = 2 \sum_{r=1}^{m/2} \ind{|\vw_r(0)^\top \vx_i| \le \lambda\norm{\vx_i}},
    \end{align*}
    where the second equality is due to the symmetric initialization~\eqref{eqn:two-layer-symm-init}. Since $\frac{\vw_r(0)^\top \vx_i}{\norm{\vx_i}} \sim \N(0, 1)$, we have $\E\left[\ind{|\vw_r(0)^\top \vx_i| \le \lambda\norm{\vx_i}}\right] = \Pr[|g|\le\lambda] \le \frac{2\lambda}{\sqrt{2\pi}}$. 
    Also note that $\vw_1(0), \ldots, \vw_{m/2}(0)$ are independent. Then by Hoeffding's inequality we know that with probability at least $1-\delta$,
    \begin{align*} 
        |N_i| \le \sqrt{\frac{2}{\pi}} \lambda m  + O\left( \sqrt{m\log\frac1\delta} \right) .
    \end{align*}
    Taking a union bound over all $i\in[n]$, we know that with high probability, 
    \begin{align} \label{eqn:neuron-switch-bound-1}
        |N_i| \lesssim  \lambda m  +  \sqrt{m\log n}, \qquad \forall i\in[n].
    \end{align}
    
    By definition, if $r\in M_i$ but $r\notin N_i$, we must have $\left| \vw_r^\top\vx_i - \vw_r(0)^\top\vx_i \right| \ge \left| \vw_r(0)^\top\vx_i \right| > \lambda\norm{\vx_i}$.
    This leads to
    \begin{align*}
        \norm{(\mW-\mW(0))\vx_i}^2
        &= \sum_{r=1}^m \left|(\vw_r-\vw_r(0))^\top\vx_i\right|^2
        \ge \sum_{r\in M_i\setminus N_i} \abs{(\vw_r-\vw_r(0))^\top\vx_i}^2 \\
        &\ge \sum_{r\in M_i\setminus N_i} \lambda^2\norm{\vx_i}^2
        \gtrsim \sum_{r\in M_i\setminus N_i} \lambda^2d = \lambda^2d\abs{M_i\setminus N_i}
    \end{align*}
    Thus we have
    \begin{align} \label{eqn:neuron-switch-bound-2}
        \abs{M_i\setminus N_i} \lesssim \frac{\norm{(\mW-\mW(0))\vx_i}^2}{\lambda^2d}
        \le \frac{\norm{\mW-\mW(0)}^2\norm{\vx_i}^2}{\lambda^2d}
        \lesssim \frac{\norm{\mW-\mW(0)}^2}{\lambda^2} , \qquad \forall i\in[n].
    \end{align}
    Combining~\eqref{eqn:neuron-switch-bound-1} and~\eqref{eqn:neuron-switch-bound-2} we obtain
    \begin{align*}
        |M_i| \lesssim \lambda m  +  \sqrt{m\log n} + \frac{\norm{\mW-\mW(0)}^2}{\lambda^2} , \qquad \forall i\in[n] .
    \end{align*}
    Letting $\lambda = \left( \frac{\norm{\mW - \mW(0)}^2}{m} \right)^{1/3}$, we get
    \begin{align}\label{eqn:neuron-switch-bound-final}
        |M_i| \lesssim m^{2/3} \norm{\mW - \mW(0)}^{2/3} + \sqrt{m\log n} , \qquad \forall i\in[n] .
    \end{align}
    Finally, we combine~\eqref{eqn:jacobian-perturb-relu-step-1}, \eqref{eqn:jacobian-perturb-rely-step-2} and~\eqref{eqn:neuron-switch-bound-final} to obtain
    \begin{align*}
        \norm{\mJ_1(\mW,\vv) - \mJ_1(\mW(0),\vv)}^2 
        &\lesssim \frac{n}{md} \left(m^{2/3} \norm{\mW - \mW(0)}^{2/3} + \sqrt{m\log n} \right) \\
        &= \frac{n}{d} \left( \frac{\norm{\mW - \mW(0)}^{2/3}}{m^{1/3}} + \sqrt{\frac{\log n}{m}} \right) .
    \end{align*}
    This proves~\eqref{eqn:jacobian-perturb-relu}.
\end{proof}

The next lemma verifies Assumption~\ref{asmp:general-jacobian-closeness} for the case of training the first layer.
\begin{lem} \label{lem:verify-asmp-first-layer}
    Let $R = \sqrt{d\log d}$.
    With high probability over the random initialization $\mW(0)$ and the training data $\mX$, for all $\mW,\tmW\in\R^{m\times d}$ such that $\norm{\mW-\mW(0)}_F\le R$ and $\norm{\tmW-\mW(0)}_F\le R$, we have
    \[
        \norm{\mJ_1(\mW,\vv)\mJ_1(\tmW,\vv)^\top - \mThetalinone} \lesssim \frac{n}{d^{1+\frac\alpha7}} .
    \]
\end{lem}
\begin{proof}
    This proof is conditioned on all the high-probability events we have shown.
    
    Now consider $\mW,\tmW\in\R^{m\times d}$ such that $\norm{\mW-\mW(0)}_F\le R$ and $\norm{\tmW-\mW(0)}_F\le R$.
If $\phi$ is a smooth activation, from Lemma~\ref{lem:jacobian-perturb-first-layer} we have
    \begin{align*}
        \norm{\mJ_1(\mW,\vv) - \mJ_1(\mW(0),\vv)}
        \lesssim \sqrt{\frac{n}{md}} \norm{\mW-\mW(0)}_F
        \le \sqrt{\frac{n}{md}} \cdot \sqrt{d\log d}
        \lesssim \sqrt{\frac{n\log d}{d^{1+\alpha}}}
        \ll \sqrt{\frac{n}{d^{1+\frac{\alpha}{2}}}},
    \end{align*}
    where we have used $m\gtrsim d^{1+\alpha}$.
    If $\phi$ is a piece-wise linear activation, from Lemma~\ref{lem:jacobian-perturb-first-layer} we have
    \begin{align*}
        \norm{\mJ_1(\mW,\vv) - \mJ_1(\mW(0),\vv)} &\lesssim \sqrt{\frac{n}{d}} \left( \frac{\norm{\mW - \mW(0)}^{1/3}}{m^{1/6}} + \left(\frac{\log n}{m}\right)^{1/4} \right) \\
        &\le \sqrt{\frac{n}{d}} \left( \frac{(d\log d)^{1/6}}{m^{1/6}} + \left(\frac{\log n}{m}\right)^{1/4} \right) \\
        &\lesssim \sqrt{\frac{n}{d}} \cdot \frac{(d\log d)^{1/6}}{d^{1/6+\alpha/6}} \\
        &\ll \frac{\sqrt{n}}{d^{\frac12 + \frac{\alpha}{7}}} .
    \end{align*}
    Hence we always have $\norm{\mJ_1(\mW,\vv) - \mJ_1(\mW(0),\vv)} \le \frac{\sqrt{n}}{d^{\frac12 + \frac{\alpha}{7}}}$.
    Similarly, we have $\norm{\mJ_1(\tmW,\vv) - \mJ_1(\mW(0),\vv)} \le \frac{\sqrt{n}}{d^{\frac12 + \frac{\alpha}{7}}}$.
    
    Note that from Proposition~\ref{prop:first-layer-ntk-init-approx-copy} and Claim~\ref{claim:data-concentration} we know
    \begin{align*}
        \norm{ \mJ_1(\mW(0),\vv)\mJ_1(\mW(0),\vv)^\top }
        \lesssim \norm{\mThetalinone} + \frac{n}{d^{1+\alpha}}
        \lesssim \frac{n}{d} + \frac{n}{d^{1+\alpha}}
        \lesssim \frac{n}{d},
    \end{align*}
    which implies $\norm{\mJ_1(\mW(0),\vv)} \lesssim \sqrt{\frac{n}{d}}$.
    It follows that $\norm{\mJ_1(\mW,\vv)} \lesssim \sqrt{\frac{n}{d}} + \frac{\sqrt{n}}{d^{\frac12 + \frac{\alpha}{7}}} \lesssim \sqrt{\frac{n}{d}}$ and $\norm{\mJ_1(\tmW,\vv)} \lesssim \sqrt{\frac{n}{d}}$.
    Then we have
    \begin{align*}
        &\norm{ \mJ_1(\mW,\vv)\mJ_1(\tmW,\vv)^\top - \mJ_1(\mW(0),\vv)\mJ_1(\mW(0),\vv)^\top  } \\
        \le\,& \norm{\mJ_1(\mW,\vv)}\cdot\norm{\mJ_1(\tmW,\vv)- \mJ_1(\mW(0),\vv)} + \norm{\mJ_1(\mW(0),\vv)}\cdot\norm{\mJ_1(\mW,\vv)-\mJ_1(\mW(0),\vv)}\\
        \lesssim\,& \sqrt{\frac nd}\cdot\frac{\sqrt{n}}{d^{\frac12 + \frac{\alpha}{7}}} + \sqrt{\frac nd}\cdot\frac{\sqrt{n}}{d^{\frac12 + \frac{\alpha}{7}}}\\
        \lesssim\,& \frac{n}{d^{1+\frac{\alpha}{7}}} .
    \end{align*}
    Combining the above inequality with Proposition~\ref{prop:first-layer-ntk-init-approx-copy}, we obtain
    \begin{align*}
        \norm{ \mJ_1(\mW,\vv)\mJ_1(\tmW,\vv)^\top - \mThetalinone}
        \lesssim \frac{n}{d^{1+\frac{\alpha}{7}}} + \frac{n}{d^{1+\alpha}} 
        \lesssim \frac{n}{d^{1+\frac{\alpha}{7}}} ,
    \end{align*}
    completing the proof.
\end{proof}

Finally, we can instantiate Theorem~\ref{thm:general-closeness} to conclude the proof of~\eqref{eqn:first-layer-training-guarantee}:
\begin{prop} \label{prop:first-layer-training-guarantee}
There exists a universal constant $c>0$ such that
    with high probability, for all $0\le t \le T = c\cdot\frac{d\log d}{\eta_1}$ simultaneously, we have:
    \begin{itemize}
        \item $\frac{1}{n}\sum_{i=1}^n (f_t^1(\vx_i)-\flinone_t(\vx_i))^2\le d^{-\frac\alpha4}$;
        \item $\norm{\mW(t)-\mW(0)}_F\le\sqrt{d\log d}$, $\norm{\vbeta(t)}\le\sqrt{d\log d}$.
    \end{itemize}
\end{prop}
\begin{proof}
    Let $R=\sqrt{d\log d}$ and $\epsilon = C\frac{n}{d^{1+\frac\alpha7}}$ for a sufficiently large universal constant $C>0$. From Lemma~\ref{lem:verify-asmp-first-layer} we know that Assumption~\ref{asmp:general-jacobian-closeness} is satisfied with parameters $\epsilon$ and $R$. (Note that $\epsilon \ll \frac{n}{d} \lesssim \norm{\mThetalinone}$.)
    Also we have $R^2\epsilon\ll n$, and $\eta_1 \ll d \lesssim\frac{n}{\norm{\mThetalinone}}$.
    Therefore, we can apply Theorem~\ref{thm:general-closeness} and obtain for all $0\le t\le T$:
    \begin{align*}
        \sqrt{\sum_{i=1}^n (f_t^1(\vx_i)-\flinone_t(\vx_i))^2} \lesssim \frac{\eta_1 t \epsilon}{\sqrt{n}} \lesssim  \frac{d\log d\cdot \frac{n}{d^{1+\frac\alpha7}}}{\sqrt n} = \frac{\sqrt{n}\log d}{d^{\frac{\alpha}{7}}}
        \ll \frac{\sqrt n}{d^{\frac{\alpha}{8}}},
    \end{align*}
    which implies
    \begin{align*}
        \frac{1}{n}\sum_{i=1}^n (f_t^1(\vx_i)-\flinone_t(\vx_i))^2 \le d^{-\frac{\alpha}{4}}.
    \end{align*}
    
    Furthermore, Theorem~\ref{thm:general-closeness} also tells us $\norm{\mW(t)-\mW(0)}_F\le\sqrt{d\log d}$ and $\norm{\vbeta(t)}\le\sqrt{d\log d}$.
\end{proof}

\subsubsection{Agreement on Distribution} \label{app:proof-first-layer-distribution}

Now we prove the second part of Theorem~\ref{thm:first-layer-main}, \eqref{eqn:first-layer-test-guarantee}, which guarantees the agreement between $f_t^1$ and $\flinone_t$ on the entire distribution $\D$.
As usual, we will be conditioned on all the high-probability events unless otherwise noted.

Given the initialization $(\mW(0),\vv)$ (recall that $\vv=\vv(0)$ is always fixed), we define an auxiliary model $\fauxone(\vx;\mW)$ which is the first-order Taylor approximation of the neural network $f(\vx; \mW,\vv)$ around $\mW(0)$:
\begin{align*}
    \fauxone(\vx;\mW) :=\,& f(\vx;\mW(0),\vv) + 
    \langle \mW-\mW(0), \nabla_{\mW}f(\vx;\mW(0),\vv) \rangle\\ 
    =\,& \langle \mW-\mW(0), \nabla_{\mW}f(\vx;\mW(0),\vv) \rangle \\
    =\,& \langle \vect{\mW-\mW(0)}, \vrho_1(\vx) \rangle,
\end{align*}
where $\vrho_1(\vx) := \nabla_{\mW}f(\vx;\mW(0),\vv)$.
Above we have used $f(\vx;\mW(0),\vv)=0$ according to the symmetric initialization~\eqref{eqn:two-layer-symm-init}.
We also denote $\fauxone_t(\vx):=\fauxone(\vx;\mW(t))$ for all $t$.

For all models, we write their predictions on all training datapoints concisely as $f^1_t(\mX),\flinone_t(\mX),\fauxone_t(\mX) \in \R^n$.
From Proposition~\ref{prop:first-layer-training-guarantee} we know that $f_t^1$ and $\flinone_t$ make similar predictions on $\mX$ (for all $t\le T$ simultaneously):
\begin{align} \label{eqn:first-layer-f-flin-close}
    \norm{f^1_t(\mX) - \flinone_t(\mX)} \le \frac{\sqrt{n}}{d^{\frac{\alpha}{8}}} .
\end{align}
We can also related the predictions of $f_t^1$ and $\fauxone_t$ by the fundamental theorem for line integrals:
\begin{equation} \label{eqn:first-layer-linearization}
\begin{aligned}
    f^1_t(\mX)
    &= f^1_t(\mX) - f^1_0(\mX)
    = \mJ_1(\mW(0)\to\mW(t),\vv)\cdot\vect{\mW(t)-\mW(0)},\\
    \fauxone_t(\mX)
    &= \fauxone_t(\mX) - \fauxone_0(\mX) = \mJ_1(\mW(0),\vv)\cdot\vect{\mW(t)-\mW(0)} ,
\end{aligned}
\end{equation}
where $\mJ_1(\mW(0)\to\mW(t),\vv):= \int_0^1 \mJ_1(\mW(0)+x(\mW(t)-\mW(0)),\vv)dx$.
Since $\norm{\mW(t)-\mW(0)}_F \le\sqrt{d\log d}$ according to Proposition~\ref{prop:first-layer-training-guarantee}, we can use Lemma~\ref{lem:jacobian-perturb-first-layer} in the same way as in the proof of Lemma~\ref{lem:verify-asmp-first-layer} and obtain
\[
\norm{\mJ_1(\mW(0)\to\mW(t),\vv) - \mJ_1(\mW(0),\vv)} \le \frac{\sqrt{n}}{d^{\frac12+\frac\alpha7}} .
\]
Then it follows from~\eqref{eqn:first-layer-linearization} that
\begin{equation} \label{eqn:first-layer-f-faux-close}
\begin{aligned}
    \norm{f^1_t(\mX) - \fauxone_t(\mX)}
    &= \norm{\left( \mJ_1(\mW(0)\to\mW(t),\vv) - \mJ_1(\mW(0),\vv) \right) \cdot \vect{\mW(t)-\mW(0)}}\\
    &\le \frac{\sqrt{n}}{d^{\frac12+\frac\alpha7}} \cdot \sqrt{d\log d}\\
    &\le \frac{\sqrt{n}}{d^{\frac{\alpha}{8}}} .
\end{aligned}
\end{equation}
Combining~\eqref{eqn:first-layer-f-flin-close} and~\eqref{eqn:first-layer-f-faux-close} we know 
\begin{align*}
    \norm{\fauxone_t(\mX) - \flinone_t(\mX)} \lesssim \frac{\sqrt n}{d^{\frac{\alpha}{8}}} .
\end{align*}
This implies
\begin{align*}
    \frac{1}{n} \sum_{i=1}^n \min\left\{\left(\fauxone_t(\vx_i) - \flinone_t(\vx_i) \right)^2, 1\right\} \le
    \frac{1}{n} \sum_{i=1}^n \left(\fauxone_t(\vx_i) - \flinone_t(\vx_i) \right)^2 \lesssim d^{-\frac{\alpha}{4}} .
\end{align*}

Next we will translate these guarantees on the training data to the distribution $\D$ using Rademacher complexity.
Note that the model $\fauxone_t(\vx) - \flinone_t(\vx)$ is by definition linear in the feature $\begin{bmatrix} \vrho_1(\vx)\\ \vpsi_1(\vx) \end{bmatrix}$, and it belongs to the following function class (for all $t\le T$):
\begin{align*}
    \gF := \left\{ \vx \mapsto \va^\top \begin{bmatrix} \vrho_1(\vx)\\ \vpsi_1(\vx)  \end{bmatrix}: \norm{\va}\le 2\sqrt{d\log d} \right\}.
\end{align*}
This is because we have $\norm{\vect{\mW(t)-\mW(0)}}\le\sqrt{d\log d}$ and $\norm{\vbeta(t)}\le\sqrt{d\log d}$ for all $t\le T$.
Using the well-known bound on the empirical Rademacher complexity of a linear function class with bounded $\ell_2$ norm (see e.g. \citet{bartlett2002rademacher}), we can bound the empirical Rademacher complexity of the function class $\gF$:
\begin{equation} \label{eqn:rademacher-first-layer}
\begin{aligned}
    \hat{\gR}_{\mX}(\gF) :=\,& \frac{1}{n} \E_{\varepsilon_1,\ldots,\varepsilon_n\simiid\unif(\{\pm1\})}\left[ \sup_{h\in\gF} \sum_{i=1}^n \varepsilon_ih(\vx_i)  \right]\\
    \lesssim\,& \frac{\sqrt{d\log d}}{n} \sqrt{\sum_{i=1}^n \left( \norm{\vrho_1(\vx_i)}^2 + \norm{\vpsi_1(\vx_i)}^2 \right)}\\
    =\,& \frac{\sqrt{d\log d}}{n} \sqrt{\Tr[\mTheta_1(\mW(0),\vv)] + \Tr[\mThetalinone]}.
\end{aligned}
\end{equation}
Since $\phi'$ is bounded and $\frac{\norm{\vx_i}^2}{d}=O(1)$ ($\forall i\in[n]$), we can bound
\begin{align*}
    \Tr[\mTheta_1(\mW(0),\vv)] = \sum_{i=1}^n \frac1m \sum_{r=1}^m \phi'\left(\vw_r(0)^\top\vx_i/\sqrt{d}\right)^2 \cdot \frac{\norm{\vx_i}^2}{d}
    \lesssim n,
\end{align*}
and
\begin{align*}
    \Tr[\mThetalinone] = \sum_{i=1}^n \left( \zeta^2\frac{\norm{\vx_i}^2}{d} + \frac{\nu^2}{d} \right) \lesssim n.
\end{align*}
Therefore we have
\begin{align*}
    \hat{\gR}_{\mX}(\gF) \lesssim \frac{\sqrt{d\log d}}{n}\sqrt{n} = \sqrt{\frac{d\log d}{n}} .
\end{align*}
Now using the standard generalization bound via Rademacher complexity (see e.g.~\citet{mohri2012foundations}), and noticing that the function $z\mapsto\min\{z^2,1\}$ is $2$-Lipschitz and bounded in $[0,1]$, we have with high probability, for all $t\le T$ simultaneously,
\begin{align} 
    &\E_{\vx\sim\D}\left[ \min\left\{\left(\fauxone_t(\vx) - \flinone_t(\vx) \right)^2, 1\right\} \right]\notag\\
    \le\,& \frac{1}{n} \sum_{i=1}^n \min\left\{\left(\fauxone_t(\vx_i) - \flinone_t(\vx_i) \right)^2, 1\right\} + O\left( \sqrt{\frac{d\log d}{n}} \right) + O\left( \frac{1}{\sqrt{n}} \right)\notag\\
    \lesssim\,& d^{-\frac{\alpha}{4}} + \sqrt{\frac{d\log d}{d^{1+\alpha}}} \tag{$n\gtrsim d^{1+\alpha}$}\\
    \lesssim\,& d^{-\frac{\alpha}{4}} . \label{eqn:first-layer-faux-flin-close-population}
\end{align}
Therefore we have shown that $\fauxone_t$ and $\flinone_t$ are close on the distribution $\D$ for all $t\le T$. To complete the proof, we need to show that $f^1_t$ and $\fauxone_t$ are close on $\D$.
For this, we take an imaginary set of test datapoints $\tvx_1,\ldots,\tvx_n \simiid \D$, which are independent of the training samples. Let $\tmX \in \R^{n\times d}$ be the corresponding test data matrix.
Since the test data are from the same distribution $\D$, the concentration properties in Claim~\ref{claim:data-concentration} still hold, and the Jacobian perturbation bounds in Lemma~\ref{lem:jacobian-perturb-first-layer} hold as well. Hence we can apply the exact same arguments in~\eqref{eqn:first-layer-f-faux-close} and obtain with high probability for all $t\le T$,
\begin{align*}
    \norm{f^1_t(\tmX) - \fauxone_t(\tmX)} \le \frac{\sqrt{n}}{d^{\frac{\alpha}{8}}} ,
\end{align*}
which implies
\begin{align*}
    \frac{1}{n} \sum_{i=1}^n \min \left\{ (f^1_t(\tvx_i) - \fauxone_t(\tvx_i))^2, 1 \right\} \le d^{-\frac{\alpha}{4}} .
\end{align*}
Now notice that $f^1_t$ and $\fauxone_t$ are independent of $\tmX$. Thus, by Hoeffding inequality, for each $t$, with probability at least $1-\delta$ we have
\begin{align*}
    &\E_{\vx\sim\D} \left[ \min \left\{ (f^1_t(\vx) - \fauxone_t(\vx))^2, 1 \right\} \right] \\
    \le\,& \frac{1}{n} \sum_{i=1}^n \min \left\{ (f^1_t(\tvx_i) - \fauxone_t(\tvx_i))^2, 1 \right\} + O\left( \sqrt{\frac{\log\frac{1}{\delta}}{n}} \right) \\
    \lesssim\,& d^{-\frac{\alpha}{4}} + \sqrt{\frac{\log\frac{1}{\delta}}{n}}.
\end{align*}
Then letting $\delta = \frac{1}{100T}$ and taking a union
bound over $t\le T$, we obtain that with high probability, for all $t\le T$ simultaneously,
\begin{equation} \label{eqn:first-layer-f-faux-close-population}
    \E_{\vx\sim\D} \left[ \min \left\{ (f^1_t(\vx) - \fauxone_t(\vx))^2, 1 \right\} \right] \lesssim d^{-\frac{\alpha}{4}} + \sqrt{\frac{\log T}{n}}.
\end{equation}
Therefore we have proved that $f^1_t$ and $\fauxone_t$ are close on $\D$. Finally, combining~\eqref{eqn:first-layer-faux-flin-close-population} and~\eqref{eqn:first-layer-f-faux-close-population}, we know that with high probability, for all $t\le T$,
\begin{equation*} 
    \E_{\vx\sim\D} \left[ \min \left\{ (f^1_t(\vx) - \flinone_t(\vx))^2, 1 \right\} \right] \lesssim d^{-\frac{\alpha}{4}} + \sqrt{\frac{\log T}{n}}.
\end{equation*}
Here we have used $\min\{(a+b)^2,1\} \le 2(\min\{a^2,1\}+\min\{b^2,1\})$ ($\forall a,b\in\R$).
Therefore we have finished the proof of~\eqref{eqn:first-layer-test-guarantee}. The proof of Theorem~\ref{thm:first-layer-main} is done.

\subsection{Proof of Corollary~\ref{cor:first-layer-well-condioned} (Training the First Layer, Well-Conditioned Data)}
\label{app:proof-first-layer-well-condition}

\begin{proof}[Proof of Corollary~\ref{cor:first-layer-well-condioned}]
    We continue to adopt the notation in \Cref{app:proof-first-layer-distribution} to use $f^1_t(\mX)$, $\flinone_t(\mX)$, etc. to represent the predictions of a model on all $n$ training datapoints.
    Given Theorem~\ref{thm:first-layer-main}, it suffices to prove that $\flinone_T$ and $\flinone_*$ are close in the following sense:
    \begin{align}
        \frac1n \sum_{i=1}^n \left( \flinone_T(\vx_i) - \flinone_*(\vx_i) \right)^2 \lesssim d^{-\Omega(\alpha)}, \label{eqn:first-layer-well-condition-to-prove-train} \\
        \E_{\vx\sim\D}\left[ \min\{(\flinone_T(\vx) - \flinone_*(\vx))^2, 1\} \right] \lesssim d^{-\Omega(\alpha)}. \label{eqn:first-layer-well-condition-to-prove-population}
    \end{align}
    
    According to the linear dynamics~\eqref{eqn:first-layer-linear-gd}, we have the following relation (see Claim~\ref{claim:general-linear-dynamics}):
    \begin{align*}
        \flinone_T(\mX) - \vy &= \left( \mI - \tfrac{1}{n}\eta_1 \mThetalinone \right)^T (-\vy) ,\\
        \flinone_*(\mX) - \vy &= \lim_{t\to\infty}\left( \mI - \tfrac{1}{n}\eta_1 \mThetalinone \right)^t (-\vy) =: \left( \mI - \tfrac{1}{n}\eta_1 \mThetalinone \right)^\infty (-\vy) .
    \end{align*}
    From the well-conditioned data assumption, it is easy to see that $\mThetalinone$'s non-zero eigenvalues are all $\Omega(\tfrac nd)$ with high probability.
    As a consequence, in all the non-zero eigen-directions of $\mThetalinone$, the corresponding eigenvalues of $\left( \mI - \tfrac{1}{n}\eta_1 \mThetalinone \right)^T$ are at most $\left( 1 - \tfrac{1}{n}\eta_1\cdot \Omega(\tfrac{n}{d}) \right)^T \le \exp\left(-\Omega\left(\tfrac{\eta_1 T}{d}\right)\right) = \exp\left( -\Omega(\log d) \right) = d^{-\Omega(1)}$. This implies
    \begin{align*}
        \norm{\flinone_T(\mX) - \flinone_*(\mX)}
        \le \norm{\left( \mI - \tfrac{1}{n}\eta_1 \mThetalinone \right)^T - \left( \mI - \tfrac{1}{n}\eta_1 \mThetalinone \right)^\infty} \cdot \norm{\vy} \lesssim d^{-\Omega(1)}\sqrt{n},
    \end{align*}
    which completes the proof of~\eqref{eqn:first-layer-well-condition-to-prove-train}.
    
    To prove~\eqref{eqn:first-layer-well-condition-to-prove-population}, we further apply the standard Rademacher complexity argument (similar to \Cref{app:proof-first-layer-distribution}). For this we just need to bound the $\ell_2$ norm of the parameters, $\norm{\vbeta(T)}$ and $\norm{\vbeta_*}$.
    From Proposition~\ref{prop:first-layer-training-guarantee}, we already have $\norm{\vbeta(T)} \le \sqrt{d\log d}$.
    Regarding $\vbeta_*$, we can directly write down its expression
    \[
        \vbeta_* = (\mPsi_1^\top \mPsi_1)^\dagger \mPsi_1^\top \vy .
    \]
    Here $\mPsi_1$ is the feature matrix defined in~\eqref{eqn:lin-feature-matrices}, and $^\dagger$ stands for the Moore–Penrose pseudo-inverse. Recall that $\mThetalinone = \mPsi_1\mPsi_1^\top$.
    Notice that every non-zero singular value of $(\mPsi_1^\top \mPsi_1)^\dagger \mPsi_1^\top$ is the inverse of a non-zero singular value of $\mPsi_1$, and that every non-zero singular value of $\mPsi_1$ is $\Omega(\sqrt{\tfrac nd})$.
    This implies $\norm{(\mPsi_1^\top \mPsi_1)^\dagger \mPsi_1^\top} \lesssim \sqrt{\tfrac dn}$. Hence we have
    \[
    \norm{\vbeta_*} \lesssim \sqrt{\tfrac dn} \sqrt{n} = \sqrt{d} .
    \]
    Therefore we can apply the standard Rademacher complexity argument and conclude the proof of~\eqref{eqn:first-layer-well-condition-to-prove-population}.
\end{proof}

\subsection{Proof of Theorem~\ref{thm:second-layer-main} (Training the Second Layer)} \label{app:proof-second-layer}

Since the first layer is kept fixed in this case, we let $\mW=\mW(0)$ for notational convenience. Similar to the proof of Theorem~\ref{thm:first-layer-main} in \Cref{app:proof-first-layer}, we still divide the proof into 3 parts: analyzing the NTK at initialization (which is also the NTK throughout training in this case), proving the agreement on training data, and proving the agreement on the distribution.

It is easy to see from the definition of $\mThetalintwo$ in~\eqref{eqn:lin-kernels} and Claim~\ref{claim:data-concentration} that if $\vartheta_0\not=0$, then $\norm{\mThetalintwo}=O(n)$ with high probability, and if $\vartheta_0=0$, then $\norm{\mThetalintwo}={O}(\tfrac{n\log n}{d})$ with high probability. As we will see in the proof, this is why we distinguish these two cases in Theorem~\ref{thm:second-layer-main}.

\subsubsection{The NTK at Initialization}

\begin{prop}\label{prop:second-layer-ntk-init-approx}
    With high probability over the random initialization $\mW$ and the training data $\mX$, we have
    \[
        \norm{ \mTheta_2(\mW) - \mThetalintwo } \lesssim \frac{n}{d^{1+\frac{\alpha}{3}}} .
    \]
\end{prop}

To prove Proposition~\ref{prop:second-layer-ntk-init-approx}, we will prove $\mTheta_2(\mW)$ is close to its expectation $\mTheta^*_2$ (defined in~\eqref{eqn:two-layer-expected-NTKs}), and then prove $\mTheta^*_2$ is close to $\mThetalintwo$. We do these steps in the next two propositions.

\begin{prop} \label{prop:second-layer-ana-ntk-approx}
    With high probability over the training data $\mX$, we have
    \begin{align*}
        \norm{\mTheta_2^* - \mThetalintwo} \lesssim \frac{n}{d^{1+\alpha}}.
    \end{align*}
\end{prop}
\begin{proof}
    We will be conditioned on the high probability events stated in Claim~\ref{claim:data-concentration}.
    
    By the definition of $\mTheta_2^*$, we know
    \begin{align*}
        \index{\mTheta_2^*}{i, j} =  \E_{\vw\sim\N(\vzero,\mI)}\left[  \phi(\vw^\top\vx_i/\sqrt{d})\phi(\vw^\top\vx_j/\sqrt{d})^\top  \right], \qquad i, j\in[n].
    \end{align*}
    We define
    \begin{align*}
        \Gamma(a, b, c) := \E_{(z_1, z_2)\sim\N\left(\vzero, \mLambda\right)} [\phi(z_1)\phi(z_2)], \text{  where } \mLambda = \begin{pmatrix}a^2& c\\c & b^2\end{pmatrix}, \quad a\ge0, b\ge0, |c|\le ab .
    \end{align*}
    Then we can write
    \begin{align*}
        \index{\mTheta_2^*}{i, j} = \Gamma\left( \frac{\norm{\vx_i}}{\sqrt{d}}, \frac{\norm{\vx_j}}{\sqrt{d}}, \frac{\vx_i^\top\vx_j}{d} \right) .
    \end{align*}
    Denote $e_i := \frac{\norm{\vx_i}}{\sqrt{d}}-1$ and $s_{i,j} := \frac{\vx_i^\top\vx_j}{d}$.
    Below we consider the diagonal and off-diagonal entries of $\mTheta_2^*$ separately. 
    
    For $i\not=j$, we do a Taylor expansion of $\Gamma$ around $(1,1,0)$:
    \begin{align*}
        &\index{\mTheta_2^*}{i,j}\\
        =\,& \Gamma(1,1,0) + \nabla\Gamma(1,1,0)^\top \begin{bmatrix}e_i\\e_j\\ s_{i,j} \end{bmatrix} + \frac12 [e_i,e_j,s_{i,j}]\cdot \nabla^2\Gamma(1,1,0) \cdot \begin{bmatrix}e_i\\e_j\\ s_{i,j} \end{bmatrix} + O(\abs{e_i}^3 + \abs{e_j}^3 + \abs{s_{i,j}}^3)\\
        =\,& \vartheta_0^2 + \vartheta_0\vartheta_1(e_i+e_j) + \zeta^2s_{i,j} + \vartheta_0\vartheta_2(e_i^2+e_j^2) + \vartheta_1^2e_ie_j + \frac12 \vartheta_1^2 s_{i,j}^2 + \gamma s_{i,j}(e_i+e_j) \pm \tilde{O}\left(\frac{1}{d^{3/2}}\right)\\
        =\,& (\vartheta_0 + \vartheta_1e_i + \vartheta_2e_i^2)(\vartheta_0 + \vartheta_1e_j + \vartheta_2e_j^2) - \vartheta_1\vartheta_2(e_ie_j^2+e_i^2e_j) - \vartheta_2^2e_i^2e_j^2 + \zeta^2s_{i,j} + \frac12\vartheta_1^2s_{i,j}^2 \\ & +\gamma s_{i,j}(e_i+e_j) \pm \tilde{O}\left(\frac{1}{d^{3/2}}\right)\\
        =\,& \index{\vq}{i}\index{\vq}{j} \pm \tilde{O}\left(\frac{1}{d^{3/2}}\right) \pm \tilde{O}\left(\frac{1}{d^{2}}\right) + \zeta^2s_{i,j} + \frac12\vartheta_1^2s_{i,j}^2  +\gamma s_{i,j}(e_i+e_j) \pm \tilde{O}\left(\frac{1}{d^{3/2}}\right)\\
        =\,& \index{\vq}{i}\index{\vq}{j} + \zeta^2s_{i,j} + \frac12\vartheta_1^2s_{i,j}^2  +\gamma s_{i,j}(e_i+e_j) \pm \tilde{O}\left(\frac{1}{d^{3/2}}\right).
    \end{align*}
    Here $\zeta, \vartheta_0, \vartheta_1, \vartheta_2$ are defined in~\eqref{eqn:second-layer-linear-model}, and $\gamma$ is the $(1,3)$-th entry in the Hessian $\nabla^2\Gamma(1,1,0)$ whose specific value is not important to us. Recall that $\index{\vq}{i} = \vartheta_0 + \vartheta_1e_i + \vartheta_2e_i^2$.
    
    On the other hand, by the definition~\eqref{eqn:lin-kernels} we have
    \begin{align*}
        \index{\mThetalintwo}{i,j} = \zeta^2 s_{i,j} + \frac{\nu^2}{2d} + \index{\vq}{i}\index{\vq}{j} .
    \end{align*}
    It follows that
    \begin{align*}
        \index{\mTheta_2^* - \mThetalintwo}{i,j}
        &= \frac12\vartheta_1^2s_{i,j}^2 - \frac{\nu^2}{2d} +\gamma s_{i,j}(e_i+e_j) \pm \tilde{O}\left(\frac{1}{d^{3/2}}\right)\\
        &= \frac12\vartheta_1^2\left(s_{i,j}^2 - \frac{\Tr[\mSigma^2]}{d^2} \right) +\gamma s_{i,j}(e_i+e_j) \pm \tilde{O}\left(\frac{1}{d^{3/2}}\right).
    \end{align*}
    Here we have used the definition of $\nu$ in~\eqref{eqn:first-layer-linear-model}.
    In the proof of Proposition~\ref{prop:first-layer-ana-ntk-approx}, we have proved that all the error terms above contribute to at most $\tilde{O}(\frac{n}{d^{1.25}})$ in spectral norm. Using the analysis there we get
    \begin{align*}
        \norm{(\mTheta_2^* - \mThetalintwo)_\offdiag} = \tilde{O}\left(\frac{n}{d^{1.25}}\right) .
    \end{align*}
    
    Regarding the diagonal entries, it is easy to see that all the diagonal entries in $\mTheta_2^*$ and $\mThetalintwo$ are $O(1)$, which implies
    \begin{align*}
        \norm{(\mTheta_2^* - \mThetalintwo)_\diag} = O(1) .
    \end{align*}
    
    Therefore we have
    \begin{align*}
        \norm{\mTheta_2^* - \mThetalintwo} = \tilde{O}\left(\frac{n}{d^{1.25}}\right) + O(1)
        = O\left( \frac{n}{d^{1+\alpha}} \right),
    \end{align*}
    since $n\gtrsim d^{1+\alpha}$ ($0<\alpha<\frac14$).
\end{proof}

\begin{prop}\label{prop:second-layer-ntk-concentration}
    With high probability over the random initialization $\mW$ and the training data $\mX$, we have
    \[
        \norm{ \mTheta_2(\mW) - \mTheta_2^* } \lesssim \frac{n}{d^{1+\frac{\alpha}{3}}} .
    \]
\end{prop}
\begin{proof}
    For convenience we denote $\mTheta_2 = \mTheta_2(\mW)$ in the proof. We will be conditioned on $\mX$ and on Claim~\ref{claim:data-concentration}, and only consider the randomness in $\mW$.
    From Proposition~\ref{prop:second-layer-ana-ntk-approx} we know that $\norm{\mTheta_2^*}= \begin{cases}\tilde{O}(n/d), \text{ if } \vartheta_0=\E[\phi(g)]=0 \\ O(n),\ \ \ \  \text{ otherwise} \end{cases}$.
    
    Define $\mTheta_2^{(r)} := \phi(\mX\vw_r/\sqrt{d})\phi(\mX\vw_r/\sqrt{d})^\top$ for each $r\in[m]$.
    We have $\mTheta_2 = \frac1m \sum_{r=1}^m \mTheta_2^{(r)}$.
    According to the initialization scheme~\eqref{eqn:two-layer-symm-init}, we know that $\mTheta_2^{(1)}, \mTheta_2^{(2)}, \ldots, \mTheta_2^{(m/2)}$ are independent, $\mTheta_2^{(m/2+1)}, \mTheta_2^{(m/2+2)}, \ldots, \mTheta_2^{(m)}$ are independent, and $\E[\mTheta_2^{(r)}]=\mTheta_2^*$ for all $r\in[m]$.
    
    Since the matrices $\mTheta_2^{(r)}$ are possibly unbounded, we will use a variant of the matrix Bernstein inequality for unbounded matrices, which can be found as Proposition 4.1 in~\citet{klochkov2020uniform}.
    There are two main steps in order to use this inequality: (i) showing that $\norm{\mTheta_2^{(r)} - \mTheta_2^*}$ is a sub-exponential random variable for each $r$ and bounding its sub-exponential norm; (ii) bounding the variance $\norm{\sum_{r=1}^{m/2} \E[(\mTheta_2^{(r)} - \mTheta_2^*)^2]}$.
    For the first step, we have
    \begin{align*}
        \norm{\mTheta_2^{(r)} - \mTheta_2^*}
        &\le \norm{\mTheta_2^{(r)}} + \norm{\mTheta_2^*}\\
        &=\norm{\phi(\mX\vw_r/\sqrt{d})}^2 + O(n) \\
        &\lesssim \norm{\phi(\vzero_n)}^2 + \norm{\mX\vw_r/\sqrt{d}}^2 + n \tag{$\phi$ is Lipschitz}\\
        &\lesssim n + \frac{\norm{\mX}^2\norm{\vw_r}^2}{d}\\
        &\lesssim n + \frac nd \norm{\vw_r}^2 .
    \end{align*}
    Since $\norm{\vw_r}^2$ is a $\chi^2$ random variable with $d$ degrees of freedom, it has sub-exponential norm $O(d)$, which implies that the random variable $\norm{\mTheta_2^{(r)} - \mTheta_2^*}$ has sub-exponential norm $O(n)$.
    
    Next we bound the variance. Let $B>0$ be a threshold to be determined. We have:
    \begin{align*}
        &\norm{\E[(\mTheta_2^{(r)}-\mTheta_2^*)^2]}\\
        =\,& \norm{\E[(\mTheta_2^{(r)})^2]-(\mTheta_2^*)^2}\\
        \le\,& \norm{\E[(\mTheta_2^{(r)})^2]}+\norm{\mTheta_2^*}^2 \\
        =\,& \norm{\E_{\vw\sim\N(\vzero,\mI)} \left[ \norm{\phi({\mX\vw}/{\sqrt{d}})}^2 \phi({\mX\vw}/{\sqrt{d}})\phi({\mX\vw}/{\sqrt{d}})^\top \right]} +\norm{\mTheta_2^*}^2 \\
        \le\,& \norm{\E_{\vw\sim\N(\vzero,\mI)} \left[ \ind{\norm{\phi({\mX\vw}/{\sqrt{d}})}\le B}\norm{\phi({\mX\vw}/{\sqrt{d}})}^2 \phi({\mX\vw}/{\sqrt{d}})\phi({\mX\vw}/{\sqrt{d}})^\top \right]}\\& + \norm{\E_{\vw\sim\N(\vzero,\mI)} \left[ \ind{\norm{\phi({\mX\vw}/{\sqrt{d}})}>B}\norm{\phi({\mX\vw}/{\sqrt{d}})}^2 \phi({\mX\vw}/{\sqrt{d}})\phi({\mX\vw}/{\sqrt{d}})^\top \right]} +\norm{\mTheta_2^*}^2 \\
        \le\,& B^2 \norm{\E_{\vw\sim\N(\vzero,\mI)} \left[  \phi({\mX\vw}/{\sqrt{d}})\phi({\mX\vw}/{\sqrt{d}})^\top \right]} \\& + \E_{\vw\sim\N(\vzero,\mI)}\left[\ind{\norm{\phi({\mX\vw}/{\sqrt{d}})}>B}\norm{\phi({\mX\vw}/{\sqrt{d}})}^4 \right] +\norm{\mTheta_2^*}^2 \\
        =\,& B^2 \norm{\mTheta_2^*} + \E_{\vw\sim\N(\vzero,\mI)}\left[\ind{\norm{\phi({\mX\vw}/{\sqrt{d}})}>B}\norm{\phi({\mX\vw}/{\sqrt{d}})}^4 \right] +\norm{\mTheta_2^*}^2\\
        \le\,& B^2 \norm{\mTheta_2^*} +\norm{\mTheta_2^*}^2 + \sqrt{ \E_{\vw\sim\N(\vzero,\mI)}\left[\ind{\norm{\phi({\mX\vw}/{\sqrt{d}})}>B}\right] \cdot \E_{\vw\sim\N(\vzero,\mI)}\left[\norm{\phi({\mX\vw}/{\sqrt{d}})}^8 \right] } \tag{Cauchy-Schwarz inequality} \\
        =\,& B^2 \norm{\mTheta_2^*} +\norm{\mTheta_2^*}^2 + \sqrt{ \Pr_{\vw\sim\N(\vzero,\mI)}\left[{\norm{\phi({\mX\vw}/{\sqrt{d}})}>B}\right] \cdot \E_{\vw\sim\N(\vzero,\mI)}\left[\norm{\phi({\mX\vw}/{\sqrt{d}})}^8 \right] }  .
    \end{align*}
    Note that $\abs{\norm{\phi(\mX\vw/\sqrt{d})} - \norm{\phi({\mX\vw'}/{\sqrt{d}})}} \lesssim \norm{\mX\vw/\sqrt{d} - \mX\vw'/\sqrt{d}} \le \frac{\norm{\mX}}{\sqrt{d}} \norm{\vw-\vw'} \lesssim \sqrt{\frac{n}{d}} \norm{\vw-\vw'}$ for all $\vw,\vw'\in\R^d$. Then by the standard Lipschitz concentration bound for Gaussian variables (see e.g. \citet{wainwright2019high}) we know that for any $s>0$:
    \begin{align*}
        \Pr_{\vw\sim\N(\vzero,\mI)}\left[\norm{\phi({\mX\vw}/{\sqrt{d}})} >M+s\right] \le e^{-\Omega\left(\frac{s^2}{n/d}\right)},
    \end{align*}
    where $M:= \E_{\vw\sim\N(\vzero,\mI)}\left[ \norm{\phi({\mX\vw}/{\sqrt{d}})} \right]$ which can be bounded as
    \begin{align*}
        M^2 &\le \E_{\vw\sim\N(\vzero,\mI)}\left[ \norm{\phi({\mX\vw}/{\sqrt{d}})}^2 \right]\\
        &\lesssim \E_{\vw\sim\N(\vzero,\mI)}\left[ \norm{\phi(\vzero_n)}^2 + \norm{\mX\vw/\sqrt{d}}^2 \right]\\
        &\lesssim n + \frac{n}{d} \E_{\vw\sim\N(\vzero,\mI)}\left[\norm{\vw}^2\right]\\
        &\lesssim n.
    \end{align*}
    Thus, letting $\frac{s^2}{n/d}=C\log n$ for a sufficiently large universal constant $C>0$, we know that with probability at least $1-n^{-10}$ over $\vw\sim\N(\vzero,\mI)$,
    \begin{align*}
        \norm{\phi({\mX\vw}/{\sqrt{d}})} \le M+s
        \lesssim \sqrt{n} + \sqrt{\frac{n}{d}\log n} 
        \lesssim \sqrt{n} .
    \end{align*}
    Hence we pick the threshold $B = C'\sqrt{n}$ which is the upper bound above, where $C'>0$ is a universal constant.
    
    We can also bound
    \begin{align*}
        & \E_{\vw\sim\N(\vzero,\mI)}\left[\norm{\phi({\mX\vw}/{\sqrt{d}})}^8 \right]\\
        =\,& \E_{\vw\sim\N(\vzero,\mI)}\left[ \left( \sum_{i=1}^n \phi(\vx_i^\top\vw/\sqrt{d})^2 \right)^4 \right] \\
        \lesssim \,& \E_{\vw\sim\N(\vzero,\mI)}\left[ \left( \sum_{i=1}^n \left( \phi(0)^2 + (\vx_i^\top\vw/\sqrt{d})^2 \right) \right)^4 \right] \tag{$\phi$ is Lipschitz \& Cauchy-Schwartz inequality}\\
        \lesssim \,& \E_{\vw\sim\N(\vzero,\mI)}\left[ \left( n + \sum_{i=1}^n  (\vx_i^\top\vw/\sqrt{d})^2  \right)^4 \right] \tag{$|\phi(0)|=O(1)$} \\
        \lesssim \,& n^4 + \E_{\vw\sim\N(\vzero,\mI)}\left[ \left( \sum_{i=1}^n  (\vx_i^\top\vw/\sqrt{d})^2  \right)^4 \right] \tag{Jensen's inequality} \\
        =\,& n^4 + n^4 \E_{\vw\sim\N(\vzero,\mI)}\left[ \left( \frac{1}{n} \sum_{i=1}^n  (\vx_i^\top\vw/\sqrt{d})^2  \right)^4 \right]\\
        \le\,& n^4 + n^4 \E_{\vw\sim\N(\vzero,\mI)}\left[  \frac{1}{n} \sum_{i=1}^n  (\vx_i^\top\vw/\sqrt{d})^8   \right] \tag{Jensen's inequality} \\
        =\,& n^4 + n^3 \sum_{i=1}^n \E_{x\sim\N(0,\norm{\vx_i}^2/d)}[x^8]\\
        \lesssim\,& n^4. \tag{$\norm{\vx_i}^2/d=O(1)$}
    \end{align*}
    
    Combining all the above, we get
    \begin{align*}
        &\norm{\E[(\mTheta_2^{(r)}-\mTheta_2^*)^2]}\\
        \le\,& B^2 \norm{\mTheta_2^*} +\norm{\mTheta_2^*}^2 + \sqrt{ \Pr_{\vw\sim\N(\vzero,\mI)}\left[{\norm{\phi({\mX\vw}/{\sqrt{d}})}>B}\right] \cdot \E_{\vw\sim\N(\vzero,\mI)}\left[\norm{\phi({\mX\vw}/{\sqrt{d}})}^8 \right] } \\
        \lesssim\,& n \norm{\mTheta_2^*} +\norm{\mTheta_2^*}^2 + \sqrt{n^{-10}\cdot n^4}\\
        =\,& n \norm{\mTheta_2^*} +\norm{\mTheta_2^*}^2 + n^{-3}.
    \end{align*}
    
    We will discuss two cases separately.
    
    \paragraph{Case 1: $\vartheta_0\not=0$.}
    Recall that in this case Theorem~\ref{thm:second-layer-main} assumes $m\gtrsim d^{2+\alpha}$.
    
    Since $\norm{\mTheta_2^*}=O(n)$, we have
    $
        \norm{\E[(\mTheta_2^{(r)}-\mTheta_2^*)^2]} \lesssim n^2
    $
    which implies
    \[
        \norm{\sum_{r=1}^{m/2} \E[(\mTheta_2^{(r)} - \mTheta_2^*)^2]} \lesssim mn^2.
    \]
    Applying Proposition 4.1 in~\citet{klochkov2020uniform}, we know that for any $u\gg \max\{ n\log m, n\sqrt{m} \} = n\sqrt{m}$,
    \begin{align*}
        \Pr\left[ \norm{\sum_{r=1}^{m/2}(\mTheta_2^{(r)}-\mTheta_2^*)} >u  \right] \lesssim n\cdot \exp\left( -\Omega\left( \min\left\{ \frac{u^2}{mn^2}, \frac{u}{n\log m} \right\} \right) \right).
    \end{align*}
    Let $u = m\cdot\frac{n}{d^{1+\frac{\alpha}{3}}}$. We can verify $u\gg n\sqrt{m}$ since $m\gtrsim d^{2+\alpha}$. Then we have
    \begin{align*}
        \Pr\left[ \norm{\sum_{r=1}^{m/2}(\mTheta_2^{(r)}-\mTheta_2^*)} > m\cdot\frac{n}{d^{1+\frac{\alpha}{3}}}  \right]
        \lesssim n\cdot \exp\left( -\Omega\left( \min\left\{ \frac{m}{d^{2+\frac{2\alpha}{3}}}, \frac{m}{d^{1+\frac{\alpha}{3}}\log m} \right\} \right) \right)
        \ll 1.
    \end{align*}
    Similarly, for the second half of the neurons we also have $\norm{\sum_{r=m/2+1}^{m}(\mTheta_2^{(r)}-\mTheta_2^*)} \le m\cdot\frac{n}{d^{1+\frac{\alpha}{3}}}$ with high probability. Therefore we have with high probability,
    \begin{align*}
        \norm{\mTheta_2 - \mTheta_2^*} \lesssim \frac{n}{d^{1+\frac{\alpha}{3}}}.
    \end{align*}
    
    \paragraph{Case 2: $\vartheta_0=0$.}
    Recall that in this case Theorem~\ref{thm:second-layer-main} assumes $m\gtrsim d^{1+\alpha}$.
    
    Since $\norm{\mTheta_2^*}=\tilde{O}(n/d)$, we have
    $
        \norm{\E[(\mTheta_2^{(r)}-\mTheta_2^*)^2]} \lesssim n\cdot \tilde{O}(n/d) + \tilde{O}((n/d)^2) + n^{-3}
        = \tilde{O}(n^2/d)
    $
    which implies
    \[
        \norm{\sum_{r=1}^{m/2} \E[(\mTheta_2^{(r)} - \mTheta_2^*)^2]} \le \tilde{O}(mn^2/d) \lesssim \frac{mn^2}{d^{1-\frac{\alpha}{10}}}.
    \]
    Applying Proposition 4.1 in~\citet{klochkov2020uniform}, we know that for any $u\gg \max\{ n\log m, n\sqrt{m/d^{1-\frac{\alpha}{10}}} \} = n\sqrt{m/d^{1-\frac{\alpha}{10}}}$,
    \begin{align*}
        \Pr\left[ \norm{\sum_{r=1}^{m/2}(\mTheta_2^{(r)}-\mTheta_2^*)} >u  \right] \lesssim n\cdot \exp\left( -\Omega\left( \min\left\{ \frac{u^2}{mn^2/d^{1-\frac{\alpha}{10}}}, \frac{u}{n\log m} \right\} \right) \right).
    \end{align*}
    Let $u = m\cdot\frac{n}{d^{1+\frac{\alpha}{3}}}$. We can verify $u\gg n\sqrt{m/d^{1-\frac{\alpha}{10}}}$ since $m\gtrsim d^{1+\alpha}$. Then we have
    \begin{align*}
        \Pr\left[ \norm{\sum_{r=1}^{m/2}(\mTheta_2^{(r)}-\mTheta_2^*)} > m\cdot\frac{n}{d^{1+\frac{\alpha}{3}}}  \right]
        \lesssim n\cdot \exp\left( -\Omega\left( \min\left\{ \frac{m}{d^{1+0.77\alpha}}, \frac{m}{d^{1+\frac{\alpha}{3}}\log m} \right\} \right) \right)
        \ll 1.
    \end{align*}
    Similarly, for the second half of the neurons we also have $\norm{\sum_{r=m/2+1}^{m}(\mTheta_2^{(r)}-\mTheta_2^*)} \le m\cdot\frac{n}{d^{1+\frac{\alpha}{3}}}$ with high probability. Therefore we have with high probability,
    \begin{align*}
        \norm{\mTheta_2 - \mTheta_2^*} \lesssim \frac{n}{d^{1+\frac{\alpha}{3}}}.
    \end{align*}
    The proof is completed.
\end{proof}

Combining Propositions~\ref{prop:second-layer-ana-ntk-approx} and~\ref{prop:second-layer-ntk-concentration} directly gives Proposition~\ref{prop:second-layer-ntk-init-approx}.

\subsubsection{Agreement on Training Data}

To prove the agreement between $f^2_t$ and $\flintwo_t$ on training data for all $t\le T = c\cdot\frac{d\log d}{\eta_2}$, we still apply Theorem~\ref{thm:general-closeness}.
This case is much easier than training the first layer (\Cref{app:proof-first-layer-training}), since the Jacobian for the second layer does not change during training, and thus Proposition~\ref{prop:second-layer-ntk-init-approx} already verifies Assumption~\ref{asmp:general-jacobian-closeness}.
Therefore we can directly instantiate Theorem~\ref{thm:general-closeness} with $\epsilon = C\frac{n}{d^{1+\frac{\alpha}{3}}}$ (for a sufficiently large constant $C$) and $R = \sqrt{d\log d}$, which gives (notice that the choice of $\eta_2$ in Theorem~\ref{thm:second-layer-main} also satisfies the condition in Theorem~\ref{thm:general-closeness})
\begin{align*}
        \sqrt{\sum_{i=1}^n (f_t^2(\vx_i)-\flintwo_t(\vx_i))^2} \lesssim \frac{\eta_2 t \epsilon}{\sqrt{n}} \lesssim  \frac{d\log d\cdot \frac{n}{d^{1+\frac\alpha3}}}{\sqrt n} = \frac{\sqrt{n}\log d}{d^{\frac{\alpha}{3}}}
        \ll \frac{\sqrt n}{d^{\frac{\alpha}{4}}},
    \end{align*}
    i.e.,
    \begin{align*}
        \frac{1}{n}\sum_{i=1}^n (f_t^2(\vx_i)-\flintwo_t(\vx_i))^2 \le d^{-\frac{\alpha}{2}}.
    \end{align*}
    This proves the first part in Theorem~\ref{thm:second-layer-main}.
    
    Note that Theorem~\ref{thm:general-closeness} also tells us $\norm{\vv(t)-\vv(0)}\le\sqrt{d\log d}$ and $\norm{\vgamma(t)}\le\sqrt{d\log d}$, which will be useful for proving the guarantee on the distribution $\D$.

\subsubsection{Agreement on Distribution} \label{app:proof-second-layer-distribution}

Now we prove the second part in Theorem~\ref{thm:second-layer-main}, which is the agreement between $f^2_t$ and $\flintwo_t$ on the distribution $\D$.
The proof is similar to the case of training the first layer (\Cref{app:proof-first-layer-distribution}), but our case here is again simpler. In particular, we do not need to define an auxiliary model anymore because $f(\vx;\mW,\vv)$ is already linear in the parameters $\vv$.
Now that $f^2_t - \flintwo_t$ is a linear model (in some feature space) with bounded parameters, we can bound the Rademacher complexity of the linear function class it belongs to, similar to \Cref{app:proof-first-layer-distribution}.
Similar to~\eqref{eqn:rademacher-first-layer}, we can bound the Rademacher complexity by
\begin{align*}
    \frac{\sqrt{d\log d}}{n} \sqrt{\Tr[\mTheta_2(\mW)] + \Tr[\mThetalintwo]} .
\end{align*}
Next we bound the above two traces.
First, we have
\begin{align*}
    \index{\mTheta_2(\mW)}{i,i}
    &= \frac{1}{m} \sum_{r=1}^m \phi(\vx_i^\top\vw_r/\sqrt{d})^2
    \lesssim \frac{1}{m} \sum_{r=1}^m \left( \phi(0)^2 + (\vx_i^\top\vw_r/\sqrt{d})^2 \right) \\
    &= 1 + \frac{1}{dm} \sum_{r=1}^m (\vx_i^\top\vw_r)^2
    \lesssim 1 + \frac{1}{dm}(dm+\log n) \lesssim 1
\end{align*}
with high probability for all $i\in[n]$ together. Here we have used the standard tail bound for $\chi^2$ random variables and a union bound over $i\in[n]$.
Hence we have $\Tr[\mTheta_2(\mW)]\lesssim n$.
For the second trace, we have
\begin{align*}
    \Tr[\mThetalintwo] = \sum_{i=1}^n \left( \zeta^2\frac{\norm{\vx_i}^2}{d} + \frac{\nu^2}{2d} + \index{\vq}{i}^2 \right) \lesssim n
\end{align*}
with high probability.
Therefore we can bound the Rademacher complexity by $\sqrt{\frac{d\log d}{n}}$. Then we can conclude the agreement guarantee on the distribution $\D$, i.e., for all $t\le T$ simultaneously,
\begin{align*} 
    &\E_{\vx\sim\D}\left[ \min\left\{\left(f^2_t(\vx) - \flintwo_t(\vx) \right)^2, 1\right\} \right]\\
    \le\,& \frac{1}{n} \sum_{i=1}^n \min\left\{\left(f^2_t(\vx_i) - \flintwo_t(\vx_i) \right)^2, 1\right\} + O\left( \sqrt{\frac{d\log d}{n}} \right) + O\left( \frac{1}{\sqrt{n}} \right)\\
    \lesssim\,& d^{-\frac{\alpha}{2}} + \sqrt{\frac{d\log d}{d^{1+\alpha}}} \tag{$n\gtrsim d^{1+\alpha}$}\\
    \lesssim\,& d^{-\frac{\alpha}{2}} . 
\end{align*}
This completes the proof of Theorem~\ref{thm:second-layer-main}.

\subsection{Proof of Theorem~\ref{thm:both-layers-main} (Training Both Layers)} \label{app:proof-both-layers}

The proof for training both layers follows the same ideas in the proofs for training the first layer only and the second layer only. In fact, most technical components needed in the proof were already developed in the previous proofs. The only new component is a Jacobian perturbation bound for the case of training both layers, Lemma~\ref{lem:jacobian-perturb-both-layers} (analog of Lemma~\ref{lem:jacobian-perturb-first-layer} for training the first layer). 

As before, we proceed in three steps.

\subsubsection{The NTK at Initialization}

\begin{prop}\label{prop:both-layer-ntk-init-approx}
    With high probability over the random initialization $(\mW(0),\vv(0))$ and the training data $\mX$, we have
    \[
        \norm{ \mTheta(\mW(0),\vv(0)) - \mThetalin } \lesssim \frac{n}{d^{1+\frac{\alpha}{3}}} .
    \]
\end{prop}
\begin{proof}
    This is a direct corollary of Propositions~\ref{prop:first-layer-ntk-init-approx-copy} and~\ref{prop:second-layer-ntk-init-approx}, given that $\mTheta(\mW(0),\vv(0))=\mTheta_1(\mW(0),\vv(0)) + \mTheta(\mW(0))$ (\eqref{eqn:two-layer-NTKs}) and $\mThetalin=\mThetalinone+\mThetalintwo$ (\eqref{eqn:lin-kernels}).
\end{proof}

\subsubsection{Agreement on Training Data}

The proof for the agreement on training data is similar to the case of training the first layer only (\Cref{app:proof-first-layer-training}).
We will again apply Theorem~\ref{thm:general-closeness}. 
For this we need a new Jacobian perturbation lemma to replace Lemma~\ref{lem:jacobian-perturb-first-layer}, since both layers are allowed to move now.

\begin{lem}[Jacobian perturbation for both layers] \label{lem:jacobian-perturb-both-layers}
    If $\phi$ is a smooth activation as in Assumption~\ref{asmp:activation}, then with high probability over the training data $\mX$, we have
    \begin{equation} \label{eqn:jacobian-perturb-both-layer-1-smooth}
        \norm{\mJ_1(\mW,\vv) - \mJ_1(\mW(0),\vv(0))} \lesssim \sqrt{\tfrac{n}{md}} \norm{\mW-\mW(0)}_F + \sqrt{\tfrac{n}{m}} \norm{\vv-\vv(0)},
        \quad \forall \mW  , \vv .
    \end{equation}
    If $\phi$ is a piece-wise linear activation as in Assumption~\ref{asmp:activation}, then with high probability over the random initialization $\mW(0)$ and the training data $\mX$, we have 
    \begin{equation} \label{eqn:jacobian-perturb-both-layer-1-relu}
    \begin{aligned}
        \norm{\mJ_1(\mW,\vv) - \mJ_1(\mW(0),\vv(0))} \lesssim \sqrt{\tfrac{n}{d}} \left( \tfrac{\norm{\mW - \mW(0)}^{1/3}}{m^{1/6}} + \left(\tfrac{\log n}{m}\right)^{1/4} \right) + \sqrt{\tfrac{n}{md}} \norm{\vv-\vv(0)}, \\ \forall \mW,\vv.
    \end{aligned}
    \end{equation}
    
    Furthermore, with high probability over the training data $\mX$, we have
    \begin{equation} \label{eqn:jacobian-perturb-both-layer-2}
        \norm{\mJ_2(\mW) - \mJ_2(\tmW)} \lesssim \sqrt{\frac{n}{md}}\norm{\mW-\tmW}_F, \qquad \forall \mW,\tmW .
    \end{equation}
\end{lem}
\begin{proof}
    We will be conditioned on $\mX$ and on the high-probability events in Claim~\ref{claim:data-concentration}.
    
    We first consider the first-layer Jacobian.
    By the definition of $\mJ_1(\mW,\vv)$ in~\eqref{eqn:first-layer-jacobian}, we have
    \begin{equation} \label{eqn:jacobian-diff-both-layer-1}
    \begin{aligned}
        &(\mJ_1(\mW,\vv) - \mJ_1(\mW(0),\vv(0)))(\mJ_1(\mW,\vv) - \mJ_1(\mW(0),\vv(0)))^\top \\
        =\,& \frac{1}{md} \Bigg( \left(\phi'\left(\mX\mW^\top/\sqrt{d}\right)\diag(\vv) - \phi'\left(\mX\mW(0)^\top/\sqrt{d}\right)\diag(\vv(0))\right) \\ \ \ \ \ \ \ &\cdot \left(\phi'\left(\mX\mW^\top/\sqrt{d}\right)\diag(\vv) - \phi'\left(\mX\mW(0)^\top/\sqrt{d}\right)\diag(\vv(0))\right)^\top \Bigg) \odot (\mX\mX^\top) .
    \end{aligned}
    \end{equation}
    Then if $\phi$ is a smooth activation, we have with high probability,
    \begin{align*}
        &\norm{\mJ_1(\mW,\vv) - \mJ_1(\mW(0),\vv(0))}^2 \\
        \le\,& \frac{1}{md} \norm{\phi'\left(\mX\mW^\top/\sqrt{d}\right)\diag(\vv) - \phi'\left(\mX\mW(0)^\top/\sqrt{d}\right)\diag(\vv(0))}^2 \cdot \max_{i\in[n]} \norm{\vx_i}^2 \tag{\eqref{eqn:jacobian-diff-both-layer-1} and Lemma~\ref{lem:hadamard-product-bound}} \\
        \lesssim \,& \frac{1}{m} \norm{\phi'\left(\mX\mW^\top/\sqrt{d}\right)\diag(\vv) - \phi'\left(\mX\mW(0)^\top/\sqrt{d}\right)\diag(\vv(0))}^2 \tag{Claim~\ref{claim:data-concentration}} \\
        \lesssim \,& \frac{1}{m} \norm{\phi'\left(\mX\mW^\top/\sqrt{d}\right)\diag(\vv(0)) - \phi'\left(\mX\mW(0)^\top/\sqrt{d}\right)\diag(\vv(0))}^2 \\ & \ + \frac{1}{m} \norm{\phi'\left(\mX\mW^\top/\sqrt{d}\right)\diag(\vv) - \phi'\left(\mX\mW^\top/\sqrt{d}\right)\diag(\vv(0))}^2 \\
        \le \,& \frac{1}{m} \norm{\phi'\left(\mX\mW^\top/\sqrt{d}\right) - \phi'\left(\mX\mW(0)^\top/\sqrt{d}\right)}_F^2 \cdot \norm{\diag(\vv(0))}^2 \\ & \ + \frac{1}{m} \norm{\phi'\left(\mX\mW^\top/\sqrt{d}\right)\diag(\vv-\vv(0))}_F^2 \\
        \le\,& \frac{n}{md}\norm{\mW-\mW(0)}_F^2 + \frac{1}{m} \norm{\phi'\left(\mX\mW^\top/\sqrt{d}\right)\diag(\vv-\vv(0))}_F^2 \tag{using the proof of Lemma~\ref{lem:jacobian-perturb-first-layer}, and $\norm{\diag(\vv(0))}=1$}\\
        \le\,& \frac{n}{md}\norm{\mW-\mW(0)}_F^2 + \frac{n}{m} \norm{\vv-\vv(0)}^2 \tag{$\phi'$ is bounded}.
    \end{align*}
    This proves~\eqref{eqn:jacobian-perturb-both-layer-1-smooth}.
    
    If $\phi$ is a piece-wise linear activation, then with high probability,
    \begin{align*}
        &\norm{\mJ_1(\mW,\vv) - \mJ_1(\mW(0),\vv(0))}^2 \\
        \le\,& \frac{1}{md}\norm{\mX\mX^\top} \cdot \max_{i\in[n]} \norm{\diag(\vv)\phi'(\mW\vx_i/\sqrt{d}) - \diag(\vv(0))\phi'(\mW(0)\vx_i/\sqrt{d})}^2 \tag{\eqref{eqn:jacobian-diff-both-layer-1} and Lemma~\ref{lem:hadamard-product-bound}} \\
        \le\,& \frac{n}{md} \cdot \max_{i\in[n]} \norm{\diag(\vv)\phi'(\mW\vx_i/\sqrt{d}) - \diag(\vv(0))\phi'(\mW(0)\vx_i/\sqrt{d})}^2 \tag{Claim~\ref{claim:data-concentration}} \\
        \lesssim \,& \frac{n}{md} \cdot \max_{i\in[n]} \norm{\diag(\vv(0))\phi'(\mW\vx_i/\sqrt{d}) - \diag(\vv(0))\phi'(\mW(0)\vx_i/\sqrt{d})}^2 \\&\ + \frac{n}{md} \cdot \max_{i\in[n]} \norm{\diag(\vv)\phi'(\mW\vx_i/\sqrt{d}) - \diag(\vv(0))\phi'(\mW\vx_i/\sqrt{d})}^2 \\
        \le \,& \frac{n}{md} \cdot \max_{i\in[n]} \norm{\phi'(\mW\vx_i/\sqrt{d}) - \phi'(\mW(0)\vx_i/\sqrt{d})}^2 + \frac{n}{md} \cdot \max_{i\in[n]} \norm{\diag(\vv-\vv(0))\phi'(\mW\vx_i/\sqrt{d})}^2 \tag{$\norm{\diag(\vv(0))}=1$} \\
        \lesssim\,& \frac{n}{d} \left( \frac{\norm{\mW - \mW(0)}^{2/3}}{m^{1/3}} + \sqrt{\frac{\log n}{m}} \right) + \frac{n}{md} \norm{\vv-\vv(0)}^2 \tag{using the proof of Lemma~\ref{lem:jacobian-perturb-first-layer}, and $\phi'$ is bounded}.
    \end{align*}
    This proves~\eqref{eqn:jacobian-perturb-both-layer-1-relu}.

    For the second-layer Jacobian, we have with high probability,
    \begin{align*}
        \norm{\mJ_2(\mW) - \mJ_2(\tmW)}
        &= \frac{1}{\sqrt{m}}\norm{\phi(\mX\mW^\top/\sqrt{d}) - \phi(\mX\tmW^\top/\sqrt{d})} \\
        &\le \frac{1}{\sqrt{m}}\norm{\mX(\mW-\tmW)^\top/\sqrt{d}}_F \tag{$\phi'$ is bounded} \\
        &\le \frac{\norm{\mX}}{\sqrt{md}} \norm{\mW-\tmW}_F \\
        &\le \sqrt{\frac{n}{md}}\norm{\mW-\tmW}_F,
    \end{align*}
    completing the proof of~\eqref{eqn:jacobian-perturb-both-layer-2}.
\end{proof}

Based on Lemma~\ref{lem:jacobian-perturb-both-layers}, we can now verify Assumption~\ref{asmp:general-jacobian-closeness} for the case of training both layers:
\begin{lem} \label{lem:verify-asmp-both-layers}
    Let $R = \sqrt{d\log d}$.
    With high probability over the random initialization and the training data, for all $(\mW,\vv)$ and $(\tmW,\tvv)$ such that $\norm{\mW-\mW(0)}_F\le R$, $\norm{\tmW-\mW(0)}_F\le R$, $\norm{\vv-\vv(0)}\le R$ and $\norm{\tvv-\vv(0)}\le R$, we have
    \[
        \norm{\mJ(\mW,\vv)\mJ(\tmW,\tvv)^\top - \mThetalin} \lesssim \frac{n}{d^{1+\frac\alpha3}} .
    \]
\end{lem}
\begin{proof}
    This proof is conditioned on all the high-probability events we have shown.
    
    Now consider $(\mW,\vv)$ and $(\tmW,\tvv)$ which satisfy the conditions stated in the lemma.
    
    If $\phi$ is a smooth activation, from Lemma~\ref{lem:jacobian-perturb-both-layers} we know
    \begin{align*}
        \norm{\mJ_1(\mW,\vv) - \mJ_1(\mW(0),\vv(0))}
        &\lesssim \sqrt{\frac{n}{md}} \norm{\mW-\mW(0)}_F + \sqrt{\frac{n}{m}} \norm{\vv-\vv(0)} \\
        & \le \sqrt{\frac{n}{md}} \cdot \sqrt{d\log d} + \sqrt{\frac{n}{m}} \cdot \sqrt{d\log d}\\
        &\lesssim \sqrt{\frac{nd\log d}{m}} \\
        &\lesssim \sqrt{\frac{n\log d}{d^{1+\alpha}}}\\
        &\ll \sqrt{\frac{n}{d^{1+\frac{2\alpha}{3}}}},
    \end{align*}
    where we have used $m\gtrsim d^{2+\alpha}$.
    If $\phi$ is a piece-wise linear activation, from Lemma~\ref{lem:jacobian-perturb-both-layers} we have
    \begin{align*}
        \norm{\mJ_1(\mW,\vv) - \mJ_1(\mW(0),\vv(0))} &\lesssim \sqrt{\frac{n}{d}} \left( \frac{\norm{\mW - \mW(0)}^{1/3}}{m^{1/6}} + \left(\frac{\log n}{m}\right)^{1/4} \right) + \sqrt{\frac{n}{md}}\norm{\vv-\vv(0)} \\
        &\le \sqrt{\frac{n}{d}} \left( \frac{(d\log d)^{1/6}}{m^{1/6}} + \left(\frac{\log n}{m}\right)^{1/4} \right) + \sqrt{\frac{n\log d}{m}}  \\
        &\lesssim \sqrt{\frac{n}{d}} \cdot \frac{(d\log d)^{1/6}}{d^{1/3+\alpha/6}} + \sqrt{\frac{n\log d}{d^{2+\alpha}}} \\
        &\ll \frac{\sqrt{n}}{d^{\frac23}} .
    \end{align*}
    Hence in either case have $\norm{\mJ_1(\mW,\vv) - \mJ_1(\mW(0),\vv(0))} \le \frac{\sqrt{n}}{d^{\frac12 + \frac{\alpha}{3}}}$.
    Similarly, we have $\norm{\mJ_1(\tmW,\tvv) - \mJ_1(\mW(0),\vv(0))} \le \frac{\sqrt{n}}{d^{\frac12 + \frac{\alpha}{3}}}$.
    
    Also, we know from Proposition~\ref{prop:first-layer-ntk-init-approx-copy} that $\norm{\mJ_1(\mW(0),\vv(0))}\lesssim \sqrt{\frac{n}{d}}$. It follows that $\norm{\mJ_1(\mW,\vv)}\lesssim \sqrt{\frac{n}{d}}$ and $\norm{\mJ_1(\tmW,\tvv)}\lesssim \sqrt{\frac{n}{d}}$.
    Then we have
    \begin{align*}
        &\norm{ \mJ_1(\mW,\vv)\mJ_1(\tmW,\tvv)^\top - \mJ_1(\mW(0),\vv(0))\mJ_1(\mW(0),\vv(0))^\top  } \\
        \le\,& \norm{\mJ_1(\mW,\vv)}\cdot\norm{\mJ_1(\tmW,\tvv)- \mJ_1(\mW(0),\vv(0))} + \norm{\mJ_1(\mW(0),\vv(0))}\cdot\norm{\mJ_1(\mW,\vv)-\mJ_1(\mW(0),\vv(0))}\\
        \lesssim\,& \sqrt{\frac nd}\cdot\frac{\sqrt{n}}{d^{\frac12 + \frac{\alpha}{3}}} + \sqrt{\frac nd}\cdot\frac{\sqrt{n}}{d^{\frac12 + \frac{\alpha}{3}}}\\
        \lesssim\,& \frac{n}{d^{1+\frac{\alpha}{3}}} .
    \end{align*}
    
    Next we look at the second-layer Jacobian. From Lemma~\ref{lem:jacobian-perturb-both-layers} we know $\norm{\mJ_2(\mW)-\mJ_2(\mW(0))} \lesssim \sqrt{\frac{n}{md}}\cdot\sqrt{d\log d} \lesssim \sqrt{\frac{n\log d}{d^{2+\alpha}}} \ll \frac{\sqrt{n}}{d^{1+\frac{\alpha}{3}}}$.
    Similarly we have $\norm{\mJ_2(\tmW)-\mJ_2(\mW(0))}  \ll \frac{\sqrt{n}}{d^{1+\frac{\alpha}{3}}}$.
    Also, from Proposition~\ref{prop:second-layer-ntk-init-approx} we know $\norm{\mJ_2(\mW(0))} \lesssim \sqrt{n}$, which implies $\norm{\mJ_2(\mW)} \lesssim \sqrt{n}$ and $\norm{\mJ_2(\tmW)} \lesssim \sqrt{n}$.
    It follows that
    \begin{align*}
        &\norm{\mJ_2(\mW)\mJ_2(\tmW)^\top - \mJ_2(\mW(0))\mJ_2(\mW(0))^\top} \\
        \le& \norm{\mJ_2(\mW)} \cdot \norm{\mJ_2(\tmW) - \mJ_2(\mW(0))} + \norm{\mJ_2(\mW(0))} \cdot \norm{\mJ_2(\mW) - \mJ_2(\mW(0))} \\
        \lesssim& \sqrt{n} \cdot \frac{\sqrt{n}}{d^{1+\frac{\alpha}{3}}} + \sqrt{n} \cdot \frac{\sqrt{n}}{d^{1+\frac{\alpha}{3}}}\\
        \lesssim& \frac{n}{d^{1+\frac{\alpha}{3}}} .
    \end{align*}
    
    Combining the above auguments for two layers, we obtain
    \begin{align*}
    &\norm{\mJ(\mW,\vv)\mJ(\tmW,\tvv)^\top - \mJ(\mW(0),\vv(0))\mJ(\mW(0),\vv(0))^\top}\\
    =\,& \Big\| \mJ_1(\mW,\vv)\mJ_1(\tmW,\tvv)^\top + \mJ_2(\mW)\mJ_2(\tmW)^\top \\&\  - \mJ_1(\mW(0),\vv(0))\mJ_1(\mW(0),\vv(0))^\top - \mJ_2(\mW(0))\mJ_2(\mW(0))^\top \Big\| \\
    \le\,& \norm{ \mJ_1(\mW,\vv)\mJ_1(\tmW,\tvv)^\top - \mJ_1(\mW(0),\vv(0))\mJ_1(\mW(0),\vv(0))^\top}\\&\ + \norm{ \mJ_2(\mW)\mJ_2(\tmW)^\top - \mJ_2(\mW(0))\mJ_2(\mW(0))^\top } \\
    \lesssim\,& \frac{n}{d^{1+\frac{\alpha}{3}}}+\frac{n}{d^{1+\frac{\alpha}{3}}}\\
    \lesssim\,& \frac{n}{d^{1+\frac{\alpha}{3}}} .
    \end{align*}
    Combining the above inequality with Proposition~\ref{prop:both-layer-ntk-init-approx}, the proof is finished.
\end{proof}

Finally, we can apply Theorem~\ref{thm:general-closeness} with $R=\sqrt{d\log d}$ and $\epsilon = O(\frac{n}{d^{1+\frac{\alpha}{3}}})$, and obtain that for all $t\le T$:
\begin{align*}
        \sqrt{\sum_{i=1}^n (f_t(\vx_i)-\flin(\vx_i))^2} \lesssim \frac{\eta t \epsilon}{\sqrt{n}} \lesssim  \frac{d\log d\cdot \frac{n}{d^{1+\frac\alpha3}}}{\sqrt n} = \frac{\sqrt{n}\log d}{d^{\frac{\alpha}{3}}}
        \ll \frac{\sqrt n}{d^{\frac{\alpha}{4}}},
    \end{align*}
    i.e.,
    \begin{align*}
        \frac{1}{n}\sum_{i=1}^n (f_t(\vx_i)-\flin_t(\vx_i))^2 \le d^{-\frac{\alpha}{2}}.
    \end{align*}
    This proves the first part in Theorem~\ref{thm:both-layers-main}.
    
    Note that Theorem~\ref{thm:general-closeness} also tells us $\norm{\mW(t)-\mW(0)}\le\sqrt{d\log d}$, $\norm{\vv(t)-\vv(0)}\le\sqrt{d\log d}$ and $\norm{\vdelta(t)}\le\sqrt{d\log d}$, which will be useful for proving the guarantee on the distribution $\D$.

\subsubsection{Agreement on Distribution}

The proof for the second part of Theorem~\ref{thm:both-layers-main} is basically identical to the case of training the first layer (\Cref{app:proof-first-layer-distribution}), so we will only sketch the differences here to avoid repetition.

Recall that in \Cref{app:proof-first-layer-distribution} we define an auxiliary model which is the first-order approximation of the network around initialization. Here since we are training both layers, we need to modify the definition of the auxiliary model to incorporate deviation from initialization in both layers:
\begin{align*}
    \faux(\vx;\mW,\vv) := \langle \mW-\mW(0), \nabla_{\mW}f(\vx;\mW(0),\vv(0)) \rangle + \langle \vv-\vv(0), \nabla_{\vv}f(\vx;\mW(0),\vv(0)) \rangle.
\end{align*}
Then we denote $\faux_t(\vx):=\faux(\vx;\mW(t),\vv(t))$.

There are two more minor changes to \Cref{app:proof-first-layer-distribution}:
\begin{enumerate}
    \item When proving $f_t$ and $\faux_t$ are close on both training data and imaginary test data, we need to bound a Jacobian perturbation. In \Cref{app:proof-first-layer-distribution} this step is done using Lemma~\ref{lem:jacobian-perturb-first-layer}. Now we simply need to use Lemma~\ref{lem:jacobian-perturb-both-layers} instead and note that $\norm{\mW(t)-\mW(0)}\le\sqrt{d\log d}$ and $\norm{\vv(t)-\vv(0)}\le\sqrt{d\log d}$.
    
    \item Instead of~\eqref{eqn:rademacher-first-layer}, the empirical Rademacher complexity of the function class that each $\faux_t - \flin_t$ lies in will be
    \begin{align*}
        &\frac{\sqrt{d\log d}}{n} \sqrt{\Tr[\mTheta(\mW(0),\vv(0))] + \Tr[\mThetalin]}\\
        =\,& \frac{\sqrt{d\log d}}{n} \sqrt{\Tr[\mTheta_1(\mW(0),\vv(0))] + \Tr[\mTheta_2(\mW(0))] + \Tr[\mThetalinone] + \Tr[\mThetalintwo]}.
    \end{align*}
    In \Cref{app:proof-first-layer-distribution,app:proof-second-layer-distribution}, we have shown that the above 4 traces are all $O(n)$ with high probability. Hence we get the same Rademacher complexity bound as before.
\end{enumerate}

Modulo these differences, the proof proceeds the same as \Cref{app:proof-first-layer-distribution}. Therefore we conclude the proof of Theorem~\ref{thm:both-layers-main}.

\subsection{Proof of Claim~\ref{claim:data-concentration}}
\label{app:two-layer-data-concentration-proof}

\begin{proof}[Proof of Claim~\ref{claim:data-concentration}]
    According to Assumption~\ref{asmp:input-distr}, we have $\vx_i = \mSigma^{1/2}\bar{\vx}_i$ where $\E[\bar{\vx}_i]=\vzero$, $\E[\bar{\vx}_i\bar{\vx}_i^\top]=\mI$, and $\bar{\vx}_i$'s entries are independent and $O(1)$-subgaussian.
    
    By Hanson-Wright inequality (specifically, Theorem 2.1 in~\citet{rudelson2013hanson}), we have for any $t\ge0$,
    \begin{align*}
        \Pr\left[ \abs{ \norm{\mSigma^{1/2}\bar{\vx}_i} - \|\mSigma^{1/2}\|_F} >t \right] \le 2\exp\left(-\Omega\left(\frac{t^2}{\norm{\mSigma^{1/2}}^2}\right)\right),
    \end{align*}
    i.e.,
    \begin{align*}
        \Pr\left[ \abs{ \norm{\vx_i} - \sqrt{d} } >t \right] \le 2\exp\left(-\Omega\left(t^2\right)\right).
    \end{align*}
    Let $t = C\sqrt{\log n}$ for a sufficiently large constant $C>0$. 
    Taking a union bound over all $i\in[n]$, 
    we obtain that with high probability, $\norm{\vx_i} = \sqrt{d} \pm O(\sqrt{\log n})$ for all $i\in[n]$ simultaneously. This proves the first property in Claim~\ref{claim:data-concentration}.
    
    For $i\not=j$, we have $\langle \vx_i,\vx_j \rangle = \bar{\vx}_i^\top \mSigma \bar{\vx}_j$. Conditioned on $\bar{\vx}_j$, we know that $\bar{\vx}_i^\top \mSigma \bar{\vx}_j$ is zero-mean and $O(\norm{\mSigma\bar{\vx}_j}^2)$-subgaussian, which means for any $t\ge0$,
    \[
        \Pr\left[ \abs{\bar{\vx}_i^\top \mSigma \bar{\vx}_j}>t \,\Big|\, \bar{\vx}_j \right] \le 2\exp\left(-\frac{t^2}{\norm{\mSigma\bar{\vx}_j}^2}\right).
    \]
    Since we have shown that $\norm{\mSigma\bar{\vx}_j}^2 \lesssim \norm{\vx_j}^2 \lesssim \sqrt{d} + \sqrt{\log n} \lesssim \sqrt{d}$ with probability at least $1-n^{-10}$, we have
    \begin{align*}
        \Pr\left[ \abs{\bar{\vx}_i^\top \mSigma \bar{\vx}_j}>t \right] \le n^{-10} + 2\exp\left( -\Omega\left(\frac{t^2}{d}\right) \right).
    \end{align*}
    Then we can take $t=C\sqrt{d\log n}$ and apply a union bound over $i,j$, which gives $\abs{\langle\vx_i,\vx_j\rangle} \lesssim \sqrt{d\log n}$ for all $i\not=j$ with high probability. This completes the proof of the second statement in Claim~\ref{claim:data-concentration}.
    
    Finally, for $\mX\mX^\top$, we can use standard covariance concentration (see, e.g., Lemma A.6 in~\citet{du2020few}) to obtain $0.9\mSigma \preceq \frac1n\mX^\top\mX \preceq 1.1\mSigma$ with high probability. This implies $\norm{\mX\mX^\top} = \norm{\mX^\top\mX} = \Theta(n)$.
\end{proof}

\section{Omitted Details in \Cref{sec:extension}} \label{app:proof-cnn}

\begin{proof}[Proof of Proposition~\ref{prop:cnn}]
    For an input $\vx \in \R^d$ and an index $k\in[d]$, we let $\index{\vx}{k:k+q}$ be the patch of size $q$ starting from index $k$, i.e., $\index{\vx}{k:k+q}:= \begin{bmatrix}
    \index{\vx}{k}, \index{\vx}{k+1},\ldots, \index{\vx}{k+q-1}
    \end{bmatrix}^\top \in \R^q.
    $
    
    For two datapoints $\vx_i$ and $\vx_j$ ($i,j\in[n]$) and a location $k\in[d]$, we define $$\rho_{i,j,k} := \frac{\left\langle \index{\vx_i}{k:k+q}, \index{\vx_j}{k:k+q} \right\rangle}{q}$$
    which is a local correlation between $\vx_i$ and $\vx_j$.
    
    Now we calculate the infinite-width NTK matrix $\mTheta_{\cnn}$, which is also the expectation of a finite-width NTK matrix with respect to the randomly initialized weights $(\mW,\mV)$. We divide the NTK matrix into two components corresponding to two layers: $\mTheta_{\cnn} = \mTheta_{\cnn}^{(1)} + \mTheta_{\cnn}^{(2)}$, and consider the two layers separately. 
    
    \paragraph{Step 1: the second-layer NTK.}
    Since the CNN model~\eqref{eqn:cnn} is linear in the second layer weights, it is easy to derive the formula for the second-layer NTK:
    \begin{align*}
        \index{\mTheta_{\cnn}^{(2)}}{i,j} &= \frac{1}{d} \E_{\vw\sim\N(\vzero_q,\mI_q)} \left[ \phi(\vw*\vx_i/\sqrt{q})^\top \phi(\vw*\vx_j/\sqrt{q}) \right] \\
        &= \frac{1}{d} \E_{\vw\sim\N(\vzero_q,\mI_q)} \left[ \sum_{k=1}^d \phi(\index{\vw*\vx_i}{k}/\sqrt{q}) \phi(\index{\vw*\vx_j}{k}/\sqrt{q}) \right] \\
        &= \frac{1}{d} \E_{\vw\sim\N(\vzero_q,\mI_q)} \left[ \sum_{k=1}^d \phi\left( \left\langle \vw, \index{\vx_i}{k:k+q} \right\rangle/\sqrt{q} \right) \phi\left( \left\langle \vw, \index{\vx_j}{k:k+q} \right\rangle/\sqrt{q} \right) \right] \\
        &= \frac1d\sum_{k=1}^d P(\rho_{i,j,k}),
    \end{align*}
    where \begin{align*}
        P(\rho) := \E_{(z_1, z_2)\sim\N\left(\vzero, \mLambda\right)} [\phi(z_1)\phi(z_2)], \text{  where } \mLambda = \begin{pmatrix}1& \rho\\\rho & 1\end{pmatrix}, \quad  |\rho|\le 1 .
    \end{align*}
    Note that we have used the property $\norm{\index{\vx_j}{k:k+q}} = \norm{\index{\vx_j}{k:k+q}} = \sqrt{q}$ since the data are from the hypercube $\{\pm1\}^d$.

    For $i\not=j$, we can do a Taylor expansion of $P$ around $0$: $P(\rho) = \zeta^2\rho \pm O(|\rho|^3)$. Here since $\phi=\erf$ is an odd function, all the even-order terms in the expansion vanish.
    Therefore we have
    \begin{align*}
        \index{\mTheta_{\cnn}^{(2)}}{i,j} = \frac{1}{d} \sum_{k=1}^d (\zeta^2 \rho_{i,j,k} \pm O(|\rho_{i,j,k}|^3))
        = \frac{1}{d}\zeta^2 \vx_i^\top \vx_j \pm \frac{1}{d} \sum_{k=1}^d  O(|\rho_{i,j,k}|^3).
    \end{align*}
    Next we bound the error term $\frac{1}{d} \sum_{k=1}^d  |\rho_{i,j,k}|^3$ for all $i\not=j$.
    For each $i,j,k$ ($i\not=j$), since $\vx_i,\vx_j\simiid\unif(\{\pm1\}^d)$, by Hoeffding's inequality we know that with probability $1-\delta$, we have $|\rho_{i,j,k}|\lesssim\sqrt{\frac{\log\tfrac{1}{\delta}}{q}}$.
    Taking a union bound, we know that with high probability, for all $i,j,k$ ($i\not=j$) we have $|\rho_{i,j,k}| = \tilde{O}(q^{-1/2})$. Now we will be conditioned on this happening. Then we write
    \begin{align*}
        \sum_{k=1}^d |\rho_{i,j,k}|^3 = \sum_{k=1}^q |\rho_{i,j,k}|^3 + \sum_{k=q+1}^{2q} |\rho_{i,j,k}|^3 + \cdots,
    \end{align*}
    i.e., we divide the sum into $\lceil d/q \rceil$ groups each containing no more than $q$ terms. By the definition of $\rho_{i,j,k}$, it is easy to see that the groups are independent. Also, we have shown that the sum in each group is at most $q\cdot \tilde{O}(q^{-3/2}) = \tilde{O}(q^{-1/2})$. Therefore, using another Hoeffding's inequality among the groups, and applying a union bound over all $i,j$, we know that with high probability for all $i,j$ ($i\not=j$),
    \begin{align*}
        \frac1d \sum_{k=1}^d |\rho_{i,j,k}|^3 \le \frac1d \tilde{O}(q^{-1/2}) \cdot \tilde{O}(\sqrt{d/q})
        = \tilde{O}\left( \frac{1}{q\sqrt d} \right).
    \end{align*}
    Therefore we have shown that with high probability, for all $i\not=j$,
    \begin{align*}
        \abs{ \index{\mTheta_{\cnn}^{(2)} - \zeta^2\mX\mX^\top/d }{i,j} } = \tilde{O}\left( \frac{1}{q\sqrt d} \right).
    \end{align*}
    This implies
    \begin{align*}
        \norm{\left(\mTheta_{\cnn}^{(2)} - \zeta^2\mX\mX^\top/d \right)_\offdiag}
        &\le \norm{\left(\mTheta_{\cnn}^{(2)} - \zeta^2\mX\mX^\top/d \right)_\offdiag}_F
        = \tilde{O}\left( \frac{n}{q\sqrt{d}} \right)\\
        &= \tilde{O}\left( \frac{n}{d^{\frac12+2\alpha}\sqrt{d}} \right)
        = O\left( \frac{n}{d^{1+\alpha}} \right) .
    \end{align*}
    
    For the diagonal entries, we can easily see
    \begin{align*}
        \norm{\left(\mTheta_{\cnn}^{(2)} - \zeta^2\mX\mX^\top/d \right)_\diag} = O(1) = O\left( \frac{n}{d^{1+\alpha}} \right) .
    \end{align*}
    Combining the above two equations, we obtain
    \begin{align*}
        \norm{\mTheta_{\cnn}^{(2)} - \zeta^2\mX\mX^\top/d } = O\left( \frac{n}{d^{1+\alpha}} \right) .
    \end{align*}

    \paragraph{The first-layer NTK.}
    We calculate the derivative of the output of the CNN with respect to the first-layer weights as:
    \begin{align*}
        \nabla_{\vw_r}f_\cnn(\vx;\mW,\mV) = \frac{1}{\sqrt{md}}\sum_{k=1}^d \index{\vv_r}{k} \phi'\left( \left\langle \vw_r, \index{\vx}{k:k+q}  \right\rangle /\sqrt{q} \right) \index{\vx}{k:k+q} /\sqrt{q} .
    \end{align*}
    Therefore, the entries in the first-layer NTK matrix are
    \begin{align*}
        \index{\mTheta_{\cnn}^{(2)}}{i,j} &= \E_{\mW,\mV}\left[ \sum_{r=1}^m \left\langle \nabla_{\vw_r}f_\cnn(\vx_i;\mW,\mV), \nabla_{\vw_r}f_\cnn(\vx_j;\mW,\mV) \right\rangle \right] \\
        &= \E_{\mW}\left[ \frac{1}{md}\sum_{r=1}^m \sum_{k=1}^d \phi'\left( \left\langle \vw_r, \index{\vx_i}{k:k+q}  \right\rangle /\sqrt{q} \right) \phi'\left( \left\langle \vw_r, \index{\vx_j}{k:k+q}  \right\rangle /\sqrt{q} \right) \rho_{i,j,k} \right] \\
        &= \E_{\vw\sim\N(\vzero_q,\mI_q)}\left[ \frac{1}{d} \sum_{k=1}^d \phi'\left( \left\langle \vw, \index{\vx_i}{k:k+q}  \right\rangle /\sqrt{q} \right) \phi'\left( \left\langle \vw, \index{\vx_j}{k:k+q}  \right\rangle /\sqrt{q} \right) \rho_{i,j,k} \right] \\
        &= \frac{1}{d} \sum_{k=1}^d Q(\rho_{i,j,k}) \cdot \rho_{i,j,k},
    \end{align*}
    where \begin{align*}
        Q(\rho) := \E_{(z_1, z_2)\sim\N\left(\vzero, \mLambda\right)} [\phi'(z_1)\phi'(z_2)], \text{  where } \mLambda = \begin{pmatrix}1&\rho\\\rho& 1\end{pmatrix}, \quad  |\rho|\le 1 .
    \end{align*}
    
    For $i\not=j$, we can do a Taylor expansion of $Q$ around $0$: $Q(\rho) = \zeta^2 \pm O(\rho^2)$. Here since $\phi'=\erf'$ is an even function, all the odd-order terms in the expansion vanish.
    Therefore we have
    \begin{align*}
        \index{\mTheta_{\cnn}^{(1)}}{i,j} = \frac{1}{d} \sum_{k=1}^d (\zeta^2 \pm O(\rho_{i,j,k}^2))\rho_{i,j,k}
        = \frac{1}{d}\zeta^2 \vx_i^\top \vx_j \pm \frac{1}{d} \sum_{k=1}^d  O(|\rho_{i,j,k}|^3).
    \end{align*}
    
    Then, using the exact same analysis for the second-layer NTK, we know that with high probability,
    \begin{align*}
        \norm{\mTheta_{\cnn}^{(1)} - \zeta^2\mX\mX^\top/d } = O\left( \frac{n}{d^{1+\alpha}} \right) .
    \end{align*}
    
    Finally, combining the results for two layers, we conclude the proof of Proposition~\ref{prop:cnn}.
\end{proof}

\end{document}